\documentclass{article}

 \usepackage[preprint]{neurips_2026}

\usepackage[utf8]{inputenc} 
\usepackage[T1]{fontenc}    
\usepackage{hyperref}       
\usepackage{url}            
\usepackage{booktabs}       
\usepackage{amsfonts}       
\usepackage{nicefrac}       
\usepackage{microtype}      
\usepackage{xcolor}         

\usepackage[pdftex]{graphicx}
\usepackage{color}
\usepackage{subcaption}
\usepackage{amssymb,amsmath,amsthm}
\usepackage{mathabx} 
\usepackage{bm}

\allowdisplaybreaks
\usepackage{mathrsfs}
\usepackage{enumerate}

\theoremstyle{plain}
\newtheorem{theorem}{Theorem}[section]
\newtheorem{lemma}[theorem]{Lemma}
\newtheorem{proposition}[theorem]{Proposition}
\newtheorem{corollary}[theorem]{Corollary}

\theoremstyle{definition}
\newtheorem{assumption}[theorem]{Assumption}
\newtheorem{definition}[theorem]{Definition}

\theoremstyle{remark}
\newtheorem{remark}[theorem]{Remark}

\usepackage{xcolor}

\title{Generalization Error Bounds for Picard-Type Operator Learning in Nonlinear Parabolic PDEs}

\author{%
  Koichi Taniguchi\\
   Department of Mathematical and Systems Engineering, Faculty of Engineering\\
  Shizuoka University\\
  3-5-1 Johoku, Chuo-ku, Hamamatsu, 432-8561, Japan \\
  \texttt{taniguchi.koichi@shizuoka.ac.jp} \\
  \And
  Sho Sonoda \\
  RIKEN AIP / CyberAgent \\
  1-4-1 Nihonbashi, Chuo-ku, Tokyo 103-0027, Japan \\
  \texttt{sho.sonoda@riken.jp} \\
}

\begin{document}

\maketitle

\begin{abstract}
Operator learning for partial differential equations (PDEs) aims to learn solution operators on infinite-dimensional function spaces from finite-resolution data. In this setting, it is important for the learned model to be discretization-invariant, or resolution-robust, and to reflect PDE-specific structure. It is therefore natural to ask how such structure should be encoded in the model architecture, hypothesis class, or learning procedure. In this paper, we study operator learning for solution operators of nonlinear parabolic PDEs based on Duhamel--Picard iteration. We formulate Picard iteration as an abstract state-transition model and present a theoretical framework for Picard-type operator learning. We derive implementation-agnostic generalization error bounds that separate the implementation error from the estimation error
associated with the abstract state-transition model induced by Picard iteration. A key consequence is that increasing the Picard depth reduces the Picard truncation error without causing an unbounded growth of the entropy-based estimation error. We also extend the analysis to long-time prediction by rolling out the same learned local model over successive time blocks. Finally, we illustrate the theory for nonlinear heat equations on the torus using a Picard-type Fourier neural operator as a concrete implementation.
\end{abstract}

\section{Introduction}
Operator learning is a deep learning framework for approximating maps between infinite-dimensional function spaces. It has become a central approach to building surrogate solvers for partial differential equations (PDEs) \citep{Lu2019,Li2020a,Kovachki2023,Kovachki2024}. Once trained, such models can predict the solution for a new input without rerunning a classical numerical solver, which is useful in repeated-query tasks such as uncertainty quantification, design optimization, inverse problems, and control. 
Unlike standard high-dimensional regression, however, the object to be learned
is an operator on function spaces. 
Even when the underlying inputs and outputs
are naturally viewed as functions, training data are often accessed through
finite-resolution representations, such as sensor values, grid values, or
truncated coefficients. 
A theory of operator learning must therefore account
both 
for statistical generalization and for robustness with respect to the choice of sensors or discretization.

For PDE surrogate modeling there is an additional issue: the target is not an arbitrary continuous operator. It is generated by an equation and carries analytic structure, such as semigroup and Duhamel formulas, stability estimates, fixed-point iterations, and sometimes comparison or conservation principles. This raises a basic statistical question. If an operator learning model is aligned with the solver structure of the PDE, can increasing its depth improve approximation without paying a growing statistical price? In generic deep-learning bounds, depth can enlarge the hypothesis class and worsen capacity estimates. In solver-aligned operator learning, by contrast, additional depth may simply perform more steps of a stable iterative scheme.

This paper studies this question for nonlinear parabolic PDEs whose local solution operators are obtained by Duhamel--Picard iteration. The work is motivated by two complementary ideas. First, \citet{Furuya2025} showed that neural-operator layers can be aligned with Picard steps and used this alignment to obtain quantitative approximation results without the usual exponential growth of parametric complexity for general operators. Second, \citet[Section~6.2]{Sonoda2025depth} identifies this Picard-aligned neural operator setting as a favorable instance of an implementation-agnostic state-transition framework: Picard iteration gives fast depth-driven approximation, while contractivity keeps the transition word-ball entropy controlled. The present paper turns that observation into a finite-sample generalization theory for Picard-type operator learning.

Our analysis separates three levels of the learning problem. The target concept is the PDE solution operator. The statistical hypothesis class is an abstract Picard class generated by admissible equation specifications and a finite number of Picard steps. A concrete neural operator is treated as an implementation of this abstract Picard class. This distinction is important: the main bounds are implementation-agnostic oracle bounds for predictors obtained by empirical risk minimization over the abstract Picard class and then realized by a concrete implementation map. Thus architecture-specific approximation or optimization issues do not disappear; they enter through an implementation error. The statistical term, however, is controlled by the Rademacher complexity of the Picard transition class rather than by the parameter complexity of a particular neural-operator architecture.

The resulting bounds decompose the prediction error into implementation error, Picard truncation error, statistical estimation error, and, in the finite-observation setting, reconstruction error from sensor data. Under the uniform contraction assumptions used throughout the paper, the Picard truncation error decays geometrically with the Picard depth $\ell$, whereas the entropy-based Rademacher bound remains controlled independently of $\ell$. This gives a concrete learning-theoretic explanation of when depth is useful in operator learning: depth is beneficial when it follows a stable solver structure, so that approximation improves while transition complexity saturates or grows slowly.

Our contributions are as follows. First, we formulate Picard-type operator learning for nonlinear parabolic PDEs in a setting where the underlying equation may be unknown but is assumed to belong to an admissible equation class. The goal is not to identify the equation itself, but to learn the corresponding solution operator using the Picard iteration structure as structural prior information. Second, we formulate this structure as an abstract state-transition model and connect PDE-specific quantities, such as contraction and stability constants, with statistical complexity through the Rademacher complexity of the abstract Picard hypothesis class. Third, we prove generalization error bounds that separate the implementation error, the Picard transition model error, the statistical estimation error, and, in the finite-observation setting, the reconstruction error from sensor data. Fourth, we show that increasing the Picard depth reduces the Picard transition model error without increasing the entropy-based Rademacher term, and we give a practical corollary describing how to balance Picard depth, sample size, and sensor resolution. Fifth, we extend the analysis to long-time prediction by rolling out the learned local model over time slabs and identify stability conditions under which the error remains controlled. Finally, we illustrate how the general theory can be used in a concrete setting by treating nonlinear heat equations on the torus with Fourier neural operators, where we derive explicit implementation error and generalization error bounds.

\section{Related Work}
In operator learning, various models have been proposed, such as DeepONet \citep{Lu2019,Lanthaler2022}, PCA-Net \citep{Bhattacharya2021,Lanthaler2023a}, and neural operators \citep{Li2020b,Li2020a,Kovachki2021,Gupta2021,Bonev2023,TC2023,Chen2024}. Their universal approximation properties have also been studied \citep{ChenChen1993,ChenChen1995,Kovachki2021,Lanthaler2022,Lanthaler2023a,LanthalerLiStuart2025}. On the other hand, when one considers general operators, the curse of parametric complexity, namely the exponential growth of model complexity caused by infinite dimensionality, may appear \citep{LanthalerStuart2026}. To avoid this difficulty, it is important to use additional structure of the target operator. Indeed, several works show that, by restricting the target to solution operators of PDEs, one can avoid the curse of parametric complexity in certain problem classes \citep{Kovachki2021,Chen2023,Lanthaler2023a,MS2023,SchwabSteinZech2023,Furuya2024,MarcatiSchwab2024,LanthalerStuart2026}. Generalization error and capacity analysis for operator learning have also been studied \citep{Lanthaler2022,Gopalani2022,KimKang2024,LaraBenitez2024,KovachkiLanthalerMhaskar2024}.

A work closely related to the present paper is \citet{Furuya2025}. They derived approximation error bounds for neural operators without the curse of parametric complexity by identifying forward propagation through the layers with the steps of Picard iteration. In this sense, their result provides an approximation-theoretic foundation for Picard-aligned neural operators.
Our work is also related to approaches that incorporate fixed-point structures into learning models, such as deep equilibrium models \citep{Bai2019,Marwah2023} and neural general operator networks based on Banach fixed point iterations \citep{FeischlSchwabZehetgruber2025}. For long-time prediction, related experimental studies include \citet{Michalowska2023} and \citet{Lippe2023}, which study error accumulation and long rollout behavior in neural PDE solvers.

This paper is based on the viewpoint that, in learning solution operators, the alignment between the internal structure of the solution operator and the model design is important. From this viewpoint, we analyze the Rademacher complexity of the abstract solution class generated by Picard iteration and derive generalization error bounds by using the implementation-agnostic framework of \citet{Sonoda2025depth}. This work provides a theoretical framework for understanding the generalization behavior of solver-aligned operator learning. See Appendix \ref{app:A} for more details. 

\section{Problem Setting}\label{sec:problem-setting}
In this section, we describe the operator learning setting addressed in this paper. Let $d\in \mathbb N$, $D\subset\mathbb R^d$ be a bounded Lipschitz domain, $T>0$ and $Q_T := D\times (0,T)$. We consider the initial-value problem of nonlinear parabolic PDEs:
\begin{equation}\tag{P}\label{P}
\begin{cases}
\partial_t u + Au = F(u) &\text{in } Q_T,\\
u(0) = u_0 &\text{in }D,
\end{cases}
\end{equation}
where $A$ is a linear operator and $F:\mathbb R\to\mathbb R$ is a nonlinear function satisfying $F(0)=0$. 
Concrete examples are given in Section~\ref{sec:torus-example}
and Appendix~\ref{appendix:B}.

In this paper, we allow the underlying equation to be unknown. Let $\mathcal A$ be a class of admissible linear operators, and let $\mathcal F$ be a class of admissible nonlinearities. We define the admissible equation class by $\Omega_{\rm eq}:=\mathcal A\times\mathcal F$. An element $\omega=(A,F)\in\Omega_{\rm eq}$ is called an admissible equation specification. For each $\omega=(A,F)$, we denote by $\mathcal G_\omega$ the solution operator to \eqref{P}. The unknown target operator $\mathcal G_\ast$ is assumed to be generated by some $\omega_\ast=(A_\ast,F_\ast)\in\Omega_{\rm eq}$, that is, $\mathcal G_\ast = \mathcal G_{\omega_\ast}$. The case where the equation is known corresponds to taking $\Omega_{\rm eq} = \{(A_\ast,F_\ast)\}$. The purpose of this paper is not to identify the equation specification $\omega_\ast$. Rather, the goal is to learn the solution operator $\mathcal G_\ast$ from data (see also Appendix \ref{app:G}). The equation class $\Omega_{\rm eq}$ is used as structural prior information for defining the operator class and for deriving generalization error bounds.

\subsection{Target Operator Class}
By the Duhamel principle, the problem \eqref{P} is formally equivalent to the integral formulation
\begin{equation}\tag{P'}\label{P'}
u(t) = S_A(t) u_0 + \int_0^t S_A(t-\tau) F(u(\tau))\, d\tau,
\end{equation}
where $S_A(t)$ is the semigroup generated by $A$. 
A solution to \eqref{P'} is often called a mild solution to \eqref{P}. In this paper, we take $U_{0,R}:=\{u_0\in C(D):\|u_0\|_{L^\infty(D)}\le R\}$ and $U_M:=\{u\in C([0,T]; C(D))):\|u\|_{L^\infty(Q_T)}\le M\}$, with $R,M>0$, as the initial-data space and the solution space, respectively. For $\omega=(A,F)\in\Omega_{\rm eq}$ and $u_0\in U_{0,R}$, define the Picard map $T_{\omega,u_0}$ by
\begin{equation}\label{eq:picard-map-omega}
T_{\omega,u_0}[v](t):=S_A(t)u_0+\int_0^t S_A(t-\tau)F(v(\tau))\,d\tau .
\end{equation}
In this paper, we impose the following assumption on $\Omega_{\rm eq}$.

\begin{assumption}\label{ass:contraction}
There exist $R,M,T,C_S>0$ and $\delta \in (0,1)$ such that, for any $\omega=(A,F)\in\Omega_{\rm eq}$ and $u_0\in U_{0,R}$, the conditions (a) and (b) hold:
\begin{enumerate}[(a)]
    \item The map $\mathcal T_{\omega,u_0} : U_M \to U_M$ is $\delta$-contractive, i.e.,
\begin{equation}\label{picard_contraction}
\|\mathcal T_{\omega,u_0} [u] - \mathcal T_{\omega,u_0}[v]\|_{L^\infty(Q_T)} \le \delta \|u-v\|_{L^\infty(Q_T)}
\end{equation}
for any $u,v \in U_M$.

\item The map $\mathcal T_{\omega,u_0}$ is Lipschitz continuous with respect to $u_0$, i.e.,
\begin{equation}\label{picard_CDoI}
\|\mathcal T_{\omega,u_0} [u] - \mathcal T_{\omega,v_0}[u]\|_{L^\infty(Q_T)} \le C_S \|u_0-v_0\|_{L^\infty(D)}
\end{equation}
for any $u_0,v_0 \in U_{0,R}$ and $u \in U_M$.
\end{enumerate}
\end{assumption}

Assumption~\ref{ass:contraction} consists of standard conditions
for applying Banach's fixed point theorem to establish local
well-posedness for nonlinear PDEs \citep{Henry1981,Pazy1983}.
The condition (a) guarantees the existence and uniqueness of the
solution as a fixed point, while (b), together with (a), yields
Lipschitz continuous dependence on the initial data.
The condition (b) depends only on the linear operator $A$,
independently of $F$, and is ensured by the uniform bound $\sup_{A\in\mathcal A}\sup_{0<t\leq T}
    \|S_A(t)\|_{L^\infty(D)\to L^\infty(D)}
    \leq C_S$.
Various examples of admissible equation classes satisfying
this assumption, together with detailed verification of these
conditions, are given in Appendix~\ref{appendix:B}.

By Banach's fixed point theorem, we have the following (see, e.g., \citep{Banach1922,Zeidler}).
\begin{proposition}\label{prop:LWP}
Suppose Assumption \ref{ass:contraction} holds. Then the following statements hold:
\begin{enumerate}[\rm (i)]
\item For any $\omega=(A,F)\in\Omega_{\rm eq}$, there exists a unique solution operator $\mathcal G_{\omega} : U_{0,R} \to U_M$ to \eqref{P'} such that
\begin{equation}\label{CDoI}
\|\mathcal G_{\omega} (u_0) - \mathcal G_{\omega}(v_0)\|_{L^\infty(Q_T)} \le \frac{C_S}{1-\delta} \|u_0 - v_0\|_{L^\infty(D)}
\end{equation}
for any $u_0, v_0 \in U_{0,R}$.
\item For each $\omega=(A,F)\in\Omega_{\rm eq}$ and $u_0 \in U_{0,R}$, the Picard iteration 
\[
u^{(\ell)} := \mathcal T_{\omega,u_0}^{[\ell]}[0] =
(\underbrace{\mathcal T_{\omega,u_0}\circ\cdots\circ \mathcal T_{\omega,u_0}}_{\ell\text{ times}})
[0],\quad \ell =1,2,\cdots,
\]
converges to the solution $u := \mathcal G_{\omega} (u_0)$ in $L^\infty(Q_T)$ as $\ell \to \infty$ with
\begin{equation}\label{picard-rate}
\mathrm{d}(u^{(\ell)}, u) \le \frac{\delta^\ell}{1-\delta} \mathrm{d}(u^{(1)}, 0),\quad \ell = 1,2,\cdots.
\end{equation}
\end{enumerate}
\end{proposition}
Form the above, the target operator is the solution operator $\mathcal G_{\ast} : U_{0,R} \to U_M$ and it has the Picard iteration structure. 
In Section \ref{sec:torus-example}, we study semilinear heat equations on the torus as a concrete example.

\subsection{Initial Data Distribution}\label{subsec:initial-data-distribution}
We assume that the initial data $u_0$ are sampled from a probability measure $\Pi$ supported on $U_0 := \{v\in H^{s_0}(D):\|v\|_{H^{s_0}(D)}\le R_0\}$ with $s_0> d/2$. We note that the Sobolev embedding gives $H^{s_0}(D)\hookrightarrow C(\overline D)$, since $D$ is a bounded Lipschitz domain and $s_0>d/2$. In particular, point evaluations are well-defined. Moreover, we choose $R_0\le C_{\rm Sob}^{-1}R$ so that $U_0\subset U_{0,R}$, where $C_{\rm Sob}=C_{\rm Sob}(D,s_0)>0$ is some constant in the Sobolev inequality.

\begin{remark}\label{rem:GP}
A canonical example is obtained by conditioning a centered Gaussian process
on the Sobolev ball $U_0$.  If the covariance kernel $K$ is chosen so that
$u_0\sim GP(0,K)$ has $H^{s_0}(D)$-valued sample paths, then the Gaussian
law on the separable Hilbert space $H^{s_0}(D)$ has positive mass on every
ball centered at the origin; hence
$\mathbb P(\|u_0\|_{H^{s_0}(D)}\le R_0)>0$
\citep[Corollary~2.6.17]{Gine2015}.  For example, on bounded Lipschitz
domains, Mat\'ern kernels with smoothness parameter $\nu$ yield
$H^s(D)$-sample paths for every $s<\nu$, while squared-exponential kernels
allow arbitrary finite Sobolev smoothness; see, e.g.,
\citet{Henderson2024}.
\end{remark}

\subsection{Population and Empirical Risks}\label{subsec:risks}
\paragraph{Ideal risks.}
Let $\Gamma:U_0\to U_M$ be a candidate operator model for the target solution operator $\mathcal G_{\ast}:U_0\to U_M$ associated with \eqref{P'}. Let $\mu$ be a probability measure on $Q_T$ describing the distribution of output query points. We define the ideal population risk by
\begin{equation}\label{eq:ideal-risk}
L[\Gamma]:=\mathbb E_{u_0\sim\Pi, z\sim\mu}\left[\left|\Gamma(u_0)(z)-\mathcal G_{\ast}(u_{0})(z)\right|^2\right].
\end{equation}
The ideal risk corresponds to the full-information setting, in which the model
$\Gamma$ is evaluated as an operator acting on the whole initial function
$u_0$, rather than on finitely many observations of $u_0$.
This setting is natural for theoretical or synthetic datasets where $u_0$ is
explicitly generated and the corresponding solution values can be evaluated.

Given i.i.d. trajectory blocks
$(u_{0,i},z_{i,1},\ldots,z_{i,q})\sim \Pi\otimes\mu^{\otimes q}$,
the corresponding empirical risk is defined by
\[
\widehat L_{n,q}[\Gamma]
:=
\frac1n\sum_{i=1}^n\frac1q\sum_{j=1}^q
\ell_{\Lambda_{\mathrm{clip}}}
\left(
\Gamma(u_{0,i})(z_{i,j}),
\mathcal G_{\ast}(u_{0,i})(z_{i,j})
\right).
\]
Here, the clipped squared loss is defined by
$\ell_{\Lambda_{\mathrm{clip}}}(a,y)
    := \min\bigl\{|a-y|^2,\Lambda_{\mathrm{clip}}\bigr\}$ with $\Lambda_{\mathrm{clip}}\ge 4M^2$.
For $a,y\in[-M,M]$, we have
$\ell_{\Lambda_{\mathrm{clip}}}(a,y)=|a-y|^2$.
Thus, when predictions and target values lie in $[-M,M]$,
the empirical and population risks use the same squared loss.

\paragraph{Risks in finite observations.}
In many practical settings, the input is available only through a finite-resolution representation, such as sensor or grid values. We model this by the observation map. Given sensor locations $\Xi_m:=\{\xi_1,\dots,\xi_m\}\subset D$, we define the observation map $P_m : H^{s_0}(D) \to \mathbb R^m$ by
\begin{equation}\label{obser_map}
P_m u_0:=(u_0(\xi_1),\cdots,u_0(\xi_m))\in\mathbb R^m.
\end{equation}
For a finite-observation model $\Gamma_m:\mathbb R^m\times Q_T\to\mathbb R$, the finite-observation population risk is defined by
\begin{equation}\label{eq:clean-risk}
L_{m}[\Gamma_m]:=\mathbb E_{u_0\sim\Pi, z\sim\mu}\left[\left|\Gamma_m(P_m u_0,z)-\mathcal G_{\ast}(u_{0})(z)\right|^2\right],
\end{equation}
and the empirical risk is defined by
\begin{equation}\label{eq:empirical-observed-risk}
\widehat L_{m, n,q}[\Gamma_m]:=\frac1n\sum_{i=1}^n\frac1q\sum_{j=1}^q \ell_{\Lambda_{\mathrm{clip}}}\left(\Gamma_m(P_m u_{0,i},z_{i,j}),\mathcal G_{\ast}(u_{0,i})(z_{i,j})\right).
\end{equation}
We also lift $\Gamma_m$ to an operator model $\Gamma_m^\uparrow:U_0\to U_M$ by
$\Gamma_m^\uparrow(u_0)(z) := \Gamma_m(P_m u_0,z)$.
Thus, even though the model only uses finite sensor values, its output is compared with the full solution value $\mathcal G_{\ast}(u_0)(z)$. In this sense, the target of evaluation remains the full solution operator. Note that
$L[\Gamma_m^\uparrow] = L_{m}[\Gamma_m]$. 
This identity identifies the finite-observation risk with the restriction of the continuum-level risk to lifted, sensor-dependent models.

\section{Picard Iteration as Abstract State-Transition Model}\label{sec:3}
In this section, we rewrite the Picard iteration as an abstract state-transition model
so that it fits the framework of \citet{Sonoda2025depth}. We also derive
bounds for the Picard truncation error. The abstract Picard class introduced
here will serve as the main hypothesis class in the generalization analysis in
Section~\ref{sec:5}, where the estimation error is controlled by the Picard truncation
error, viewed as an oracle term, together with the Rademacher complexity of abstract Picard class. 
The details and proofs for this section are given in Appendix~\ref{app:C}.

\subsection{The $\ell$-step Picard Transition Model}
Define the state space $X:=U_{0,R}\times Q_T\times U_M$. For $(v,z,u)\in X$, where $z=(x,t)\in Q_T$, define $f_{\omega}(v,z,u):=(v,z,\mathcal T_{\omega,v}[u])$, $J(v,z):=(v,z,0)$, and $h(v,z,u):=u(z)$, where $f_{\omega}$ is the transition, $J$ is the initial embedding, and $h$ is the readout. Then the $\ell$-step Picard transition model is $$b_{\omega, \ell}(v,z):=h\circ f_{\omega}^{[\ell]}\circ J(v,z)=(\mathcal T_{\omega,v}^{[\ell]}[0])(z).$$ We define the abstract Picard class by $\mathcal B^{\Omega_{\rm eq}}_{\ell}:=\{b_{\omega,\ell} : \omega\in\Omega_{\rm eq}\}$.

\begin{lemma}[Picard truncation error]\label{lem:picard_error1}
Suppose that Assumption \ref{ass:contraction} holds. Then, 
for any $\omega\in\Omega_{\rm eq}$ and $\ell \in \mathbb N$, we have
\begin{equation*}
\sup_{v\in U_{0,R}} \|b_{\omega, \ell}(v,\cdot)-\mathcal G_{\omega}[v]\|_{L^\infty(Q_T)}\le \frac{M\delta^\ell}{1-\delta}.
\end{equation*}
Consequently, $\inf_{b \in \mathcal B^{\Omega_{\rm eq}}_{\ell}} L[b]\le M^2\delta^{2\ell}/(1-\delta)^2.$
\end{lemma}

\subsection{The $\ell$-step Finite-Observation Picard Transition Model}
We consider the case where the learner observes only $P_m u_0$. Let $\mathcal E_m\subset\{E:\mathbb R^m\to U_{0}\}$ be a class of reconstruction maps. For $E\in\mathcal E_m$, define the state space $X_m:=\mathbb R^m\times Q_T\times U_M$. For $(a,z,u) \in X_m$, define $f_{\omega,E}(a,z,u):=(a,z,\mathcal T_{\omega,E(a)}[u])$, $J_m(a,z):=(a,z,0)$, and $h_m(a,z,u):=u(z)$. Hence, the $\ell$-step finite-observation Picard transition model is $$b_{\omega,E,\ell}(a,z):=h_m\circ f_{\omega,E}^{[\ell]}\circ J_m(a,z)=(\mathcal T_{\omega, E(a)}^{[\ell]}[0])(z).$$ We define the abstract finite-observation Picard class by $\mathcal B^{\Omega_{\rm eq}}_{m,\ell}:=\{b_{\omega,E,\ell}:\omega\in \Omega_{\rm eq}, E\in\mathcal E_m\}$. The reconstruction error is defined by $\varepsilon_{\rm rec}^2(m):=\inf_{E\in\mathcal E_m}\mathbb E_{u_0\sim\Pi}[\|E(P_m u_0)-u_0\|_{L^\infty(D)}^2]$.

\begin{lemma}[Picard truncation error]\label{lem:picard_error2}
Suppose that Assumption \ref{ass:contraction} holds. Then,
for any $\omega\in\Omega_{\rm eq}$, $E\in\mathcal E_m$, $u_0 \in U_0$ and $\ell \in \mathbb N$, we have
\begin{equation*}
\|b_{\omega,E,\ell}(P_m u_0,\cdot)-\mathcal G_{\omega}[u_0]\|_{L^\infty(Q_T)}\le \frac{M\delta^\ell}{1-\delta}+\frac{C_S}{1-\delta}\|E(P_m u_0)-u_0\|_{L^\infty(D)}.
\end{equation*}
Consequently, $\inf_{b \in \mathcal B^{\Omega_{\rm eq}}_{m,\ell}} L_m(b)\le 2M^2\delta^{2\ell}/ (1-\delta)^2 + 2C_S^2\varepsilon_{\rm rec}^2(m)/(1-\delta)^2$.
\end{lemma}

\section{Generalization Error Bounds}\label{sec:5}
Let $\mathfrak N_{\ell} \subset \{\Gamma : U_0 \to U_M\}$ and $\mathfrak N_{m,\ell}\subset\{\Gamma_m : \mathbb R^m \times Q_T \to \mathbb R\}$, and let $\operatorname{embed}:\mathcal B^{\Omega_{\rm eq}}_{\ell}\to\mathfrak N_{\ell}$ and $\operatorname{embed}_m:\mathcal B^{\Omega_{\rm eq}}_{m,\ell}\to\mathfrak N_{m,\ell}$ be implementation embeddings from the abstract Picard classes into the corresponding implementation classes, respectively. Define the implementation errors by $\varepsilon_{\rm imp}:=\sup_{b\in\mathcal B^{\Omega_{\rm eq}}_{\ell}}|L[\operatorname{embed} b]-L[b]|$ and $\varepsilon_{\rm imp}^{m}:=\sup_{b\in\mathcal B^{\Omega_{\rm eq}}_{m,\ell}}|L_m[\operatorname{embed}_m b]-L_m[b]|$.

Let $\mathcal X$ be an input space and let $\mathcal B\subset\{b:\mathcal X\times Q_T\to\mathbb R\}$. Let $\nu_{\mathcal X}$ be the distribution of the input variable on $\mathcal X$. For a trajectory-level sample $\mathcal S_{n,q}^{\mathcal X}=\{(a_i,\{z_{i,j}\}_{j=1}^q)\}_{i=1}^n$ with independent blocks $(a_i,z_{i,1},\ldots,z_{i,q})\sim \nu_{\mathcal X}\otimes\mu^{\otimes q}$, define the pointwise empirical Rademacher complexity by
\begin{equation*}
\widehat{\mathfrak R}^{\rm pt}_{\mathcal S_{n,q}^{\mathcal X}}(\mathcal B):=\mathbb E_\sigma\left[\sup_{b\in\mathcal B}\frac1{nq}\sum_{i=1}^n\sum_{j=1}^q\sigma_{i,j} b(a_i,z_{i,j})\right],
\end{equation*}
where $\{\sigma_{i,j}:1\le i\le n,1\le j\le q\}$ are independent Rademacher variables. In the bounds below we use the trajectory-normalized version
\begin{equation*}
\widehat{\mathfrak R}_{\mathcal S_{n,q}^{\mathcal X}}(\mathcal B):=\sqrt q\,\widehat{\mathfrak R}^{\rm pt}_{\mathcal S_{n,q}^{\mathcal X}}(\mathcal B)=\mathbb E_\sigma\left[\sup_{b\in\mathcal B}\frac1{n\sqrt q}\sum_{i=1}^n\sum_{j=1}^q\sigma_{i,j} b(a_i,z_{i,j})\right].
\end{equation*}
We use the following shorthand: $\widehat{\mathfrak R}^{\Omega_{\rm eq}}:=\widehat{\mathfrak R}_{\mathcal S_{n,q}^{U_0}}(\mathcal B^{\Omega_{\rm eq}}_{\ell})$ and $\widehat{\mathfrak R}_{m}^{\Omega_{\rm eq}}:=\widehat{\mathfrak R}_{\mathcal S_{n,q}^{P_m(U_0)}}(\mathcal B^{\Omega_{\rm eq}}_{m,\ell})$, with $\nu_{U_0}=\Pi$ and $\nu_{P_m(U_0)}=(P_m)_{\#}\Pi$.

The following theorem is an implementation-agnostic oracle bound in the sense of \citet{Sonoda2025depth}. The empirical minimization is performed over the abstract Picard class, and the selected abstract predictor is then realized by a concrete implementation map. Thus the bound does not claim that every parameter-space ERM over a neural-operator class enjoys the same guarantee; any architecture-specific approximation or optimization gap is represented through the implementation error.

\begin{theorem}[Generalization error bounds]\label{thm:generalization}
Suppose that Assumption \ref{ass:contraction} holds and 
that all predictors are clipped to $[-M,M]$. Then there exists a universal constant $C>0$ such that, for every $0<\zeta<1$, the following statements hold with probability at least $1-\zeta$:
\begin{enumerate}[\rm (i)]
\item Let $\widehat b_{\ell}\in\operatorname*{argmin}_{b\in\mathcal B^{\Omega_{\rm eq}}_{\ell}}\widehat L_{n,q}[b]$ and set $\widehat\Gamma_{\ell}:=\operatorname{embed}(\widehat b_{\ell})\in\mathfrak N_{\ell}$. Then
\begin{equation*}
L[\widehat\Gamma_{\ell}]\le \varepsilon_{\rm imp}+\frac{M^2\delta^{2\ell}}{(1-\delta)^2}+CM\widehat{\mathfrak R}^{\Omega_{\rm eq}}+CM^2\sqrt{\frac{\log(2/\zeta)}{n}}.
\end{equation*}
\item Let $\widehat b_{m,\ell}\in\operatorname*{argmin}_{b\in\mathcal B^{\Omega_{\rm eq}}_{m,\ell}}\widehat L_{m,n,q}[b]$ and set $\widehat\Gamma_{m,\ell}:=\operatorname{embed}_m(\widehat b_{m,\ell})\in\mathfrak N_{m,\ell}$. Then
\begin{equation*}
L_m[\widehat\Gamma_{m,\ell}]\le \varepsilon_{\rm imp}^{m}+\frac{2M^2\delta^{2\ell}}{(1-\delta)^2}+\frac{2C_S^2}{(1-\delta)^2}\varepsilon_{\rm rec}^2(m)+CM\widehat{\mathfrak R}_{m}^{\Omega_{\rm eq}}+CM^2\sqrt{\frac{\log(2/\zeta)}{n}}.
\end{equation*}
\end{enumerate}
\end{theorem}
The proof is based on viewing the Picard class defined in Section~\ref{sec:3} as an abstract state-transition model and applying the bias-variance decomposition in the framework of \citet[Theorem~1]{Sonoda2025depth}. The empirical-process step is performed over the $n$ independent trajectory blocks. Since the clipped squared loss is $CM$-Lipschitz in the scalar prediction on $[-M,M]$, the block loss $r\mapsto q^{-1}\sum_{j=1}^q\ell_{\Lambda_{\mathrm{clip}}}(r_j,y_j)$ is $CM/\sqrt q$-Lipschitz with respect to the Euclidean norm on $\mathbb R^q$; Maurer's vector contraction inequality \citep{Maurer2016} then yields the terms $CM\widehat{\mathfrak R}^{\Omega_{\rm eq}}$ and $CM\widehat{\mathfrak R}_{m}^{\Omega_{\rm eq}}$. The concentration term is also at the trajectory-block level and therefore scales as $n^{-1/2}$. The model error is estimated by using Lemmas~\ref{lem:picard_error1} and \ref{lem:picard_error2}. The details of proof is given in Appendix \ref{app:D}. 

\begin{remark}\label{rem:5.2}
Theorem~\ref{thm:generalization} applies to any admissible equation class
$\Omega_{\rm eq}$ satisfying Assumption~\ref{ass:contraction}. Therefore, as
illustrated in Appendix~\ref{appendix:B}, it provides a general theoretical
framework for Picard-type operator learning across a broad range of nonlinear
parabolic PDEs (see Proposition \ref{prop:ass_app} and examples in Appendix~\ref{appendix:B}). In particular, the framework does not require the underlying
equation specification to be identified; it only requires that the target
equation belong to an admissible class satisfying the uniform Picard contraction
condition.
\end{remark}

\paragraph{A more explicit form.}
The Rademacher terms in Theorem~\ref{thm:generalization} can be bounded by metric entropy with respect to the empirical pointwise metric
\begin{equation*}
d_{\mathcal S,2}(b,b'):=\left(\frac1{nq}\sum_{i=1}^n\sum_{j=1}^q|b(a_i,z_{i,j})-b'(a_i,z_{i,j})|^2\right)^{1/2}.
\end{equation*}
Dudley's entropy integral gives $\widehat{\mathfrak R}_{\mathcal S_{n,q}^{\mathcal X}}(\mathcal B)\lesssim n^{-1/2}\int_0^{2M}\sqrt{\log N(\mathcal B,d_{\mathcal S,2},\varepsilon)}\,d\varepsilon$. Let $d_\Omega(\omega,\omega'):=\sup_{u_0\in U_{0,R}}\sup_{u\in U_M}\|\mathcal T_{\omega,u_0}[u]-\mathcal T_{\omega',u_0}[u]\|_{L^\infty(Q_T)}$. If $\log N(\Omega_{\rm eq},d_\Omega,\varepsilon)\le\mathcal H_\Omega(\varepsilon)$, then, by applying the Rademacher-complexity estimate of \citet{Sonoda2025depth} to the abstract Picard transition class and using the contraction property in Assumption~\ref{ass:contraction}, the full-information Rademacher term satisfies
\begin{equation*}
\widehat{\mathfrak R}^{\Omega_{\rm eq}}\lesssim \frac1{\sqrt n}\int_0^{2M}\sqrt{\mathcal H_\Omega(c(1-\delta)\varepsilon)}\,d\varepsilon,
\end{equation*}
where $c>0$ is a universal constant. In particular, this upper bound is independent of the Picard depth $\ell$. For the finite-observation class, one obtains similarly $\widehat{\mathfrak R}_{m}^{\Omega_{\rm eq}}\lesssim \mathcal I_{\Omega,m}/\sqrt n$, where $\mathcal I_{\Omega,m}$ is the corresponding entropy integral involving both $(\Omega_{\rm eq},d_\Omega)$ and the reconstruction class $\mathcal E_m$. This bound is also independent of $\ell$. Thus increasing $\ell$ reduces the Picard truncation error while the entropy-based statistical bound does not increase with $\ell$.

Moreover, if $\varepsilon_{\rm rec}^2(m)\le C_{\rm rec}^2m^{-2\beta}$, then
\begin{equation*}
L_m[\widehat\Gamma_{m,\ell}]\le \varepsilon_{\rm imp}^{m}+\frac{2M^2\delta^{2\ell}}{(1-\delta)^2}+\frac{2C_S^2C_{\rm rec}^2}{(1-\delta)^2}m^{-2\beta}+\frac{CM}{\sqrt n}\mathcal I_{\Omega,m}+CM^2\sqrt{\frac{\log(1/\zeta)}{n}}.
\end{equation*}
For quasi-uniform sensors and stable reconstruction on $U_0\subset H^{s_0}(D)$, $s_0>d/2$, one typically has $\beta=(s_0-d/2)/d$.

\begin{corollary}[Practical depth and sensor scaling]\label{cor:practical-scaling}
Assume the setting of Theorem~\ref{thm:generalization}. Suppose that the full-information entropy integral is finite, $\mathcal I_{\Omega}:=\int_0^{2M}\sqrt{\mathcal H_\Omega(c(1-\delta)\varepsilon)}\,d\varepsilon<\infty$, and that the finite-observation entropy integral satisfies $\mathcal I_{\Omega,m}\le C_I m^\alpha$ for some $C_I>0$ and $\alpha\ge0$. Suppose also that $\varepsilon_{\rm rec}^2(m)\le C_{\rm rec}^2m^{-2\beta}$ for some $\beta>0$. Let $\ell_n:=\left\lceil \log n/(4|\log\delta|)\right\rceil$. Then the full-information bound with Picard depth $\ell_n$ takes the form
\begin{equation*}
L[\widehat\Gamma_{\ell_n}]\le \varepsilon_{\rm imp}+C\left(\frac{M^2}{\sqrt n}+\frac{M\mathcal I_\Omega}{\sqrt n}+M^2\sqrt{\frac{\log(2/\zeta)}{n}}\right).
\end{equation*}
Moreover, if $m_n\asymp n^{1/(2(2\beta+\alpha))}$, then the finite-observation bound with $m=m_n$ and $\ell=\ell_n$ takes the form
\begin{equation*}
L_{m_n}[\widehat\Gamma_{m_n,\ell_n}]\le \varepsilon_{\rm imp}^{m_n}+C\left(n^{-\beta/(2\beta+\alpha)}+M^2\sqrt{\frac{\log(2/\zeta)}{n}}\right),
\end{equation*}
where the constant $C$ may depend on $M,C_S,\delta,C_{\rm rec},C_I$, but not on $n$. In the quasi-uniform reconstruction example, $\beta=(s_0-d/2)/d$.
\end{corollary}
The details of proof is given in Appendix \ref{app:D}. 
Corollary~\ref{cor:practical-scaling} gives the design rule suggested by the main theorem. Since the Picard truncation term is geometric, logarithmic depth $\ell_n=O(\log n)$ is enough to push it down to the statistical scale $n^{-1/2}$. The finite-observation case adds the usual trade-off between sensor resolution and statistical complexity: increasing $m$ reduces reconstruction error but may enlarge the reconstruction-class entropy.

\section{Long-Time Prediction by Local Rollout}\label{sec:long-time}
In this section, we study long-time prediction by iterating the same
learned local model $\widehat\Gamma_\ell$ over successive time blocks.
Throughout this section, we assume that $\widehat\Gamma_\ell$,
initially defined on $U_0$, extends to a map
$\widehat\Gamma_\ell:U_{0,R}\to U_M$.
We use the same notation for this extension. 
The resulting error bound depends on the rollout-relevant local error
\eqref{local_error_1}, which measures the uniform approximation error
on the states visited by the rollout.
Control of this error is an additional requirement and does not follow
directly from the population risk bound in Theorem~\ref{thm:generalization}.
Since the finite-observation case can be treated similarly,
we discuss only the full-information case here.

Define $\psi:U_M\to L^\infty(D)$ by $\psi(u):=u(\cdot,T)$, and assume that $\|\psi(u)-\psi(v)\|_{L^\infty(D)}\le L_\psi\|u-v\|_{L^\infty(Q_T)}$. Note that $L_{\psi} \le 1$. Let $\Phi_\ast:=\psi\circ\mathcal G_\ast$. For $\kappa\in\mathbb N$, define $U_0^{[\kappa]}:=\{u_0\in U_0:\Phi_\ast^j(u_0)\in U_{0,R}\text{ for }j=0,\ldots,\kappa\}$. This condition ensures that the same local solution theory can be applied on each time block. We assume $\Pi(U_0^{[\kappa]})>0$ and write $\Pi_\kappa:=\Pi(\cdot\mid U_0^{[\kappa]})$.

For $u_0\in U_0^{[\kappa]}$, define the exact states by $v_0=u_0$ and $v_{j+1}=\Phi_\ast(v_j)$. The exact solution on the $j$-th block is $\mathcal G_\ast[v_j]$ after the time shift $t=jT+s$, $s\in[0,T]$. Next, define the approximate transition by $\widehat\Phi_\ell:=\mathcal P_R\circ\psi\circ\widehat\Gamma_\ell$, where $\mathcal P_R:C(D)\to U_{0,R}$ is the pointwise clipping map. The approximate states are $\widehat v_0=u_0$ and $\widehat v_{j+1}=\widehat\Phi_\ell(\widehat v_j)$. Thus the model prediction on the $j$-th block is $\widehat\Gamma_\ell[\widehat v_j]$. For $j=0,\ldots,\kappa-1$, define the block-wise risk by
\begin{equation*}
L_j^{\rm block}[\widehat\Gamma_\ell]:=\mathbb E_{u_0\sim\Pi_\kappa,\ z\sim\mu}\left[\left|\widehat\Gamma_\ell[\widehat v_j](z)-\mathcal G_\ast[v_j](z)\right|^2\right].
\end{equation*}
Let $\mathcal V_\kappa(\widehat\Gamma_\ell):=\{v_j(u_0),\widehat v_j(u_0):u_0\in U_0^{[\kappa]},\ j=0,\ldots,\kappa\}\subset U_{0,R}$, and define the rollout-relevant local error by
\begin{equation}\label{local_error_1}
\varepsilon_{\rm loc,\ell}^{[\kappa]}:=\sup_{v\in\mathcal V_\kappa(\widehat\Gamma_\ell)}\left\|\widehat\Gamma_\ell[v]-\mathcal G_\ast[v]\right\|_{L^\infty(Q_T)} .
\end{equation}

\begin{theorem}[Block-wise long-time error propagation]\label{thm:blockwise-long-time}
Suppose that Assumption~\ref{ass:contraction} holds. Then, for every $j=0,\ldots,\kappa-1$,
\begin{equation*}
L_j^{\rm block}[\widehat\Gamma_\ell]\le \left(\sum_{r=0}^{j}\left(L_\psi\frac{C_S}{1-\delta}\right)^r\right)^2\left(\varepsilon_{\rm loc,\ell}^{[\kappa]}\right)^2 .
\end{equation*}
\end{theorem}
This result shows that long-time prediction is obtained by deterministic propagation of the local error. Since no additional empirical risk minimization is performed after $\widehat\Gamma_\ell$ has been learned, the rollout step does not introduce a new Rademacher complexity term. The block index $j$ enters only through the stability factor in Theorem~\ref{thm:blockwise-long-time}. 
As in Theorem~\ref{thm:generalization}, 
Theorem~\ref{thm:blockwise-long-time} also applies to any admissible equation class
$\Omega_{\rm eq}$ satisfying Assumption~\ref{ass:contraction} (see also Remark \ref{rem:5.2}).
The proof iterates the local stability estimate \eqref{CDoI} along the exact and approximate trajectories. 
See Appendix \ref{app:E} for more details and the proof and also Appendix \ref{app:G} for additional remarks.

\section{A Concrete Example: Nonlinear Heat Equations on the Torus}
\label{sec:torus-example}

We consider the nonlinear heat equation on
$\mathbb T^d:=\mathbb R^d/(2\pi\mathbb Z)^d$:
\begin{equation}\label{eq:torus-heat}
    \partial_t u-\Delta_{\mathbb T^d}u=F(u),
    \quad
    u(0)=u_0.
\end{equation}
Let $R,M,L>0$, and set
$\mathcal F_{M,L}
:=
\{F:[-M,M]\to\mathbb R:
F(0)=0,\ {\rm Lip}_{[-M,M]}(F)\le L\}$
and
$\Omega_{\rm eq}
:=
\{(-\Delta_{\mathbb T^d},F):F\in\mathcal F_{M,L}\}$.
For $F\in\mathcal F_{M,L}$ and $u_0\in U_0$,
we denote by $\mathcal T_{F,u_0}$ the corresponding Picard map.
We assume that $R,M,L,T>0$ and $\delta\in(0,1)$ satisfy
\begin{equation}\label{key_para}
    R+TLM\le M,
    \quad
    TL\le\delta.
\end{equation}
Then Assumption~\ref{ass:contraction} holds with $C_S=1$,
using $C(\mathbb T^d)$ as the space of continuous functions on $\mathbb T^d$.
Let $s_0>d/2$ and
$U_0:=\{u_0\in H^{s_0}(\mathbb T^d):
\|u_0\|_{H^{s_0}}\le R_0\}\subset U_{0,R}$.
See Appendices~\ref{appendix:B}
and~\ref{app:F:torus-setting} for the details and the learning setting.
We discuss only the full-information case here.

The linear heat semigroup is diagonalized by the Fourier basis:
$e^{t\Delta_{\mathbb T^d}}f
=\sum_{\xi\in\mathbb Z^d}
e^{-t|\xi|^2}\widehat f(\xi)e_\xi$,
where $e_\xi(x):=(2\pi)^{-d/2}e^{i\xi\cdot x}$.
For the implementation, we use the tensor-product Fej\'er
operator $P_N$ and put
\[
    S_N(t)f:=P_Ne^{t\Delta_{\mathbb T^d}}f,
    \quad
    \mathcal K_N[g](t)
    :=\int_0^t S_N(t-\tau)g(\tau)\,d\tau.
\]
The operator $P_N$ is $L^\infty$-contractive and retains
the modes $\xi\in\{-N,\ldots,N\}^d$;
see Appendix~\ref{app:F:picard-type-fno}.

\begin{definition}[Picard-type Fourier neural operator]
Let $N,\ell,H\in\mathbb N$, and let
$\rho:\mathbb R\to\mathbb R$ be a one-dimensional ReLU
network of size at most $H$ satisfying
$\rho(0)=0$ and ${\rm Lip}_{[-M,M]}(\rho)\le L$.
For $u_0\in U_0$, define
$u^{(0)}_{N,\rho}[u_0]:=0$ and
\[
    u^{(j+1)}_{N,\rho}[u_0]
    :=
    S_N(\cdot)u_0
    +
    \mathcal K_N
    \left[
        \rho\bigl(u^{(j)}_{N,\rho}[u_0]\bigr)
    \right],
    \quad j=0,\ldots,\ell-1.
\]
The $\ell$-step Picard-type FNO
$\Gamma^{\rm FNO}_{N,\ell,\rho}:U_0\to U_M$
is defined by
$\Gamma^{\rm FNO}_{N,\ell,\rho}(u_0)
:=u^{(\ell)}_{N,\rho}[u_0]$.
\end{definition}

We denote the corresponding class by
$\mathfrak N^{\rm FNO}_{N,\ell,H}(M,L)$.
Only $\rho$ is trainable and is shared across all Picard steps.
Figure~\ref{fig:picard-fno-architecture} illustrates the
Picard-type FNO architecture. See Appendix~\ref{app:F:picard-type-fno} for more details.

\begin{theorem}[Implementation error for the Picard-type FNO]
\label{thm:torus-implementation-error}
Assume \eqref{key_para}. Then, 
for any $\eta\in(0,1)$ and any $F\in\mathcal F_{M,L}$,
there exists a one-dimensional ReLU network $\rho_F$ such that
\[
    \|F-\rho_F\|_{L^\infty([-M,M])}\le\eta,
    \quad
    \rho_F(0)=0,
    \quad
    {\rm Lip}_{[-M,M]}(\rho_F)\le L,
    \quad
    H(\rho_F)\le C(1+ML\eta^{-1}).
\]
Moreover, there exist errors $a_N(\ell)$, independent of
$\eta$, such that
\[
    \sup_{F\in\mathcal F_{M,L}}
    \sup_{u_0\in U_0}
    \left\|
    \Gamma^{\rm FNO}_{N,\ell,\rho_F}(u_0)
    -
    \mathcal T_{F,u_0}^{[\ell]}[0]
    \right\|_{L^\infty(Q_T)}
    \le
    \frac{1-\delta^\ell}{1-\delta}\,T\eta+a_N(\ell).
\]
In addition, for any $0<\gamma<\min\{s_0-d/2,1\}$,
there exists a constant $C_\gamma>0$, independent of
$N,\ell,\eta$, such that
\[
    a_N(\ell)\le C_\gamma N^{-\gamma}
    \quad\text{for any }N,\ell\in\mathbb N.
\]
Consequently, the implementation map
$\operatorname{embed}_{N,\eta}(b_{(-\Delta_{\mathbb T^d},F),\ell})
:=\Gamma^{\rm FNO}_{N,\ell,\rho_F}$,
defined using one fixed representative $F$ for each predictor,
satisfies
\[
    \varepsilon_{\rm imp}^{\rm FNO}(N,\ell,\eta)
    \le
    4M
    \left\{
        \frac{1-\delta^\ell}{1-\delta}\,T\eta+a_N(\ell)
    \right\}.
\]
In particular, this error can be made arbitrarily small
by choosing $\eta$ sufficiently small and $N$ sufficiently
large, uniformly in $\ell$.
\end{theorem}
\begin{figure}[t]
    \centering
    \includegraphics[width=\linewidth]{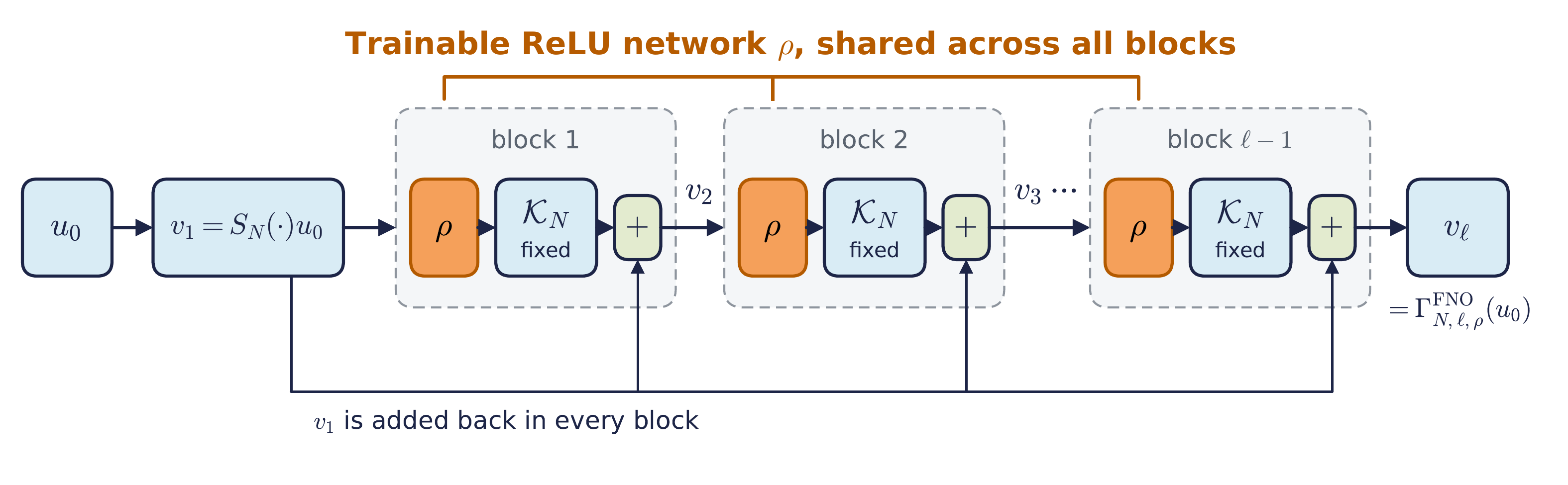}
    \caption{
        Picard-type FNO.
        The first iterate is
        $v_1=S_N(\cdot)u_0$.
        Each subsequent block computes
        $v_{j+1}=v_1+\mathcal K_N[\rho(v_j)]$,
        $j=1,\ldots,\ell-1$.
        The same ReLU network $\rho$ is shared across
        all blocks, while $\mathcal K_N$ is fixed.
        The term $v_1$ is added in every block.
        The output is
        $v_\ell=\Gamma^{\rm FNO}_{N,\ell,\rho}(u_0)$.
    }
    \label{fig:picard-fno-architecture}
\end{figure}

The proof uses piecewise affine approximation of $F$ by a
ReLU network, Picard stability, and uniform Fej\'er
approximation on compact families.
The explicit rate follows from Sobolev regularity and
heat-semigroup smoothing;
see Appendix~\ref{app:F:proof-implementation-error}
and Proposition~\ref{prop:F:fourier-rate}.

\begin{lemma}[Rademacher complexity of the Picard class]
\label{lem:rademacher-torus}
Assume \eqref{key_para}.
Then, for any $\ell\in\mathbb N$, the empirical
Rademacher complexity of the Picard class satisfies
\[
    \widehat{\mathfrak R}^{\Omega_{\rm eq}}
    =
    \widehat{\mathfrak R}_{\mathcal S_{n,q}^{U_0}}
    \bigl(\mathcal B^{\Omega_{\rm eq}}_\ell\bigr)
    \le
    CM
    \sqrt{
        \frac{LT(1-\delta^\ell)}{(1-\delta)n}
    },
\]
where $C>0$ is a universal constant.
\end{lemma}

The proof is given in Appendix~\ref{app:F:proof-rademacher-torus}.

\begin{theorem}[Generalization error bound for the Picard-type FNO]
\label{thm:torus-fno-generalization}
Assume \eqref{key_para}.
Let
$\widehat b_\ell\in
\operatorname*{argmin}_{b\in\mathcal B_\ell^{\Omega_{\rm eq}}}
\widehat L_{n,q}[b]$
and set
$\widehat\Gamma^{\rm FNO}_{N,\ell,\eta}
:=\operatorname{embed}_{N,\eta}(\widehat b_\ell)$.
Then there exists a universal constant $C>0$ such that,
for any $0<\zeta<1$, with probability at least $1-\zeta$,
\[
\begin{aligned}
    L\!\left[\widehat\Gamma^{\rm FNO}_{N,\ell,\eta}\right]
    &\le
    4M
    \left\{
        \frac{1-\delta^\ell}{1-\delta}\,T\eta+a_N(\ell)
    \right\}
    +
    \frac{M^2\delta^{2\ell}}{(1-\delta)^2}
    \\
    &\quad+
    CM^2
    \sqrt{
        \frac{LT(1-\delta^\ell)}{(1-\delta)n}
    }
    +
    CM^2\sqrt{\frac{\log(2/\zeta)}{n}}.
\end{aligned}
\]
\end{theorem}

The statistical term is independent of the Fourier cutoff
and remains bounded as the Picard depth increases.
The proof and sufficient parameter choices for a prescribed
population risk are given in
Appendix~\ref{app:F:proof-torus-fno-generalization}
and Corollary~\ref{cor:F:parameter-scaling}.
Long-time prediction by local rollout is treated in
Appendix~\ref{app:F:long-time-fno}.
See also Appendix~\ref{app:G} for further discussion.

\section{Conclusion}
We developed a generalization theory for Picard-type operator learning for
nonlinear parabolic PDEs. The main point is to use the Duhamel principle and Picard
iteration as an abstract state-transition model and to separate the statistical
complexity of this abstract Picard class from the approximation capability of a
concrete neural-operator implementation. The resulting bounds decompose the
error into implementation, Picard truncation, estimation, and finite-observation
reconstruction terms. Under the contraction assumptions used throughout the
paper, increasing the Picard depth improves the truncation error without
increasing the entropy-based Rademacher term. We also showed how the same local
view extends to rollout over multiple time blocks and how the theory can be
instantiated for nonlinear heat equations on the torus with Picard-type FNOs.

\section*{Acknowledgments}
This work was supported by JSPS KAKENHI 24K21316, 26K17014, 
JST BOOST JPMJBY24E2, JST CREST JPMJCR25I5, JST ERATO JPMJER2601, JST ARiSE JPMJAR2633.


\bibliographystyle{plainnat}
\bibliography{references}


\newpage

\appendix
\section*{Limitations}
The analysis is theoretical and relies on structural assumptions on the target
equation class, most notably local well-posedness, uniform boundedness, and a
contractive Picard map on the time interval under consideration. The
implementation error is kept abstract in the main results, and it must be
estimated separately for each concrete architecture and PDE class; we provide
one such illustration for the nonlinear heat equation on the torus. The
long-time result also depends on stability of the local solution map and on the
rollout staying inside the admissible state set. These assumptions are standard
for this type of analysis, but they limit the direct scope of the bounds for
strongly unstable dynamics or for settings where the learned model is used far
outside the training distribution.

\section*{Broader Impacts}
This work is foundational and does not introduce a deployed system, dataset, or
high-risk model. Its main potential benefit is to clarify when solver-aligned
operator learning can provide reliable surrogate models for PDE-driven
scientific and engineering problems, which may reduce the cost of repeated
numerical simulation. A possible negative impact is overreliance on learned
surrogates in safety-critical physical modeling when the assumptions behind the
generalization bound are not checked. In practical deployments, the bounds
should therefore be complemented by problem-specific validation, uncertainty
assessment, and monitoring of whether inputs remain within the regime covered by
the model assumptions.


\section{Related Work}\label{app:A}

\paragraph{Operator learning architectures.}
Operator learning models include DeepONet \citep{Lu2019,Lanthaler2022},
PCA-Net \citep{Bhattacharya2021,Lanthaler2023a}, graph neural operators
\citep{Li2020b}, Fourier neural operators \citep{Li2020a,Kovachki2021},
wavelet and multiwavelet neural operators \citep{Gupta2021,TC2023}, and
recent variants such as geometric or equivariant neural operators
\citep{Bonev2023,Chen2024}. Universal approximation results are available for
several of these architectures
\citep{ChenChen1993,ChenChen1995,Kovachki2021,Lanthaler2022,Lanthaler2023a,
LanthalerLiStuart2025}. These results establish expressivity for broad operator
classes, but they do not by themselves explain when an operator can be learned
efficiently from finite data or approximated without an exponential dependence
on the ambient function-space dimension.

\paragraph{Approximation theory and structural PDE classes.}
For general continuous or Lipschitz operators, operator learning can suffer from
the curse of parametric complexity \citep{LanthalerStuart2026}. A common way to
avoid this obstruction is to exploit additional structure of the target
operator. Positive approximation results have been obtained for solution
operators of Darcy flow, Navier--Stokes equations, elliptic and parabolic
problems, advection--diffusion equations, Hamilton--Jacobi equations, and
holomorphic operator classes
\citep{Kovachki2021,Chen2023,Lanthaler2023a,Deng2022,MS2023,
SchwabSteinZech2023,Furuya2024,MarcatiSchwab2024,HerrmannSchwabZech2024,
AdcockDexterMoraga2024,LanthalerStuart2026}. These works show that efficient
operator approximation is possible when the target class has exploitable
regularity, spectral structure, low-dimensional compression, holomorphy,
characteristics, or other PDE-specific structure.

\paragraph{Generalization and capacity in operator learning.}
Finite-sample analysis for operator learning has been developed for DeepONets,
FNO-type models, and more general operator classes
\citep{Lanthaler2022,Gopalani2022,KimKang2024,LaraBenitez2024,
KovachkiLanthalerMhaskar2024}. These works typically analyze either the
capacity of a concrete architecture or the sample complexity of a specified
operator approximation class. Our focus is complementary. We do not estimate the
Rademacher complexity of all parameters in a neural operator. Instead, following
the implementation-agnostic perspective of \citet{Sonoda2025depth}, we analyze
the statistical complexity of the abstract Picard transition class that the
implementation is intended to realize. The concrete architecture enters through
an implementation error.

\paragraph{Picard-aligned and fixed-point neural operators.}
The closest approximation-theoretic predecessor is \citet{Furuya2025}. They
rewrite nonlinear parabolic PDEs in Duhamel form and align each forward layer of
a neural operator with one Picard step. Their result shows that the solution
operator can be approximated without exponential growth in model complexity;
the neural network component is used mainly to approximate the scalar
nonlinearity. This gives the approximation side of the present work. We add the
statistical side: once the Picard steps are viewed as a contractive
state-transition system, the Rademacher term can be controlled through the
entropy of the transition family and need not grow with Picard depth.

This perspective is also related to fixed-point models in machine learning,
including deep equilibrium models \citep{Bai2019}, FNO-DEQ models for
steady-state PDEs \citep{Marwah2023}, and neural general operator networks based
on constructive fixed-point principles \citep{FeischlSchwabZehetgruber2025}.
Those works incorporate fixed-point structure into architectures or
approximation schemes. Our contribution is a finite-sample generalization
analysis for the abstract solution class generated by Duhamel--Picard
iteration.

\paragraph{Depth, state transitions, and rollout.}
\citet{Sonoda2025depth} studies when depth improves generalization in an
implementation-agnostic state-transition model. Section~6.2 of that work
identifies Picard-aligned neural operators as a favorable regime: the target
solution operator has an iterative structure, Picard approximation improves
rapidly with depth, and contractive transitions keep word-ball entropy
controlled. The present paper instantiates this observation for nonlinear
parabolic PDEs, adds finite-observation reconstruction errors, and derives
operator-learning generalization bounds. For long-time prediction, related
experimental work studies error accumulation and rollout stabilization in neural
PDE solvers \citep{Michalowska2023,Lippe2023}. Our rollout analysis is a
deterministic propagation result for a learned local solution operator; it does
not add a new statistical complexity term after the local model has been
learned.

\section{Sufficient Conditions and Examples of Admissible Equation Class 
$\Omega_{\rm eq}$}\label{appendix:B}

In this appendix, we give 
a sufficient condition for
Assumption~\ref{ass:contraction}, together with examples of linear
operators $A$, nonlinearities $F$, and admissible equation classes $\Omega_{\rm eq}:=\mathcal A\times\mathcal F$.
We also explain how a common local time can be chosen and discuss
the scope of the present setting of 
solution space. 

To begin with, let us recall Assumption \ref{ass:contraction} of 
admissible equation class 
$\Omega_{\rm eq}$: there, exist $R,M, T, C_S>0$ and $\delta \in (0,1)$ such that the following conditions (a) and (b) hold. 
\begin{enumerate}[(a)]
    \item The map $\mathcal T_{\omega,u_0} : U_M \to U_M$ is $\delta$-contractive, i.e.,
\begin{equation}\label{picard_contraction'}
\|\mathcal T_{\omega,u_0} [u] - \mathcal T_{\omega,u_0}[v]\|_{L^\infty(Q_T)} \le \delta \|u-v\|_{L^\infty(Q_T)}
\end{equation}
for any $u,v \in U_M$.

\item The map $\mathcal T_{\omega,u_0}$ is Lipschitz continuous with respect to $u_0$, i.e.,
\begin{equation}\label{picard_CDoI'}
\|\mathcal T_{\omega,u_0} [u] - \mathcal T_{\omega,v_0}[u]\|_{L^\infty(Q_T)} \le C_S \|u_0-v_0\|_{L^\infty(D)}
\end{equation}
for any $u_0,v_0 \in U_{0,R}$ and $u \in U_M$.
\end{enumerate}

The conditions (a) and (b) 
in
Assumption~\ref{ass:contraction} are the standard conditions used
in Banach's fixed-point theorem to obtain local well-posedness
\citep{Banach1922,Henry1981,Pazy1983,Zeidler}.
Indeed,
Assumption~\ref{ass:contraction} yields Proposition~\ref{prop:LWP}, i.e., the result on local well-posedness for \eqref{P} in $C([0,T]; C(D))$. 

\subsection{A sufficient condition for Assumption \ref{ass:contraction}}
We 
state a typical and useful sufficient condition on $\Omega_{\rm eq}$ that
guarantees Assumption~\ref{ass:contraction}.

\begin{proposition}\label{prop:ass_app}
Let
$R,M,T>0$ and $\delta\in(0,1)$.
Set $C_b(D):=C(D)\cap L^\infty(D)$, equipped with the supremum norm. Let $\mathcal A$ be a class of linear operators with associated
$C_0$-semigroups $S_A(t)$ on $L^2(D)$ such that, for any
$A\in\mathcal A$, the restriction of $S_A(t)$ to $C_b(D)$
is a $C_0$-semigroup on $C_b(D)$, and
\begin{equation}\label{key_L^infty-est}
    C_{\mathcal A}(T)
    :=
    \sup_{A\in\mathcal A}
    \sup_{0<t\le T}
    \|S_A(t)\|_{L^\infty(D)\to L^\infty(D)}
    <\infty.
\end{equation}
Let $\mathcal F$ be a class of functions
$F:\mathbb R\to\mathbb R$ satisfying $F(0)=0$ and
\begin{equation}\label{example_key}
    L_{\mathcal F}(M)
    :=
    \sup_{F\in\mathcal F}
    {\rm Lip}_{[-M,M]}(F)
    <\infty.
\end{equation}
If
\begin{equation}\label{con-map_key}
    C_{\mathcal A}(T)R
    +TC_{\mathcal A}(T)L_{\mathcal F}(M)M
    \le M,
    \qquad
    TC_{\mathcal A}(T)L_{\mathcal F}(M)
    \le\delta,
\end{equation}
then $\Omega_{\rm eq}:=\mathcal A\times\mathcal F$ satisfies
Assumption~\ref{ass:contraction} with $C_S=C_{\mathcal A}(T)$.
\end{proposition}

\begin{proof}
    The inequality \eqref{picard_CDoI'} immediately follows from \eqref{key_L^infty-est}. In fact, 
    by the definition of the Picard map, we have
\[
\mathcal T_{\omega,u_0}[u](t)
    -
    \mathcal T_{\omega,v_0}[u](t) = S_A(t)(u_0-v_0)
    \]
    for $t \in (0,T)$, and hence, 
\[
\begin{split}
    \|\mathcal T_{\omega,u_0}[u]
    -
    \mathcal T_{\omega,v_0}[u]\|_{L^\infty(Q_T)}
    & \le
\sup_{0\le t\le T}
    \|S_A(t)\|_{L^\infty(D)\to L^\infty(D)}
    \|u_0-v_0\|_{L^\infty(D)}\\
    & \le C_{\mathcal A} (T)\|u_0-v_0\|_{L^\infty(D)}
\end{split}
\]
This implies that the condition (b) holds. 

Next, we show that the condition (a) holds. Now, an operator $A \in \mathcal A$ generates a $C_0$-semigroup on $L^2(D)$ and it satisfies the $L^\infty$-estimate \eqref{key_L^infty-est}. Moreover, 
the nonlinear term $F:\mathbb R\to \mathbb R$ in $\mathcal F$ satisfies $F(0)=0$ and local Lipschitz condition in $U_M$ for $M>0$, i.e., there, exists a constant ${\rm Lip}_{[-M,M]}(F)>0$ such that
\begin{equation}\label{nonlinear_localLip}
    |F(z_1) -F(z_2)|\le {\rm Lip}_{[-M,M]}(F)|z_1-z_2|
\end{equation}
for any $z_1,z_2 \in [-M,M]$. 
We take $R,M,T$ and $\delta\in (0,1)$ as 
\begin{equation}\label{con-map1}
    \sup_{0\le t\le T}
    \|S_A(t)\|_{L^\infty(D)\to L^\infty(D)} R + 
    T\sup_{0\le t\le T}\|S_A(t)\|_{L^\infty(D)\to L^\infty(D)}{\rm Lip}_{[-M,M]}(F) M \le M
\end{equation}
and 
\begin{equation}\label{con-map2}
    T\sup_{0\le t\le T}\|S_A(t)\|_{L^\infty(D)\to L^\infty(D)}{\rm Lip}_{[-M,M]}(F) \le \delta.
\end{equation}
For each $u_0 \in U_{0,R}$ and for any 
$u \in U_M$, it follows from $F(0)=0$, 
\eqref{nonlinear_localLip} and 
\eqref{con-map1} that 
\[
\begin{split}
    \|\mathcal T_{\omega,u_0}[u]\|_{L^\infty(Q_T)}
    & \le \|S_A(t)u_0\|_{L^\infty(Q_T)} + 
    \left\|\int_0^t
    S_A(t-\tau) F(u(\tau))\, \tau \right\|_{L^\infty(Q_T)}\\
    & \le \|S_A(t)\|_{L^\infty(D)\to L^\infty(D)} \|u_0\|_{L^\infty(D)}\\
    & \quad + 
    \sup_{0<t\le T}
    \int_0^t
    \|S_A(t-\tau)\|_{L^\infty(D)\to L^\infty(D)}
    \,d\tau \times {\rm Lip}_{[-M,M]}(F) \|u\|_{L^\infty(Q_T)}\\
    & \sup_{0\le t\le T}
    \|S_A(t)\|_{L^\infty(D)\to L^\infty(D)} R \\
    & \quad + T\sup_{0\le t\le T}\|S_A(t)\|_{L^\infty(D)\to L^\infty(D)}{\rm Lip}_{[-M,M]}(F) M \le M.
\end{split}
\]
Moreover, since $S_A(t)$ is the $C_0$-semigroup, $\mathcal T_{\omega,u_0}[u]$ is continuous in $Q_T$ by the standard argument. 
Hence, 
it is shown that $\mathcal T_{\omega,u_0}$ is a map from $U_M$ into itself. Next, for each $u_0 \in U_{0,R}$ and 
for any $u,v\in U_M$, we write
\[
    \mathcal T_{\omega,u_0}[u](t)
    -
    \mathcal T_{\omega,u_0}[v](t)
    =
    \int_0^t
    S_A(t-\tau)
    \bigl(
        F(u(\tau))-F(v(\tau))
    \bigr)
    \,d\tau .
\]
Since $u,v\in U_M$, we have $u(\tau,x),v(\tau,x)\in[-M,M]$ for a.e.
$(\tau,x)\in Q_T$. 
Hence, it follows from \eqref{nonlinear_localLip} and \eqref{con-map2} that 
\[
\begin{split}
    & \|\mathcal T_{\omega,u_0}[u]
    -
    \mathcal T_{\omega,u_0}[v]\|_{L^\infty(Q_T)}\\
    & \le  
    \sup_{0<t\le T}
    \int_0^t
    \|S_A(t-\tau)\|_{L^\infty(D)\to L^\infty(D)}
    \,d\tau \times {\rm Lip}_{[-M,M]}(F)\|u-v\|_{L^\infty(Q_T)} \\
    & \le T\sup_{0\le t\le T}\|S_A(t)\|_{L^\infty(D)\to L^\infty(D)}{\rm Lip}_{[-M,M]}(F) \|u-v\|_{L^\infty(Q_T)}\\
    & \le \delta
    \|u-v\|_{L^\infty(Q_T)}.
\end{split}
\]
This proves \eqref{picard_contraction'}.
Since $S_A(t)$ is strongly continuous on $C_b(D)$ and
$F$ is Lipschitz continuous on $[-M,M]$,
the Duhamel formula gives
$\mathcal T_{\omega,u_0}[u]\in C([0,T];C_b(D))$.
Together with the above norm estimate, this shows that
$\mathcal T_{\omega,u_0}[u]\in U_M$.
Summarizing the above, 
we conclude Proposition \ref{prop:ass_app}. 
\end{proof}

\begin{remark}[Choice of a common local time]
Fix $R,M>0$ and $\delta\in(0,1)$.
Suppose that
\[
    C_{\mathcal A}(T_0)\leq C_S,
    \quad L_{\mathcal F}(M)\leq L,
    \quad C_SR<M
\]
for some $T_0,C_S,L>0$.
Then \eqref{con-map_key} holds whenever
\begin{equation}\label{eq:B:common-local-time}
    0<T\leq
    \min\left\{
        T_0,\,
        \frac{M-C_SR}{C_SLM},\,
        \frac{\delta}{C_SL}
    \right\}.
\end{equation}
Thus, these common bounds give a common local time
for all equations in the class.
If $L_{\mathcal F}(M)=0$, any $0<T\leq T_0$ is allowed.
If the semigroups are $L^\infty$-contractive, we can take
$C_S=1$, so the condition $C_SR<M$ becomes $R<M$.
\end{remark}

\begin{remark}[Continuity and boundary compatibility]
The $L^\infty$-bound \eqref{key_L^infty-est} and strong
continuity on $L^2(D)$ alone do not ensure the continuity
required by $U_M$. This is why Proposition~\ref{prop:ass_app}
also assumes strong continuity on $C_b(D)$.
One may replace $C_b(D)$ by a closed subspace of bounded
continuous functions, equipped with the supremum norm,
on which the semigroups are strongly continuous and which
is preserved by $u\mapsto F(u)$.
For the torus, one uses $C(\mathbb T^d)$.
For homogeneous Dirichlet problems on a sufficiently regular
bounded domain, one uses
\[
    C_0(D):=\{f\in C(\overline D):f|_{\partial D}=0\}.
\]
Since $F(0)=0$, the nonlinearity preserves $C_0(D)$.
The initial-data and solution spaces must then be defined
using this space, and the sampling distribution must be
supported on initial data satisfying the boundary condition.
\end{remark}

\subsection{Examples of linear operators $A$}
\label{app:B:linear}

In this appendix, we give examples of linear operators $A$ satisfying 
\begin{equation}\label{ass:L^infty-contractive}
    \sup_{0\le t\le T}
    \|S_A(t)\|_{L^\infty(D)\to L^\infty(D)}
    <\infty. 
\end{equation}
As shown in the proof of Proposition~\ref{prop:ass_app}, \eqref{ass:L^infty-contractive} is equivalent to the condition (b) with $C_S = 
    \sup_{0\le t\le T}
    \|S_A(t)\|_{L^\infty(D)\to L^\infty(D)}$. 
For simplicity, and to keep the setting consistent with this paper, we state the examples only for $A$ on bounded domains $D$. However, it is known that \eqref{ass:L^infty-contractive} also holds in many cases for $A$ on the whole Euclidean spaces $\mathbb R^d$, unbounded domains, Riemannian manifolds, groups, fractals or general metric measure spaces (see, e.g., \citep{BuiDAnconaDuongMuller2019,BuiDAnconaNicola2020,Davies1989,Furuya2025,
IkedaTaniguchiWakasugi2024,
Ouhabaz2005} and references therein).  

We list several examples below.
\begin{enumerate}[1.]
    \item {\bf (Laplacian on the torus)} 
Let $A=-\Delta_{\mathbb T^d}$ be the Laplacian on
$\mathbb T^d:=\mathbb R^d/(2\pi\mathbb Z)^d$, 
that is, the Laplacian on $D=(0,2\pi)^d$ with periodic boundary condition.
Then $S_A(t)=e^{t\Delta_{\mathbb T^d}}$ is diagonalized by the Fourier basis
\[
    e_\xi(x):=(2\pi)^{-d/2}e^{i\xi\cdot x},
    \quad \xi\in\mathbb Z^d.
\]
More precisely, we have
\[
    S_A(t)f
    =
    \sum_{\xi\in\mathbb Z^d}
    e^{-t|\xi|^2}\widehat f(\xi)e_\xi,
\]
where, $\widehat f$ 
is the Fourier coefficient of $f$:
\[
    \widehat f(\xi)
    :=
    \int_{\mathbb T^d}f(x)\overline{e_\xi(x)}\,dx.
\]
Moreover, the heat semigroup is $L^\infty$-contractive:
\begin{equation}\label{L^infty-contractive}
        \sup_{t>0}
    \|S_A(t)\|_{L^\infty(\mathbb T^d)\to L^\infty(\mathbb T^d)}
    \le 1.
\end{equation}
Therefore, $A=-\Delta_{\mathbb T^d}$ satisfies
\eqref{ass:L^infty-contractive} with $T=\infty$, and we can take $C_S=1$ in (b).

 \item {\bf (Dirichlet Laplacian)}
    Let $D\subset \mathbb R^d$ be a bounded domain, and let
    $A=-\Delta_{\rm Dir}$ be the Dirichlet Laplacian on $D$.
    Then $S_A(t)=e^{t\Delta_{\rm Dir}}$ is  
    $L^\infty$-contractive and has an exponential time decay:
\begin{equation}\label{L^infty-contractive_exp-decay}
        \|S_A(t)\|_{L^\infty(D)\to L^\infty(D)}
    \le
    \min\{1, C e^{-\lambda_1 t}\},\quad t>0
\end{equation}
for some constant $C \ge 1$, where, $\lambda_1$ is the first eigenvalue of $-\Delta_{\rm Dir}$. 
    Thus, $A=-\Delta_{\rm Dir}$ satisfies
\eqref{ass:L^infty-contractive} with $T=\infty$, and we can take $C_S=1$ in (b).
    See, for example, \citep{Davies1989}, 
    \citep[Appendix A.1]{Furuya2025}, \citep[Subsection 5.1]{IkedaTaniguchiWakasugi2024}.

    \item {\bf (Neumann or Robin Laplacian)}
    Let $D\subset\mathbb R^d$ be a bounded Lipschitz domain.
    Let $A=-\Delta_{\gamma}$ be the Robin Laplacian on $D$ associated with
    the boundary condition
    \[
        \partial_\nu u+\gamma u=0
        \quad \text{on } \partial D,
    \]
    where, $\gamma\in L^\infty(\partial D)$ and $\gamma\ge0$.
    The case $\gamma=0$ corresponds to the Neumann Laplacian
    $A=-\Delta_{\rm Neu}$.
    Then $S_A(t)=e^{-tA}$ satisfies $L^\infty$-contractive and has an exponential time decay
    \eqref{L^infty-contractive_exp-decay}, where, one has $\lambda_1=0$
    in the pure Neumann case $\gamma=0$. 
    Therefore, $A=-\Delta_{\gamma}$ satisfies
\eqref{ass:L^infty-contractive} with $T=\infty$, and we can take $C_S=1$ in (b).
    See, for example, \citep{Ouhabaz2005},
    \citep[Appendix A.2]{Furuya2025}.

    \item {\bf (Schr\"odinger operator with a singular potential)}
    Let $D\subset\mathbb R^d$ be a bounded domain, and let $A=-\Delta_{\rm Dir}+V$ 
    be the Schr\"odinger operator with the zero Dirichlet boundary condition,
    where, $V=V_+-V_-$ is a real-valued measurable potential satisfying
    \[
        V_\pm\ge0,\quad
        V_+\in L^1_{\rm loc}(D),\quad
        V_-\in K_d(D),
    \]
    where, $K_d(D)$ denotes the Kato class on $D$.
    Then the semigroup $S_A(t)=e^{-tA}$ satisfies the following $L^\infty$-estimate: 
    For each $T>0$, there, exists a constant
    $C_A(T)>0$ such that
    \begin{equation}\label{L^infty-contractive2}
        \sup_{0<t\le T}
        \|S_A(t)\|_{L^\infty(D)\to L^\infty(D)}
        \le C_A(T).
    \end{equation}
    Hence, $A=-\Delta_{\rm Dir}+V$ satisfies
    \eqref{ass:L^infty-contractive}, and we can take $C_S=C_A(T)$ in (b). If $V\ge 0$ (i.e. 
    $V_- =0$), then 
    \eqref{L^infty-contractive} holds. 
    Therefore, 
    $A=-\Delta_{\rm Dir}+V$ satisfies
\eqref{ass:L^infty-contractive} with $T=\infty$, and we can take $C_S=1$ in (b).
    See, for example, \citep{IwabuchiMatsuyamaTaniguchi2018},
    \citep[Appendix A.3]{Furuya2025}, \citep[Subsection 5.3]{IkedaTaniguchiWakasugi2024}.

    \item {\bf (Schr\"odinger operator with a Dirac delta potential)}
    Let $D=(a,b)\subset\mathbb R$, $x_0\in D$, and $\alpha>0$.
    We consider the one-dimensional Schr\"odinger operator
    \[
        A=-\frac{d^2}{dx^2}+\alpha\delta_{x_0}
    \]
    with the zero Dirichlet boundary condition. More precisely, $A$ is
    realized as the Laplacian on $D\setminus\{x_0\}$ with the interface
    condition $u'(x_0+)-u'(x_0-)=\alpha u(x_0)$. 
    Then $S_A(t)=e^{-tA}$ satisfies the $L^\infty$-contractive estimate 
    \eqref{L^infty-contractive}. 
    Hence, $A=-d^2/dx^2+\alpha\delta_{x_0}$ satisfies
\eqref{ass:L^infty-contractive} with $T=\infty$, and we can take $C_S=1$ in (b).
    See, for example, 
    \citep[Subsection 5.4]{IkedaTaniguchiWakasugi2024}.

    \item {\bf (Uniformly elliptic operator)}
    Let $D\subset\mathbb R^d$ be a bounded domain, and let $A$ be a
    uniformly elliptic operator of divergence form
    \[
        A u
        =
        -\sum_{i,j=1}^d \partial_i\bigl(a_{ij}(x)\partial_j u\bigr)
        +c(x)u
    \]
    with a suitable boundary condition, such as the Dirichlet, Neumann, or
    Robin boundary condition. We assume that $a_{ij}\in L^\infty(D)$,
    $c\in L^\infty(D)$, $c\ge0$, and that there, exists $C>0$ such that
    \[
        \sum_{i,j=1}^d a_{ij}(x)\xi_i\xi_j
        \ge
        C |\xi|^2
    \]
    for all $\xi\in\mathbb R^d$ and a.e. $x\in D$.
    Then $S_A(t)=e^{-tA}$ satisfies the $L^\infty$-contractive
    estimate
    \eqref{L^infty-contractive}
    under the standard maximum principle assumptions. Hence, $A$ satisfies
\eqref{ass:L^infty-contractive} with $T=\infty$, and we can take $C_S=1$ in (b).
    See, for example, \citep{Ouhabaz2005},
    \citep[Appendix A.4]{Furuya2025}, \citep[Subsection 5.5]{IkedaTaniguchiWakasugi2024} for more general cases. 

    \item {\bf (Fractional order or higher order case)}
    We consider the fractional powers $A^{\alpha/2}$, $\alpha>0$, of the above non-negative operators $A$.
    Then we can see that the semigroup $S_{A^{\alpha/2}}(t) = e^{-tA^{\alpha/2}}$ satisfies \eqref{L^infty-contractive} or \eqref{L^infty-contractive2}. Hence, $A^{\alpha/2}$ satisfies 
    \eqref{ass:L^infty-contractive}. 
    See, for example, \citep{Davies1995,Ouhabaz2005}, 
    \citep[Subsection 5.7]{IkedaTaniguchiWakasugi2024} for more details.
\end{enumerate}

\subsection{Examples of nonlinearities}
\label{app:B:nonlinear}

The nonlinear part of Proposition~\ref{prop:ass_app} requires
$F(0)=0$ and a uniform Lipschitz bound
\[
    L_{\mathcal F}(M)
    :=
    \sup_{F\in\mathcal F}
    {\rm Lip}_{[-M,M]}(F)
    <\infty.
\]
We list several examples below.

\begin{enumerate}[1.]
    \item {\bf (Power nonlinearities)}
    Let
    \[
        F(u)=\lambda|u|^{p-1}u
        \quad (\lambda\in\mathbb R,\  p\ge1).
    \]
    Then $F(0)=0$ and
    \[
        {\rm Lip}_{[-M,M]}(F)
        \le |\lambda|pM^{p-1}.
    \]
    Thus, for fixed $p$, a bounded set of coefficients
    $\lambda$ gives a uniform Lipschitz bound.
    The exponent $p$ may also vary in a bounded subset
    of $[1,\infty)$.

    \item {\bf (Polynomial nonlinearities)}
    Let
    \[
        F(u)=\sum_{j=1}^p a_j u^j
        \quad (p\in\mathbb N,\  a_j\in\mathbb R).
    \]
    Then $F(0)=0$ and
    \[
        {\rm Lip}_{[-M,M]}(F)
        \le \sum_{j=1}^p j|a_j|M^{j-1}.
    \]
    Hence, uniformly bounded degrees and coefficients
    give a uniform Lipschitz bound.

    \item {\bf (Allen--Cahn and Fisher--KPP nonlinearities)}
    The nonlinearities
    \[
        F(u)=u-u^3,
        \quad
        F(u)=u-u^2
    \]
    satisfy $F(0)=0$, with Lipschitz bounds
    $1+3M^2$ and $1+2M$, respectively.
    These are special cases of the polynomial nonlinearities
    above.

    \item {\bf (Other smooth nonlinearities)}
    The functions
    \[
        F(u)=\sin u,\quad
        F(u)=\tanh u,\quad
        F(u)=e^u-1
    \]
    satisfy $F(0)=0$, with Lipschitz bounds
    $1$, $1$, and $e^M$, respectively.
    More generally, any class
    $\mathcal F\subset C^1(\mathbb R)$ satisfying
    $F(0)=0$ for any $F\in\mathcal F$ and
    \[
        \sup_{F\in\mathcal F}
        \sup_{|u|\le M}|F'(u)|<\infty
    \]
    has a uniform Lipschitz bound on $[-M,M]$.
\end{enumerate}

\begin{remark}[Time- and space-dependent nonlinearities]
The results in Sections \ref{sec:problem-setting}--\ref{sec:5} extend to nonlinearities
$F=F(t,x,u)$ satisfying $F(t,x,0)=0$ and
\[
    \sup_{F\in\mathcal F}
    \sup_{(t,x)\in[0,T_0]\times D}
    {\rm Lip}_{[-M,M]}\bigl(F(t,x,\cdot)\bigr)
    <\infty
\]
for some $T_0>0$, provided that the required continuity
and boundary conditions are preserved and a common
local time is chosen as in Proposition~\ref{prop:ass_app}.
The equation class and the Picard map must be modified
accordingly. The entropy-based bounds require the
corresponding entropy assumptions on the modified class.
For time-dependent equations, the rollout formulation
in Section~\ref{sec:long-time} must also account for
the dependence of the local solution map on the starting time.
\end{remark}

\begin{remark}[Scope of the criterion]
The criterion applies to reaction terms $F=F(u)$
that are Lipschitz continuous on $[-M,M]$.
It does not directly cover
\[
    F(u)=|u|^p\quad(0<p<1),
    \qquad
    F(u)=u\log|u|,
\]
where, the latter is defined to be zero at $u=0$,
since these functions are not locally Lipschitz at zero.
Derivative-dependent terms, such as $|\nabla u|^p$ with
$p>1$, also lie outside the present setting.
Such equations require separate analysis.
To apply the present framework, one must choose suitable
solution spaces and establish the required contraction
and stability estimates.
\end{remark}

\subsection{Examples of admissible equation classes}
\label{app:B:equation-classes}

We combine the preceding examples to construct admissible
equation classes.
In each case, we use the continuous function space specified
below to define the initial-data and solution spaces.
Choose $0<R<M$, $L>0$, $\delta\in(0,1)$, and $T>0$ such that
\[
    R+TLM\le M,
    \quad
    TL\le\delta.
\]

\begin{enumerate}[1.]
    \item {\bf (Unknown nonlinearity)}
    Fix $A=-\Delta_{\mathbb T^d}$ and set
    \[
        \mathcal F_{M,L}
        :=
        \left\{
            F:[-M,M]\to\mathbb R:
            F(0)=0,\ 
            {\rm Lip}_{[-M,M]}(F)\le L
        \right\},
    \]
    \[
        \Omega_{\rm eq}
        :=
        \{(-\Delta_{\mathbb T^d},F):F\in\mathcal F_{M,L}\}.
    \]
    Each $F$ may be extended to $\mathbb R$ without
    increasing its Lipschitz constant.
    We use $C(\mathbb T^d)$ as the continuous function space.
    The heat semigroup is strongly continuous on this space
    and is $L^\infty$-contractive.
    Thus, the assumptions of Proposition~\ref{prop:ass_app}
    hold with $C_{\mathcal A}(T)=1$ and
    $L_{\mathcal F}(M)\le L$.
    Hence, Assumption~\ref{ass:contraction} holds with $C_S=1$.

    Here, the unknown equation specification is the
    nonlinearity.
    This class is studied in Section~\ref{sec:torus-example}
    and Appendix~\ref{app:F:torus-setting}.

    \item {\bf (Unknown boundary condition)}
    Let $D$ be a bounded smooth domain.
    Fix a locally Lipschitz function $F:\mathbb R\to\mathbb R$
    with $F(0)=0$ and ${\rm Lip}_{[-M,M]}(F)\le L$.
    For a fixed $\gamma_0>0$, set
    \[
        \mathcal R_{\gamma_0}
        :=
        \{\gamma\in C^1(\partial D):0\le\gamma\le\gamma_0\},
        \quad
        \Omega_{\rm eq}
        :=
        \{(-\Delta_\gamma,F):\gamma\in\mathcal R_{\gamma_0}\},
    \]
    where, $-\Delta_\gamma$ is the Robin Laplacian with
    boundary condition
    \[
        \partial_\nu u+\gamma u=0
        \quad\text{on }\partial D.
    \]
    We use $C(\overline D)$ as the continuous function space.
    The Robin heat semigroups are strongly continuous on
    this space and are $L^\infty$-contractive.
    Thus, the assumptions of Proposition~\ref{prop:ass_app}
    hold with $C_{\mathcal A}(T)=1$ and
    $\mathcal F=\{F\}$.
    Hence, Assumption~\ref{ass:contraction} holds with $C_S=1$.

    Here, the unknown equation specification is the
    boundary coefficient $\gamma$.

    \item {\bf (Unknown coefficients)}
    Let $D$ be a bounded smooth domain, and define
    \[
        A_{a,c}u
        :=
        -\sum_{i,j=1}^d
        \partial_i\bigl(a_{ij}(x)\partial_j u\bigr)
        +c(x)u
    \]
    with homogeneous conormal Neumann boundary condition
    \[
        (a\nabla u)\cdot\nu=0
        \quad\text{on }\partial D.
    \]
    Let $\mathcal C$ be a class of real smooth coefficients
    $(a,c)$ on $\overline D$ such that $a=a^\top$ and
    \[
        \lambda|\xi|^2
        \le
        \sum_{i,j=1}^d a_{ij}(x)\xi_i\xi_j
        \le\Lambda|\xi|^2,
        \quad
        0\le c(x)\le C_0
    \]
    for any $x\in\overline D$ and $\xi\in\mathbb R^d$,
    where, $0<\lambda\le\Lambda$ and $C_0>0$ are fixed.
    Fix a locally Lipschitz function $F:\mathbb R\to\mathbb R$
    with $F(0)=0$ and ${\rm Lip}_{[-M,M]}(F)\le L$, and set
    \[
        \Omega_{\rm eq}
        :=
        \{(A_{a,c},F):(a,c)\in\mathcal C\}.
    \]
    We use $C(\overline D)$ as the continuous function space.
    The associated semigroups are strongly continuous on
    this space and are $L^\infty$-contractive.
    Thus, the assumptions of Proposition~\ref{prop:ass_app}
    hold with $C_{\mathcal A}(T)=1$ and
    $\mathcal F=\{F\}$.
    Hence, Assumption~\ref{ass:contraction} holds with $C_S=1$.

    Here, the unknown equation specification is given by
    the coefficients $(a,c)$.
\end{enumerate}

Therefore, these classes satisfy Assumption~\ref{ass:contraction} and hence, 
the equation-class
assumption of Theorem~\ref{thm:generalization}.

\subsection{Scope and possible extensions}
\label{app:B:scope}

Proposition~\ref{prop:ass_app} gives uniform Picard contraction
on a sufficiently short time interval under common semigroup
and local Lipschitz bounds, together with the required
continuity assumptions.
The nonlinearity itself need not be small.
More generally, in many scaling-subcritical settings,
the contraction factor can be made small by shortening
the time interval.
Uniform contraction over an equation class, however,
requires uniform bounds on the constants in these estimates.

The present analysis uses continuous solution spaces
equipped with the $L^\infty_{t,x}$ norm.
This setting is natural for many semilinear parabolic
equations \citep{Henry1981,Pazy1983}.
For nonlinear wave equations or the Navier--Stokes equations,
mixed spaces $L_t^rL_x^s$ or Sobolev spaces may be more
suitable
\citep{GigaMiyakawa1985,Cannone2004,CazenaveHaraux1998,Tao2006}.
Extending the framework to these spaces requires
contraction, stability, and complexity estimates for
the corresponding Picard class, together with suitable
observation and readout maps.

The criterion directly covers reaction terms $F=F(u)$
that are locally Lipschitz.
Derivative-dependent terms, such as $|\nabla u|^p$,
require additional regularity and different estimates
for the Duhamel map.
For viscous Hamilton--Jacobi equations, related fixed point
arguments in suitable subcritical spaces are available;
see \citet{BenArtziSoupletWeissler2002}.
Sublinear terms $|u|^p$ with $0<p<1$ and logarithmic terms
$u\log|u|$ are not locally Lipschitz at zero.
Their well-posedness and contraction properties therefore
require separate analysis and do not follow from
Proposition~\ref{prop:ass_app}.
Depending on the equation, other methods, such as
monotonicity or regularization, may be needed
\citep{CazenaveHaraux1998,Pao1992,CrandallLiggett1971,Barbu2010}.

Finally, short-time Picard contraction alone does not
guarantee stable long-time prediction.
Section~\ref{sec:long-time} analyzes repeated use of
a fixed local model under additional conditions on
rollout admissibility and local errors.
Extensions to other iterative methods are discussed
in Appendix~\ref{app:G}.

\section{Picard Truncation Error Bounds: Proofs of Lemmas \ref{lem:picard_error1} and \ref{lem:picard_error2}}\label{app:C}

In this appendix we prove Lemmas~4.1 and~4.2. We write
$d(u,v):=\|u-v\|_{L^\infty(Q_T)}$ for $u,v\in U_M$.

\begin{lemma}[Lemma \ref{lem:picard_error1}]\label{lem:picard_error1-app}
Suppose that Assumption \ref{ass:contraction} holds. Then, 
for any $\omega\in\Omega_{\rm eq}$ and $\ell \in \mathbb N$, we have
\begin{equation}\label{C.1_1}
\sup_{v\in U_{0,R}} \|b_{\omega, \ell}(v,\cdot)-\mathcal G_{\omega}[v]\|_{L^\infty(Q_T)}\le \frac{M\delta^\ell}{1-\delta}.
\end{equation}
Consequently, 
\begin{equation}\label{C.1_2}
    \inf_{b\in\mathcal B^{\Omega_{\rm eq}}_\ell} L[b]
    \le
    \frac{M^2\delta^{2\ell}}{(1-\delta)^2}.
\end{equation}
\end{lemma}

\begin{proof}[Proof of Lemma~4.1]
Fix $\omega\in\Omega_{\rm eq}$ and $v\in U_{0,R}$. Set
$u^{(0)}:=0$ and
\[
    u^{(\ell+1)}:=\mathcal T_{\omega,v}[u^{(\ell)}],
    \quad \ell=0,1,2,\cdots .
\]
Then, $b_{\omega,\ell}(v,\cdot)=u^{(\ell)}
    =
    \mathcal T_{\omega,v}^{[\ell]}[0]$ 
by the definition of $b_{\omega,\ell}$.
Let $u:=\mathcal G_\omega[v]$ be the unique fixed point of
$\mathcal T_{\omega,v}$. Since $\mathcal T_{\omega,v}$ is
$\delta$-contractive on $U_M$, we have
\[
    d(u^{(r+1)},u^{(r)})
    \le
    \delta^r d(u^{(1)},u^{(0)}),
    \quad r=0,1,2,\cdots .
\]
Therefore, for $N>\ell$, we estimate
\[
\begin{aligned}
    d(u^{(N)},u^{(\ell)})
    \le
    \sum_{r=\ell}^{N-1} d(u^{(r+1)},u^{(r)})  
    \le
    \sum_{r=\ell}^{N-1} \delta^r d(u^{(1)},u^{(0)})  
    \le
    \frac{\delta^\ell}{1-\delta} d(u^{(1)},u^{(0)}).
\end{aligned}
\]
Letting $N\to\infty$ and using $u^{(N)}\to u$ in $L^\infty(Q_T)$, we obtain
\[
    d(u^{(\ell)},u)
    \le
    \frac{\delta^\ell}{1-\delta} d(u^{(1)},u^{(0)}).
\]
Since $u^{(1)}=\mathcal T_{\omega,v}[0]\in U_M$ and $u^{(0)}=0$, we have
$d(u^{(1)},u^{(0)})=\|u^{(1)}\|_{L^\infty(Q_T)}\le M$. Hence
\[
    \|b_{\omega,\ell}(v,\cdot)-\mathcal G_\omega[v]\|_{L^\infty(Q_T)}
    \le
    \frac{M\delta^\ell}{1-\delta}.
\]
Taking the supremum over $v\in U_{0,R}$ gives \eqref{C.1_1}. 

We next prove the bound \eqref{C.1_2} for the oracle term. Let $\omega_\ast\in\Omega_{\rm eq}$
be such that $\mathcal G_\ast=\mathcal G_{\omega_\ast}$. Since
$b_{\omega_\ast,\ell}\in\mathcal B^{\Omega_{\rm eq}}_\ell$, we have 
\[
    \inf_{b\in\mathcal B^{\Omega_{\rm eq}}_\ell} L[b]
    \le
    L[b_{\omega_\ast,\ell}]  
    =
    \mathbb E_{u_0\sim\Pi,z\sim\mu}
    \left[
    \left|
    b_{\omega_\ast,\ell}(u_0,z)
    -
    \mathcal G_{\omega_\ast}[u_0](z)
    \right|^2
    \right] 
    \le
    \left(
    \frac{M\delta^\ell}{1-\delta}
    \right)^2 .
\]
Thus, \eqref{C.1_2} is shown. 
The proof of Lemma \ref{lem:picard_error1-app} is finished.
\end{proof}

\begin{lemma}[Lemma \ref{lem:picard_error2}]\label{lem:picard_error2-app}
Suppose that Assumption \ref{ass:contraction} holds. Then, 
for any $\omega\in\Omega_{\rm eq}$, $E\in\mathcal E_m$, $u_0 \in U_0$ and $\ell \in \mathbb N$, we have
\begin{equation}\label{C.2_1}
\|b_{\omega,E,\ell}(P_m u_0,\cdot)-\mathcal G_{\omega}[u_0]\|_{L^\infty(Q_T)}\le \frac{M\delta^\ell}{1-\delta}+\frac{C_S}{1-\delta}\|E(P_m u_0)-u_0\|_{L^\infty(D)}.
\end{equation}
Consequently, 
\begin{equation}\label{C.2_2}
    \inf_{b\in\mathcal B^{\Omega_{\rm eq}}_{m,\ell}} L_m[b]
    \le
    \frac{2M^2\delta^{2\ell}}{(1-\delta)^2}
    +
    \frac{2C_S^2}{(1-\delta)^2}
    \varepsilon_{\rm rec}^2(m).
\end{equation}
\end{lemma}

\begin{proof}[Proof of Lemma~4.2]
Fix $\omega\in\Omega_{\rm eq}$, $E\in\mathcal E_m$, and $u_0\in U_0$. Put
\[
    \widetilde u_0:=E(P_m u_0).
\]
Since $E:\mathbb R^m\to U_0$ and $U_0\subset U_{0,R}$, we have
$\widetilde u_0\in U_{0,R}$. By the definition of the finite-observation
Picard model,
\[
    b_{\omega,E,\ell}(P_m u_0,\cdot)
    =
    \mathcal T_{\omega,\widetilde u_0}^{[\ell]}[0].
\]
Using Lemma \ref{lem:picard_error1-app} with initial data $\widetilde u_0$, we obtain
\[
    \left\|
    b_{\omega,E,\ell}(P_m u_0,\cdot)
    -
    \mathcal G_\omega[\widetilde u_0]
    \right\|_{L^\infty(Q_T)}
    \le
    \frac{M\delta^\ell}{1-\delta}.
\]
On the other hand, by \eqref{CDoI} (i.e. 
the continuous dependence on initial data in Proposition \ref{prop:LWP}), we have 
\[
    \left\|
    \mathcal G_\omega[\widetilde u_0]
    -
    \mathcal G_\omega[u_0]
    \right\|_{L^\infty(Q_T)}
    \le
    \frac{C_S}{1-\delta}
    \|\widetilde u_0-u_0\|_{L^\infty(D)}.
\]
Therefore, by the triangle inequality, we obtain \eqref{C.2_1}:
\[
\begin{aligned}
    &
    \left\|
    b_{\omega,E,\ell}(P_m u_0,\cdot)
    -
    \mathcal G_\omega[u_0]
    \right\|_{L^\infty(Q_T)}
    \\
    & \qquad \le
    \left\|
    b_{\omega,E,\ell}(P_m u_0,\cdot)
    -
    \mathcal G_\omega[\widetilde u_0]
    \right\|_{L^\infty(Q_T)}
    +
    \left\|
    \mathcal G_\omega[\widetilde u_0]
    -
    \mathcal G_\omega[u_0]
    \right\|_{L^\infty(Q_T)}
    \\
    &\qquad\le
    \frac{M\delta^\ell}{1-\delta}
    +
    \frac{C_S}{1-\delta}
    \|E(P_m u_0)-u_0\|_{L^\infty(D)}.
\end{aligned}
\]

We now prove the corresponding oracle bound \eqref{C.2_2}. Let $\omega_\ast\in\Omega_{\rm eq}$
be such that $\mathcal G_\ast=\mathcal G_{\omega_\ast}$. For each
$E\in\mathcal E_m$, the preceding estimate gives
\[
    \left|
    b_{\omega_\ast,E,\ell}(P_m u_0,z)
    -
    \mathcal G_\ast[u_0](z)
    \right|^2
\le
    2\left(
    \frac{M\delta^\ell}{1-\delta}
    \right)^2
    +
    2\left(
    \frac{C_S}{1-\delta}
    \right)^2
    \|E(P_m u_0)-u_0\|_{L^\infty(D)}^2
\]
for any $u_0\in U_0$ and $z\in Q_T$. Taking expectation with respect to
$u_0\sim\Pi$ and $z\sim\mu$, we obtain
\[
\begin{aligned}
    L_m[b_{\omega_\ast,E,\ell}]
    &\le
    \frac{2M^2\delta^{2\ell}}{(1-\delta)^2}
    +
    \frac{2C_S^2}{(1-\delta)^2}
    \mathbb E_{u_0\sim\Pi}
    \left[
    \|E(P_m u_0)-u_0\|_{L^\infty(D)}^2
    \right].
\end{aligned}
\]
Since $b_{\omega_\ast,E,\ell}\in\mathcal B^{\Omega_{\rm eq}}_{m,\ell}$, it
follows that
\[
\begin{aligned}
    \inf_{b\in\mathcal B^{\Omega_{\rm eq}}_{m,\ell}} L_m[b]
    &\le
    \inf_{E\in\mathcal E_m} L_m[b_{\omega_\ast,E,\ell}]  \\
    &\le
    \frac{2M^2\delta^{2\ell}}{(1-\delta)^2}
    +
    \frac{2C_S^2}{(1-\delta)^2}
    \inf_{E\in\mathcal E_m}
    \mathbb E_{u_0\sim\Pi}
    \left[
    \|E(P_m u_0)-u_0\|_{L^\infty(D)}^2
    \right].
\end{aligned}
\]
By the definition of $\varepsilon_{\rm rec}^2(m)$, this yields
\eqref{C.2_2}. 
The proof of Lemma \ref{lem:picard_error2-app} is finished.
\end{proof}

\section{Generalization Error Bounds: Proofs of Theorem \ref{thm:generalization} and Corollary \ref{cor:practical-scaling}}\label{app:D}

In this appendix we prove Theorem~\ref{thm:generalization} and
Corollary~\ref{cor:practical-scaling}. We first record the uniform deviation
estimate used in the proof. The statement is formulated at the trajectory-block
level so that it applies both to the full-information case and to the
finite-observation case.

\begin{remark}[Uniform deviation estimate for the clipped loss]
\label{rem:uniform-deviation}
Let $\mathcal X$ be an input space and let
$\mathcal B\subset\{b:\mathcal X\times Q_T\to[-M,M]\}$.
Let $W$ be a latent trajectory variable with law $\nu$, let
$A(W)\in\mathcal X$ be the model input, and let
$y(W,z)\in[-M,M]$ be the target value at the query point $z$. For independent
blocks $(W_i,z_{i,1},\ldots,z_{i,q})\sim \nu\otimes\mu^{\otimes q}$, set
$a_i=A(W_i)$ and $y_{i,j}=y(W_i,z_{i,j})$. For the trajectory-level sample
$\mathcal S_{n,q}^{\mathcal X}
=
\{(a_i,\{z_{i,j}\}_{j=1}^q)\}_{i=1}^n$, define the pointwise empirical
Rademacher complexity by
\[
    \widehat{\mathfrak R}^{\rm pt}_{\mathcal S_{n,q}^{\mathcal X}}(\mathcal B)
    :=
    \mathbb E_\sigma
    \left[
    \sup_{b\in\mathcal B}
    \frac1{nq}
    \sum_{i=1}^n
    \sum_{j=1}^q
    \sigma_{i,j} b(a_i,z_{i,j})
    \right],
\]
where $\{\sigma_{i,j}:1\le i\le n,1\le j\le q\}$ are independent Rademacher
variables. As in Section~5, we use the trajectory-normalized version
\[
    \widehat{\mathfrak R}_{\mathcal S_{n,q}^{\mathcal X}}(\mathcal B)
    :=
    \sqrt q\,\widehat{\mathfrak R}^{\rm pt}_{\mathcal S_{n,q}^{\mathcal X}}(\mathcal B)
    =
    \mathbb E_\sigma
    \left[
    \sup_{b\in\mathcal B}
    \frac1{n\sqrt q}
    \sum_{i=1}^n
    \sum_{j=1}^q
    \sigma_{i,j} b(a_i,z_{i,j})
    \right].
\]
Define
\[
    L_{\mathcal X}[b]
    :=
    \mathbb E_{W\sim\nu,z\sim\mu}
    \ell_{\Lambda_{\mathrm{clip}}}\bigl(b(A(W),z),y(W,z)\bigr)
\]
and
\[
    \widehat L_{\mathcal X,n,q}[b]
    :=
    \frac1n\sum_{i=1}^n\frac1q\sum_{j=1}^q
    \ell_{\Lambda_{\mathrm{clip}}}\bigl(b(a_i,z_{i,j}),y_{i,j}\bigr).
\]
Then there exists a universal constant $C>0$ such that, with probability at
least $1-\zeta$,
\[
    \sup_{b\in\mathcal B}
    \left|
    L_{\mathcal X}[b]-\widehat L_{\mathcal X,n,q}[b]
    \right|
    \le
    C M
    \widehat{\mathfrak R}_{\mathcal S_{n,q}^{\mathcal X}}(\mathcal B)
    +
    C M^2
    \sqrt{\frac{\log(2/\zeta)}{n}} .
\]
Indeed, symmetrization is applied to the $n$ independent trajectory blocks.
Conditional on the sample, define
\[
    \phi_i(r_1,\ldots,r_q)
    :=
    \frac1q
    \sum_{j=1}^q
    \ell_{\Lambda_{\mathrm{clip}}}(r_j,y_{i,j}) .
\]
Since the clipped squared loss is $CM$-Lipschitz in the scalar prediction on
$[-M,M]$, each $\phi_i$ is $CM/\sqrt q$-Lipschitz with respect to the Euclidean
norm on $\mathbb R^q$. Maurer's vector contraction inequality
\citep{Maurer2016} therefore controls the Rademacher complexity of the clipped
loss class by
$CM\widehat{\mathfrak R}_{\mathcal S_{n,q}^{\mathcal X}}(\mathcal B)$.
The last term follows from bounded-difference concentration for the block
losses, which are bounded by $CM^2$.
\end{remark}

\begin{proof}[Proof of Theorem~\ref{thm:generalization}]
We first prove the full-information bound. Put
\[
    \mathcal B_\ell:=\mathcal B^{\Omega_{\rm eq}}_\ell .
\]
Since all predictors are clipped to $[-M,M]$ and
$\mathcal G_\ast[u_0]\in U_M$, the target values
$\mathcal G_\ast[u_0](z)$ also belong to $[-M,M]$. Hence the clipped squared
loss is bounded by $4M^2$ and is $4M$-Lipschitz in the first variable.

Apply Remark~\ref{rem:uniform-deviation} with $W=u_0$, $A(W)=u_0$, and
$y(W,z)=\mathcal G_\ast(u_0)(z)$. Then, with probability at least $1-\zeta$,
we have
\[
    \sup_{b\in\mathcal B_\ell}
    \left|
    L[b]-\widehat L_{n,q}[b]
    \right|
    \le
    C M\widehat{\mathfrak R}^{\Omega_{\rm eq}}
    +
    C M^2\sqrt{\frac{\log(2/\zeta)}{n}} .
\]
Denote the right-hand side by $\Delta_\ell$. Thus, on this event,
\[
    L[b]\le \widehat L_{n,q}[b]+\Delta_\ell,
    \qquad
    \widehat L_{n,q}[b]\le L[b]+\Delta_\ell
\]
for every $b\in\mathcal B_\ell$.

Let
$\widehat b_\ell\in\operatorname*{argmin}_{b\in\mathcal B_\ell}
\widehat L_{n,q}[b]$ and
$\widehat\Gamma_\ell:=\operatorname{embed}(\widehat b_\ell)$.
By the definition of the implementation error,
\[
    L[\widehat\Gamma_\ell]
    =
    L[\operatorname{embed}(\widehat b_\ell)]
    \le
    L[\widehat b_\ell]+\varepsilon_{\rm imp}.
\]
Using the uniform deviation estimate and the empirical minimality of
$\widehat b_\ell$, we obtain, for every $b\in\mathcal B_\ell$,
\[
\begin{aligned}
    L[\widehat b_\ell]
    \le
    \widehat L_{n,q}[\widehat b_\ell]+\Delta_\ell \le
    \widehat L_{n,q}[b]+\Delta_\ell \le
    L[b]+2\Delta_\ell .
\end{aligned}
\]
Taking the infimum over $b\in\mathcal B_\ell$ gives
\[
    L[\widehat b_\ell]
    \le
    \inf_{b\in\mathcal B_\ell}L[b]
    +
    2\Delta_\ell .
\]
Therefore,
\[
    L[\widehat\Gamma_\ell]
    \le
    \varepsilon_{\rm imp}
    +
    \inf_{b\in\mathcal B_\ell}L[b]
    +
    2\Delta_\ell .
\]
By Lemma~\ref{lem:picard_error1},
\[
    \inf_{b\in\mathcal B_\ell}L[b]
    \le
    \frac{M^2\delta^{2\ell}}{(1-\delta)^2}.
\]
Absorbing the factor $2$ into the universal constant $C$, we obtain
\[
    L[\widehat\Gamma_\ell]
    \le
    \varepsilon_{\rm imp}
    +
    \frac{M^2\delta^{2\ell}}{(1-\delta)^2}
    +
    C M\widehat{\mathfrak R}^{\Omega_{\rm eq}}
    +
    C M^2\sqrt{\frac{\log(2/\zeta)}{n}} .
\]
This proves part {\rm (i)}.

We next prove the finite-observation bound. Put
\[
    \mathcal B_{m,\ell}:=\mathcal B^{\Omega_{\rm eq}}_{m,\ell}.
\]
Apply the same uniform deviation estimate with $W=u_0$, $A(W)=P_m u_0$, and
$y(W,z)=\mathcal G_\ast(u_0)(z)$ to the input space $P_m(U_0)$ and the class
$\mathcal B_{m,\ell}$. Then, with probability at least $1-\zeta$,
\[
    \sup_{b\in\mathcal B_{m,\ell}}
    \left|
    L_m[b]-\widehat L_{m,n,q}[b]
    \right|
    \le
    C M\widehat{\mathfrak R}^{\Omega_{\rm eq}}_m
    +
    C M^2\sqrt{\frac{\log(2/\zeta)}{n}} .
\]
Denote the right-hand side by $\Delta_{m,\ell}$. Let
$\widehat b_{m,\ell}
\in
\operatorname*{argmin}_{b\in\mathcal B_{m,\ell}}
\widehat L_{m,n,q}[b]$ and
$\widehat\Gamma_{m,\ell}:=\operatorname{embed}_m(\widehat b_{m,\ell})$.
Then, by the definition of $\varepsilon_{\rm imp}^m$,
\[
    L_m[\widehat\Gamma_{m,\ell}]
    =
    L_m[\operatorname{embed}_m(\widehat b_{m,\ell})]
    \le
    L_m[\widehat b_{m,\ell}]
    +
    \varepsilon_{\rm imp}^m .
\]
As before, the empirical minimality of $\widehat b_{m,\ell}$ implies that, for
every $b\in\mathcal B_{m,\ell}$,
\[
\begin{aligned}
    L_m[\widehat b_{m,\ell}]
    \le
    \widehat L_{m,n,q}[\widehat b_{m,\ell}]
    +
    \Delta_{m,\ell} \le
    \widehat L_{m,n,q}[b]
    +
    \Delta_{m,\ell} \le
    L_m[b]
    +
    2\Delta_{m,\ell}.
\end{aligned}
\]
Taking the infimum over $b\in\mathcal B_{m,\ell}$, we get
\[
    L_m[\widehat b_{m,\ell}]
    \le
    \inf_{b\in\mathcal B_{m,\ell}}L_m[b]
    +
    2\Delta_{m,\ell}.
\]
Thus
\[
    L_m[\widehat\Gamma_{m,\ell}]
    \le
    \varepsilon_{\rm imp}^m
    +
    \inf_{b\in\mathcal B_{m,\ell}}L_m[b]
    +
    2\Delta_{m,\ell}.
\]
By Lemma~\ref{lem:picard_error2},
\[
    \inf_{b\in\mathcal B_{m,\ell}}L_m[b]
    \le
    \frac{2M^2\delta^{2\ell}}{(1-\delta)^2}
    +
    \frac{2C_S^2}{(1-\delta)^2}
    \varepsilon_{\rm rec}^2(m).
\]
Therefore,
\[
    L_m[\widehat\Gamma_{m,\ell}]
    \le
    \varepsilon_{\rm imp}^m
    +
    \frac{2M^2\delta^{2\ell}}{(1-\delta)^2}
    +
    \frac{2C_S^2}{(1-\delta)^2}
    \varepsilon_{\rm rec}^2(m)
    +
    C M\widehat{\mathfrak R}^{\Omega_{\rm eq}}_m
    +
    C M^2\sqrt{\frac{\log(2/\zeta)}{n}} .
\]
This proves part {\rm (ii)}.

If one wants the two estimates to hold simultaneously, one applies the above
two deviation estimates with confidence parameter $\zeta/2$ and absorbs the
resulting harmless factor into the universal constant $C$.
\end{proof}

\begin{proof}[Proof of Corollary~\ref{cor:practical-scaling}]
We first consider the full-information case. By the entropy bound stated in
Section~\ref{sec:5}, we have
\[
    \widehat{\mathfrak R}^{\Omega_{\rm eq}}
    \le
    \frac{C}{\sqrt n}\mathcal I_\Omega,
\]
where
\[
    \mathcal I_\Omega
    :=
    \int_0^{2M}
    \sqrt{\mathcal H_\Omega(c(1-\delta)\varepsilon)}
    \,d\varepsilon
    <\infty .
\]
Let
\[
    \ell_n
    :=
    \left\lceil
    \frac{\log n}{4|\log\delta|}
    \right\rceil .
\]
Since $0<\delta<1$, we have $|\log\delta|=-\log\delta$. Hence
\[
    \delta^{2\ell_n}
    =
    \exp(-2\ell_n|\log\delta|)
    \le
    \exp\left(-\frac12\log n\right)
    =
    n^{-1/2}.
\]
Therefore,
\[
    \frac{M^2\delta^{2\ell_n}}{(1-\delta)^2}
    \le
    C\frac{M^2}{\sqrt n},
\]
where the constant $C$ may depend on $\delta$, but not on $n$. Substituting
these estimates into Theorem~\ref{thm:generalization} gives
\[
    L[\widehat\Gamma_{\ell_n}]
    \le
    \varepsilon_{\rm imp}
    +
    C
    \left(
    \frac{M^2}{\sqrt n}
    +
    \frac{M\mathcal I_\Omega}{\sqrt n}
    +
    M^2\sqrt{\frac{\log(2/\zeta)}{n}}
    \right).
\]
This proves the full-information estimate.

We next consider the finite-observation case. Assume that
\[
    \varepsilon_{\rm rec}^2(m)\le C_{\rm rec}^2 m^{-2\beta}
\]
and
\[
    \mathcal I_{\Omega,m}\le C_I m^\alpha .
\]
Using the entropy estimate
$\widehat{\mathfrak R}^{\Omega_{\rm eq}}_m
\le
C\mathcal I_{\Omega,m}/\sqrt n$, Theorem~\ref{thm:generalization} gives
\[
\begin{aligned}
    L_m[\widehat\Gamma_{m,\ell}]
    \le\;&
    \varepsilon_{\rm imp}^{m}
    +
    \frac{2M^2\delta^{2\ell}}{(1-\delta)^2}
    +
    \frac{2C_S^2C_{\rm rec}^2}{(1-\delta)^2}m^{-2\beta}
\\
    &\quad
    +
    C M C_I \frac{m^\alpha}{\sqrt n}
    +
    C M^2\sqrt{\frac{\log(1/\zeta)}{n}} .
\end{aligned}
\]
We now take $\ell=\ell_n$ and $m=m_n$, where
\[
    m_n\asymp n^{1/(2(2\beta+\alpha))}.
\]
The choice of $\ell_n$ gives
\[
    \delta^{2\ell_n}\le n^{-1/2}.
\]
Moreover, by the choice of $m_n$,
\[
    m_n^{-2\beta}
    \asymp
    n^{-\beta/(2\beta+\alpha)}
\]
and
\[
    \frac{m_n^\alpha}{\sqrt n}
    \asymp
    n^{\alpha/(2(2\beta+\alpha))}n^{-1/2}
    =
    n^{-\beta/(2\beta+\alpha)}.
\]
Since $\alpha\ge0$, we have
$\beta/(2\beta+\alpha)\le 1/2$, and hence
$n^{-1/2}\le n^{-\beta/(2\beta+\alpha)}$ for $n\ge1$. Therefore the Picard
truncation term is also bounded by a constant multiple of
$n^{-\beta/(2\beta+\alpha)}$. Consequently,
\[
    L_{m_n}[\widehat\Gamma_{m_n,\ell_n}]
    \le
    \varepsilon_{\rm imp}^{m_n}
    +
    C
    \left(
    n^{-\beta/(2\beta+\alpha)}
    +
    M^2\sqrt{\frac{\log(2/\zeta)}{n}}
    \right),
\]
where $C$ may depend on $M,C_S,\delta,C_{\rm rec},C_I$, but not on $n$.
This proves the finite-observation estimate.

For quasi-uniform sensors and stable reconstruction on
$U_0\subset H^{s_0}(D)$ with $s_0>d/2$, the usual Sobolev reconstruction rate
gives
\[
    \beta=\frac{s_0-d/2}{d}.
\]
This completes the proof of Corollary~\ref{cor:practical-scaling}.
\end{proof}

\section{Long-Time Prediction: Proofs of Theorem \ref{thm:blockwise-long-time}}\label{app:E}

In this appendix, we consider long-time prediction based on the learned local model.

The idea is based on the standard argument for extending local-in-time
solutions. First, let us briefly recall the local-in-time solution theory.
Suppose that $\Omega_{\rm eq}:=\mathcal A\times\mathcal F$ is an admissible
equation class satisfying the assumptions of Proposition~\ref{prop:ass_app}.
Then, for any $R,M$ with $0<R<M$ and any $\delta\in(0,1)$, we can choose
$T=T(R,M,\delta)>0$ sufficiently small so that \eqref{con-map_key} holds.
Hence, by Proposition~\ref{prop:ass_app}, Assumption~\ref{ass:contraction}
holds for this choice of $R,M,\delta$ and $T$. Consequently, the local
well-posedness result in Proposition~\ref{prop:LWP} follows from Banach's fixed
point theorem. In other words, for each $\omega\in\Omega_{\rm eq}$, there
exists a unique solution operator $\mathcal G_\omega$ which maps an initial
datum $u_0\in U_{0,R}$ to the corresponding solution $u$ on $[0,T]$ of
\eqref{P'}.

If the terminal value $u(T)$ belongs to $U_{0,R}$, that is,
$\|u(T)\|_{L^\infty(D)}\le R$, then $u(T)$ can be used as a new initial data.
Thus, the solution can be extended to $[T,2T]$ by applying the same local
solution operator $\mathcal G_\omega$ again. More precisely, after the time
shift $s=t-T$, the solution on $[T,2T]$ is given by
\[
    u(T+s)=\mathcal G_\omega[u(T)](s),
    \qquad 0\le s\le T.
\]
Equivalently, if we define the terminal map
$\psi(u):=u(T)$ and $\Phi_\omega:=\psi\circ\mathcal G_\omega$, then the
second-step initial data is $\Phi_\omega(u_0)$, and the solution on the second
time block is
\[
    \mathcal G_\omega[\Phi_\omega(u_0)].
\]
Repeating this procedure as long as the terminal values remain in $U_{0,R}$,
we can extend the solution successively over longer time intervals. The idea of
the long-time prediction result is to use this standard continuation mechanism
at the level of the learned local solution operator.

This idea has three main advantages. First, the same local solution operator is
used at every time block. Therefore, once the solution operator has been learned
well on the first time interval $[0,T]$, no additional learning is needed for
later time intervals. Long-time prediction can be obtained by repeatedly using
the short-time learned model.

Second, it is enough to check the size of the terminal value. As long as the
terminal value remains in $U_{0,R}$, or equivalently its $L^\infty(D)$ norm
remains bounded by $R$, the same local solution operator can be applied again. Thus, one can
predict the solution with a theoretical guarantee, although the accumulated
error of the short-time learned model must be taken into account.

Third, for PDEs whose solutions are uniformly bounded, the exact solution does
not exceed the threshold $R$ if $R$ is chosen appropriately. Moreover, one can
often show that a short-time learned model, such as a neural operator, can
approximate the local solution operator well while keeping the output uniformly
bounded by $R$ (see Remark \ref{rem_clip} and \citet{Furuya2025}). In such a case, the rollout remains in the admissible state set,
and the same local error estimate can be iterated over successive time blocks.

\subsection{Formulation and Result}

Assume that every trajectory considered below admits a terminal trace. In our
present setting, this is automatic, since the solution space is defined as
$U_M:=\{u\in C([0,T];C(D)):\|u\|_{L^\infty(Q_T)}\le M\}$.
Define $\psi:U_M\to L^\infty(D)$ by $\psi(u):=u(\cdot,T)$, and assume that 
\[
\|\psi(u)-\psi(v)\|_{L^\infty(D)}\le L_\psi\|u-v\|_{L^\infty(Q_T)}.
\]
Let $\Phi_\ast:=\psi\circ\mathcal G_\ast$. For $\kappa\in\mathbb N$, define $U_0^{[\kappa]}:=\{u_0\in U_0:\Phi_\ast^j(u_0)\in U_{0,R}\text{ for }j=0,\ldots,\kappa\}$. This condition ensures that the same local solution theory can be applied on each time block. We assume $\Pi(U_0^{[\kappa]})>0$ and write $\Pi_\kappa:=\Pi(\cdot\mid U_0^{[\kappa]})$.

For $u_0\in U_0^{[\kappa]}$, define the exact states by $v_0=u_0$ and $v_{j+1}=\Phi_\ast(v_j)$. The exact solution on the $j$-th block is $\mathcal G_\ast(v_j)$ after the time shift $t=jT+s$, $s\in[0,T]$. Next, define the approximate transition by $\widehat\Phi_\ell:=\mathcal P_R\circ\psi\circ\widehat\Gamma_\ell$, where, $\mathcal P_R:C(D)\to U_{0,R}$ is the pointwise clipping map. The approximate states are $\widehat v_0=u_0$ and $\widehat v_{j+1}=\widehat\Phi_\ell(\widehat v_j)$. Thus, the model prediction on the $j$-th block is $\widehat\Gamma_\ell[\widehat v_j]$. For $j=0,\ldots,\kappa-1$, define the block-wise risk by
\begin{equation*}
L_j^{\rm block}[\widehat\Gamma_\ell]:=\mathbb E_{u_0\sim\Pi_\kappa,\ z\sim\mu}\left[\left|\widehat\Gamma_\ell[\widehat v_j](z)-\mathcal G_\ast[v_j](z)\right|^2\right].
\end{equation*}
Let $\mathcal V_\kappa(\widehat\Gamma_\ell):=\{v_j(u_0),\widehat v_j(u_0):u_0\in U_0^{[\kappa]},\ j=0,\ldots,\kappa\}\subset U_{0,R}$, and define the rollout-relevant local error by
\begin{equation}\label{local_error}
\varepsilon_{\rm loc,\ell}^{[\kappa]}:=\sup_{v\in\mathcal V_\kappa(\widehat\Gamma_\ell)}\left\|\widehat\Gamma_\ell[v]-\mathcal G_\ast[v]\right\|_{L^\infty(Q_T)} .
\end{equation}

\begin{theorem}[Theorem \ref{thm:blockwise-long-time}, Block-wise long-time error propagation]\label{thm:blockwise-long-time_app}
Assume Assumption~\ref{ass:contraction} and the terminal trace condition above. Then, for any $j=0,\ldots,\kappa-1$,
\begin{equation*}
L_j^{\rm block}[\widehat\Gamma_\ell]\le \left(\sum_{r=0}^{j}\left(L_\psi\frac{C_S}{1-\delta}\right)^r\right)^2\left(\varepsilon_{\rm loc,\ell}^{[\kappa]}\right)^2 .
\end{equation*}
\end{theorem}
This result shows that long-time prediction is obtained by deterministic propagation of the local error. Since no additional empirical risk minimization is performed after $\widehat\Gamma_\ell$ has been learned, the rollout step does not introduce a new Rademacher complexity term. The block index $j$ enters only through the stability factor in Theorem~\ref{thm:blockwise-long-time_app}. The proof iterates the local stability estimate \eqref{CDoI} along the exact and approximate trajectories.

\subsection{Proof of Theorem \ref{thm:blockwise-long-time_app}}
For simplicity, 
we write
\[
    L_G:=\frac{C_S}{1-\delta},
    \quad
    A_G:=L_\psi L_G
    =
    L_\psi\frac{C_S}{1-\delta},
    \quad \varepsilon:=\varepsilon_{\rm loc,\ell}^{[\kappa]}.
\]
Recall that
\[
    v_{j+1}=\Phi_\ast(v_j)
    =
    \psi\bigl(\mathcal G_\ast[v_j]\bigr),
    \quad
    \widehat v_{j+1}
    =
    \widehat\Phi_\ell(\widehat v_j)
    =
    \mathcal P_R\psi\bigl(\widehat\Gamma_\ell[\widehat v_j]\bigr).
\]
Since $v_j,\widehat v_j\in\mathcal V_\kappa(\widehat\Gamma_\ell)$ for
$j=0,\ldots,\kappa$, the definition of $\varepsilon$ gives
\[
    \left\|
    \widehat\Gamma_\ell[w]-\mathcal G_\ast[w]
    \right\|_{L^\infty(Q_T)}
    \le
    \varepsilon
\]
for any $w\in\{v_j,\widehat v_j:j=0,\ldots,\kappa\}$.
We also use that the pointwise clipping map $\mathcal P_R$ is nonexpansive in
$L^\infty(D)$, namely
\[
    \|\mathcal P_R f-\mathcal P_R g\|_{L^\infty(D)}
    \le
    \|f-g\|_{L^\infty(D)}.
\]
Moreover, since $v_{j+1}\in U_{0,R}$, we have
$\mathcal P_R v_{j+1}=v_{j+1}$.
Define the state error
\[
    e_j:=\|\widehat v_j-v_j\|_{L^\infty(D)}.
\]
Then $e_0=0$. We first prove the recursive estimate
\[
    e_{j+1}
    \le
    L_\psi\varepsilon
    +
    A_G e_j,
    \quad j=0,\ldots,\kappa-1.
\]
In fact, using $v_{j+1}=\psi(\mathcal G_\ast[v_j])$ and
$\mathcal P_R v_{j+1}=v_{j+1}$, we have
\[
\begin{aligned}
    e_{j+1}
    &=
    \left\|
    \widehat v_{j+1}-v_{j+1}
    \right\|_{L^\infty(D)}
    =
    \left\|
    \mathcal P_R\psi\bigl(\widehat\Gamma_\ell[\widehat v_j]\bigr)
    -
    \mathcal P_R\psi\bigl(\mathcal G_\ast[v_j]\bigr)
    \right\|_{L^\infty(D)}
    \\
    &\le
    \left\|
    \psi\bigl(\widehat\Gamma_\ell[\widehat v_j]\bigr)
    -
    \psi\bigl(\mathcal G_\ast[v_j]\bigr)
    \right\|_{L^\infty(D)}
    \le
    L_\psi
    \left\|
    \widehat\Gamma_\ell[\widehat v_j]
    -
    \mathcal G_\ast[v_j]
    \right\|_{L^\infty(Q_T)} .
\end{aligned}
\]
By adding and subtracting $\mathcal G_\ast[\widehat v_j]$, we obtain
\[
    \left\|
    \widehat\Gamma_\ell[\widehat v_j]
    -
    \mathcal G_\ast[v_j]
    \right\|_{L^\infty(Q_T)}
    \le
    \left\|
    \widehat\Gamma_\ell[\widehat v_j]
    -
    \mathcal G_\ast[\widehat v_j]
    \right\|_{L^\infty(Q_T)}
    +
    \left\|
    \mathcal G_\ast[\widehat v_j]
    -
    \mathcal G_\ast[v_j]
    \right\|_{L^\infty(Q_T)} .
\]
The first term is bounded by $\varepsilon$ by the definition of
$\varepsilon_{\rm loc,\ell}^{[\kappa]}$. For the second term, since
$\mathcal G_\ast=\mathcal G_{\omega_\ast}$ and both $v_j$ and $\widehat v_j$
belong to $U_{0,R}$, the estimate \eqref{CDoI} gives
\[
    \left\|
    \mathcal G_\ast[\widehat v_j]
    -
    \mathcal G_\ast[v_j]
    \right\|_{L^\infty(Q_T)}
    \le
    L_G
    \|\widehat v_j-v_j\|_{L^\infty(D)}
    =
    L_G e_j .
\]
Therefore,
\[
    e_{j+1}
    \le
    L_\psi(\varepsilon+L_G e_j)
    =
    L_\psi\varepsilon + A_G e_j.
\]
This proves the recursive estimate.
We now solve this recursion. Since $e_0=0$, induction gives
\[
    e_j
    \le
    L_\psi\varepsilon
    \sum_{r=0}^{j-1} A_G^r,
    \quad j=1,\ldots,\kappa.
\]
For $j=0$, this estimate is understood as $e_0=0$.

Next, for the prediction error on the $j$-th block, we estimate
\[
\begin{aligned}
    \left\|
    \widehat\Gamma_\ell[\widehat v_j]
    -
    \mathcal G_\ast[v_j]
    \right\|_{L^\infty(Q_T)}
    &\quad\le
    \left\|
    \widehat\Gamma_\ell[\widehat v_j]
    -
    \mathcal G_\ast[\widehat v_j]
    \right\|_{L^\infty(Q_T)}
    +
    \left\|
    \mathcal G_\ast[\widehat v_j]
    -
    \mathcal G_\ast[v_j]
    \right\|_{L^\infty(Q_T)}
    \\
    &\quad\le
    \varepsilon
    +
    L_G e_j.
\end{aligned}
\]
Using the bound on $e_j$, we obtain
\[
    \left\|
    \widehat\Gamma_\ell[\widehat v_j]
    -
    \mathcal G_\ast[v_j]
    \right\|_{L^\infty(Q_T)}
    \le
    \varepsilon
    +
    L_G L_\psi\varepsilon
    \sum_{r=0}^{j-1} A_G^r
=
    \varepsilon
    \left(
    1+
    A_G
    \sum_{r=0}^{j-1} A_G^r
    \right)
=
    \varepsilon
    \sum_{r=0}^{j} A_G^r .
\]
Hence, for any $z\in Q_T$,
\[
    \left|
    \widehat\Gamma_\ell[\widehat v_j](z)
    -
    \mathcal G_\ast[v_j](z)
    \right|^2
    \le
    \left(
    \sum_{r=0}^{j} A_G^r
    \right)^2
    \varepsilon^2 .
\]
Taking expectation with respect to $u_0\sim\Pi_\kappa$ and $z\sim\mu$, we obtain
\[
    L_j^{\rm block}[\widehat\Gamma_\ell]
    \le
    \left(
    \sum_{r=0}^{j} A_G^r
    \right)^2
    \varepsilon^2 .
\]
Substituting $A_G=L_\psi C_S/(1-\delta)$ and
$\varepsilon=\varepsilon_{\rm loc,\ell}^{[\kappa]}$, we conclude Theorem~\ref{thm:blockwise-long-time_app}.

\section{Nonlinear Heat Equations on the Torus and Picard-Type FNOs}

\subsection{Nonlinear Heat Equations on the Torus and the Operator Learning Setting}
\label{app:F:torus-setting}

We specialize the abstract Picard-type operator learning framework of
Sections~\ref{sec:problem-setting}--\ref{sec:long-time} to nonlinear heat equations
on the torus. Let $\mathbb T^d:=\mathbb R^d/(2\pi\mathbb Z)^d$ and
$Q_T:=\mathbb T^d\times(0,T)$. We consider
\begin{equation}\label{nonlinearHE_torus}
        \partial_t u-\Delta_{\mathbb T^d}u=F(u),
    \quad
    u(0)=u_0 .
\end{equation}
The linear part is fixed and the unknown equation specification is the 
nonlinearity. For $R,M,L>0$, set
\[
    \mathcal F_{M,L}
    :=
    \left\{
    F:[-M,M]\to\mathbb R:
    F(0)=0,\ 
    {\rm Lip}_{[-M,M]}(F)\le L
    \right\},
\]
and define the admissible equation class by
\[
    \Omega_{\rm eq}
    :=
    \left\{
    \omega_F:=(-\Delta_{\mathbb T^d},F):
    F\in\mathcal F_{M,L}
    \right\}.
\]
We write $\mathcal G_F$ and $\mathcal T_{F,u_0}$ instead of
$\mathcal G_{\omega_F}$ and $\mathcal T_{\omega_F,u_0}$ for simplicity.
See Appendix \ref{appendix:B} for the basic facts of this admissible equation class. 
As in Section~\ref{sec:problem-setting}, let
\[
    U_{0,R}
    :=
    \left\{
    u_0\in C(\mathbb T^d):
    \|u_0\|_{L^\infty(\mathbb T^d)}\le R
    \right\},
    \quad
    U_M
    :=
    \left\{
    u\in C([0,T];C(\mathbb T^d)):
    \|u\|_{L^\infty(Q_T)}\le M
    \right\}.
\]
Recall the definition of the Picard map $\mathcal T_{F,u_0}$ for 
$F\in\mathcal F_{M,L}$ and $u_0\in U_{0,R}$: 
\begin{equation}\label{picard-torus}
        \mathcal T_{F,u_0}[v](t)
    :=
    e^{t\Delta_{\mathbb T^d}}u_0
    +
    \int_0^t
    e^{(t-\tau)\Delta_{\mathbb T^d}}
    F(v(\tau))\,d\tau .
\end{equation}
For any $R, M > 0$ with $R<M$ and for any $\delta \in (0,1)$, we can take $T$ sufficiently small so that
\begin{equation}
\label{eq:F:small-time-conditions}
    R+TLM\le M,
    \quad
    TL\le\delta .
\end{equation}
Then, by applying Proposition \ref{prop:ass_app} with $C_S=1$, we have
\[
    \|\mathcal T_{F,u_0}[v]\|_{L^\infty(Q_T)}
    \le R+TLM\le M,
    \quad
    \|\mathcal T_{F,u_0}[v]-\mathcal T_{F,u_0}[w]\|_{L^\infty(Q_T)}
    \le \delta\|v-w\|_{L^\infty(Q_T)} 
\]
for any $v,w\in U_M$, and 
\begin{equation}
\label{eq:F:initial-data-lipschitz}
    \|\mathcal T_{F,u_0}[v]-\mathcal T_{F,\widetilde u_0}[v]\|_{L^\infty(Q_T)}
    \le
    \|u_0-\widetilde u_0\|_{L^\infty(\mathbb T^d)} .
\end{equation}
for any $u_0, \tilde u_0 \in U_{0,R}$. 
Thus, Assumption~\ref{ass:contraction} holds with
$C_S=1$ in this setting. 
Consequently, by
Proposition~\ref{prop:LWP}, for each $F\in\mathcal F_{M,L}$, there, exists a
unique solution operator $\mathcal G_F:U_{0,R}\to U_M$ to \eqref{nonlinearHE_torus}. 
Define 
\begin{equation}\label{notation_picardsol}
        u_{F,u_0}^{(0)}:=0,\quad 
    u_{F,u_0}^{(\ell+1)}
    :=
    \mathcal T_{F,u_0}\bigl[u_{F,u_0}^{(\ell)}\bigr],
    \quad \ell=0,1,2,\cdots .
\end{equation}
Then, by Lemma~\ref{lem:picard_error1} (Lemma \ref{lem:picard_error1-app}), we have the Picard truncation error bound: 
\begin{equation}
\label{eq:F:picard-truncation-error}
    \sup_{F\in\mathcal F_{M,L}}
    \sup_{u_0\in U_{0,R}}
    \left\|
    u_{F,u_0}^{(\ell)}-\mathcal G_F(u_0)
    \right\|_{L^\infty(Q_T)}
    \le
    \frac{M\delta^\ell}{1-\delta}.
\end{equation}

Next, we specify the learning setting. Let $s_0>d/2$ and
\[
    U_0
    :=
    \left\{
    u_0\in H^{s_0}(\mathbb T^d):
    \|u_0\|_{H^{s_0}(\mathbb T^d)}\le R_0
    \right\}
    \subset U_{0,R},
\]
where, $R_0$ is chosen so that the inclusion holds by the Sobolev embedding. Let
$\Pi$ be a probability measure supported on $U_0$, and let $\mu$ be a
probability measure on $Q_T$. The target equation is generated by some
$F_\ast\in\mathcal F_{M,L}$, and we write
$\mathcal G_\ast:=\mathcal G_{F_\ast}$.

For a candidate operator $\Gamma:U_0\to U_M$, the ideal population
risk is
\[
    L[\Gamma]
    :=
    \mathbb E_{u_0\sim\Pi,\ z\sim\mu}
    \left[
    \left|
    \Gamma(u_0)(z)-\mathcal G_\ast(u_0)(z)
    \right|^2
    \right].
\]
The corresponding empirical risk $\widehat L_{n,q}$ is the one defined in
Section~\ref{subsec:risks}. Since the Picard iterates considered below are
$U_M$-valued, their pointwise values are uniformly bounded by $M$.

The abstract $\ell$-step Picard class in this torus setting is
\[
    \mathcal B_\ell^{\Omega_{\rm eq}}
    =
    \left\{
    b_{F,\ell}:F\in\mathcal F_{M,L}
    \right\},
    \quad
    b_{F,\ell}(u_0,z)
    :=
    \left(
    \mathcal T_{F,u_0}^{[\ell]}[0]
    \right)(z).
\]
By \eqref{eq:F:picard-truncation-error} and Lemma~\ref{lem:picard_error1} (Lemma \ref{lem:picard_error1-app}), we have
\[
    \inf_{b\in\mathcal B_\ell^{\Omega_{\rm eq}}}L[b]
    \le
    \frac{M^2\delta^{2\ell}}{(1-\delta)^2}.
\]

For the Rademacher complexity estimate, we use the equation metric
\[
    d_{\Omega_{\rm eq}}(F,G)
    :=
    \sup_{u_0\in U_{0,R}}
    \sup_{v\in U_M}
    \left\|
    \mathcal T_{F,u_0}[v]
    -
    \mathcal T_{G,u_0}[v]
    \right\|_{L^\infty(Q_T)} .
\]
Since the linear part is fixed, the initial-data terms cancel and
\[
    d_{\Omega_{\rm eq}}(F,G)
    \le
    T\|F-G\|_{L^\infty([-M,M])}.
\]
Thus, the entropy of $\Omega_{\rm eq}$ is controlled by the entropy of the
one-dimensional Lipschitz class $\mathcal F_{M,L}$.

Finally, we record the compactness facts used in the implementation theorem.
Since $s_0>d/2$, the Sobolev embedding is compact in the sense that
\[
    U_0\Subset C(\mathbb T^d).
\]
Moreover, by the Ascoli--Arzel\`a theorem, we have
\[
    \mathcal F_{M,L}\Subset C([-M,M]).
\]
These two compactness properties will be used to obtain uniform convergence of
the $L^\infty$-stable Fourier cutoffs on the compact families generated by
finitely many Picard steps.

\subsection{Picard-Type Fourier Neural Operators}
\label{app:F:picard-type-fno}

We now define the concrete Picard-type FNO implementation used in
Theorem~\ref{thm:torus-implementation-error}. The construction follows the
same Picard-aligned principle as in \citet{Furuya2025}, but we use the Fourier
diagonalization of the heat semigroup on $\mathbb T^d$ 
and an
$L^\infty$-stable Fej\'er cutoff (i.e. the Ces\`aro
mean of rectangular Fourier partial sums).

Let
\[
    e_\xi(x):=(2\pi)^{-d/2}e^{i\xi\cdot x},
    \quad
    \xi\in\mathbb Z^d .
\]
For $f\in C(\mathbb T^d)$, we define the Fourier coefficient of $f$ by
\[
    \widehat f(\xi)
    :=
    \langle \overline e_\xi,f\rangle_{\mathbb T^d}
    =
    \int_{\mathbb T^d} f(x)\overline{e_\xi(x)}\,dx .
\]
Recall that 
\[
    e^{t\Delta_{\mathbb T^d}}f
    =
    \sum_{\xi\in\mathbb Z^d}
    e^{-t|\xi|^2}\widehat f(\xi)e_\xi,
\]
For $N\in\mathbb N$, define 
\[
    \mathfrak m_N(\xi)
    :=
    \prod_{r=1}^d
    \left(
    1-\frac{|\xi_r|}{N+1}
    \right)_+,
    \quad
    \xi=(\xi_1,\cdots,\xi_d)\in\mathbb Z^d .
\]
The symbol $\mathfrak m_N$ is fixed and non-trainable. Let
\[
    \Lambda_N
    :=
    \left\{
    \xi\in\mathbb Z^d:
    |\xi_r|\le N,\ r=1,\cdots,d
    \right\}.
\]
The Fej\'er operator is
\begin{equation}\label{Fejer_proj}
        P_N f
    :=
    \sum_{\xi\in\Lambda_N}
    \mathfrak m_N(\xi)\widehat f(\xi)e_\xi .
\end{equation}
Equivalently, $P_N$ is the tensor-product Fej\'er mean, i.e. the Ces\`aro
mean of rectangular Fourier partial sums. In particular, 
\[
    \|P_N f\|_{L^\infty(\mathbb T^d)}
    \le
    \|f\|_{L^\infty(\mathbb T^d)},
    \quad
    P_N f\to f
    \quad
    \text{in } C(\mathbb T^d)
\]
hold for any $f\in C(\mathbb T^d)$ (see, e.g., \citet[Chapter 3]{Grafakos2014ClassicalFourierAnalysis} and 
\citet[Chapter~I]{Katznelson2004}).
We define the Fej\'er approximation of the heat semigroup by 
\begin{equation}\label{truncated_heat_semigroup}
        S_N(t)f
    :=
    P_Ne^{t\Delta_{\mathbb T^d}}f
    =
    \sum_{\xi\in\Lambda_N}
    \mathfrak m_N(\xi)
    e^{-t|\xi|^2}
    \langle \overline e_\xi,f\rangle_{\mathbb T^d}
    e_\xi .
\end{equation}
For $g\in C([0,T];C(\mathbb T^d))$, define 
\[
    \mathcal K_N[g](t)
    :=
    \int_0^t S_N(t-\tau)g(\tau)\,d\tau .
\]
Equivalently, we write
\[
    \mathcal K_N[g](t,x)
    =
    \sum_{\xi\in\Lambda_N}
    \mathfrak m_N(\xi)
    \langle \Psi_{\xi,t},g\rangle_{Q_T}
    e_\xi(x),
\]
where
\[
    \Psi_{\xi,t}(\tau,y)
    :=
    \mathbf 1_{\{0<\tau<t\}}
    e^{-(t-\tau)|\xi|^2}
    \overline{e_\xi(y)}
\]
and
\[
    \langle \Psi_{\xi,t},g\rangle_{Q_T}
    :=
    \int_0^T
    \int_{\mathbb T^d}
    \Psi_{\xi,t}(\tau,y)g(\tau,y)\,dy\,d\tau .
\]
Thus, $\mathcal K_N$ retains only finitely many spatial
Fourier modes, while the time variable remains continuous.
Note that the operators $P_N$, $S_N(t)$, and
$\mathcal K_N$ preserve real-valued functions, 
although we use the complex-valued Fourier basis. 

Let $H\in\mathbb N$. We denote by $\mathcal N_{H,M,L}^{\rm ReLU}$ the class of
one-dimensional ReLU networks $\rho:\mathbb R\to\mathbb R$ with size at most
$H$ such that
\[
    \rho(0)=0,
    \quad
    {\rm Lip}_{[-M,M]}(\rho)\le L .
\]
Only the network $\rho$ 
is trainable in the Picard-type FNO below; $\mathfrak m_N$ and 
$e^{-t|\xi|^2}$ are fixed by the PDE structure.

For $\rho\in\mathcal N_{H,M,L}^{\rm ReLU}$ and $u_0\in U_{0,R}$, define the
approximate Picard map by
\begin{equation}
\label{eq:F:approx-picard-map}
    \mathcal T^N_{\rho,u_0}[v](t)
    :=
    S_N(t)u_0
    +
    \mathcal K_N[\rho(v)](t),
\end{equation}
where, $\rho(v)(\tau,y):=\rho(v(\tau,y))$. Equivalently,
we write 
\[
    \mathcal T^N_{\rho,u_0}[v](t,x)
    =
    \sum_{\xi\in\Lambda_N}
    \mathfrak m_N(\xi)
    e^{-t|\xi|^2}
    \langle \overline e_\xi,u_0\rangle_{\mathbb T^d}
    e_\xi(x)
+
    \sum_{\xi\in\Lambda_N}
    \mathfrak m_N(\xi)
    \langle \Psi_{\xi,t},\rho(v)\rangle_{Q_T}
    e_\xi(x).
\]

\begin{definition}[Picard-type Fourier neural operator]
\label{def:F:picard-type-fno}
Let $N,\ell,H\in\mathbb N$ and
$\rho\in\mathcal N_{H,M,L}^{\rm ReLU}$. For $u_0\in U_{0,R}$, define
\[
u^{(0)}_{N,\rho}[u_0]:=0,\quad 
    u^{(j+1)}_{N,\rho}[u_0]
    :=
    \mathcal T^N_{\rho,u_0}
    \left[
    u^{(j)}_{N,\rho}[u_0]
    \right],\quad j=0,\cdots,\ell-1.
\]
That is,
\[
    u^{(j+1)}_{N,\rho}[u_0](t)
    =
    S_N(t)u_0
    +
    \mathcal K_N
    \left[
    \rho\left(u^{(j)}_{N,\rho}[u_0]\right)
    \right](t).
\]
The $\ell$-step Picard-type Fourier neural operator 
$\Gamma^{\rm FNO}_{N,\ell,\rho} : U_0\to U_M$ 
is defined by 
\begin{equation}\label{PFNO}
      \Gamma^{\rm FNO}_{N,\ell,\rho}(u_0)
    :=
    u^{(\ell)}_{N,\rho}[u_0].  
\end{equation}
The corresponding implementation class is
\[
    \mathfrak N^{\rm FNO}_{N,\ell,H}(M,L)
    :=
    \left\{
    \Gamma^{\rm FNO}_{N,\ell,\rho} : 
    \rho\in\mathcal N_{H,M,L}^{\rm ReLU}
    \right\}
\]
as a class of maps from $U_0$ to $U_M$.
\end{definition}

Following the two-channel construction in the proof of
Lemma~5 of \citet{Furuya2025}, we express the operator in
Definition~\ref{def:F:picard-type-fno} as a composition
of an input map, shared hidden blocks, and an output map.
Define the input and output maps by
\[
    \mathcal L_Nu_0
    :=
    \begin{pmatrix}
        S_N(\cdot)u_0\\
        S_N(\cdot)u_0
    \end{pmatrix},
    \quad
    \mathcal Q(v_1,v_2):=v_1.
\]
For $v=(v_1,v_2)^\top$, set
\[
    \mathcal R_\rho(v)
    :=
    \begin{pmatrix}
        \rho(v_1)\\
        v_2
    \end{pmatrix},
    \quad
    W:=
    \begin{pmatrix}
        0&1\\
        0&1
    \end{pmatrix},
    \quad
    \mathbb K_Nv
    :=
    \begin{pmatrix}
        \mathcal K_N[v_1]\\
        0
    \end{pmatrix}.
\]
Then the hidden block
$\mathcal B_{N,\rho}:=(W+\mathbb K_N)\circ\mathcal R_\rho$
satisfies
\[
    \mathcal B_{N,\rho}
    \begin{pmatrix}
        v_1\\
        v_2
    \end{pmatrix}
    =
    \begin{pmatrix}
        v_2+\mathcal K_N[\rho(v_1)]\\
        v_2
    \end{pmatrix}.
\]
Since $\rho(0)=0$, the first Picard iterate is
$u^{(1)}_{N,\rho}[u_0]=S_N(\cdot)u_0$.
Each hidden block performs one additional Picard step
while preserving $S_N(\cdot)u_0$ in the second channel.
Hence
\[
    \Gamma^{\rm FNO}_{N,\ell,\rho}
    =
    \mathcal Q\circ
    \mathcal B_{N,\rho}^{[\ell-1]}
    \circ\mathcal L_N,
\]
where $\mathcal B_{N,\rho}^{[0]}$ is the identity map.
The same scalar network $\rho$ is shared across all
hidden blocks, while the linear operators are fixed.


\begin{remark}[Comparison with other neural-operator architectures]
\label{rem:F:architecture-comparison}
A standard FNO layer has the form
\[
    v_{j+1}
    =
    \sigma\left(
        W_jv_j
        +\mathcal F_x^{-1}
        [R_j\mathcal F_xv_j]
        +b_j
    \right),
\]
where $W_j$, $R_j$, and $b_j$ are trainable and may vary
between layers; see \citet{Li2020a}. 
The construction in the proof of Lemma~5 of
\citet{Furuya2025} instead follows Picard iteration.
It uses two channels to retain the current iterate and
the linear initial-data term, and shares the same
nonlinear network across Picard steps.
The integral operator is approximated by a finite-rank
space-time kernel expansion. 
Our construction uses the same two-channel structure.
The first channel is updated by
\[
    v_{j+1}=a+\mathcal K_N[\rho(v_j)],
    \quad
    a:=S_N(\cdot)u_0,
\]
while the second channel retains $a$.
Only the scalar network $\rho$ is trainable.
The input map, output map, and Fourier--Duhamel operator
are fixed, and the same $\rho$ is shared across all blocks.
Here, the Fourier cutoff acts only in space, while the
Volterra integral in time is retained.
Thus, each block performs one Picard step rather than
a standard FNO layer.
\end{remark}


\begin{remark}\label{rem_clip}
Assume that the parameters $R,M,L,T,\delta$ satisfy
\eqref{eq:F:small-time-conditions}. Then the following elementary stability
property holds, which shows that no additional clipping is required in
Definition~\ref{def:F:picard-type-fno}. In fact, since both
$e^{t\Delta_{\mathbb T^d}}$ and $P_N$ are $L^\infty$-contractive, we have
\begin{equation}\label{eq:K_N}
        \|S_N(t)f\|_{L^\infty(\mathbb T^d)}
    \le
    \|f\|_{L^\infty(\mathbb T^d)},
    \quad
    \|\mathcal K_N[g]\|_{L^\infty(Q_T)}
    \le
    T\|g\|_{L^\infty(Q_T)} .
\end{equation}
Hence, for any $u_0\in U_{0,R}$ and $v\in U_M$, 
\begin{equation}
\label{eq:F:approx-picard-self-map}
    \|\mathcal T^N_{\rho,u_0}[v]\|_{L^\infty(Q_T)}
    \le
    R+TLM
    \le M .
\end{equation}
Moreover, for $v,w\in U_M$,
\begin{equation}\label{eq:F:approx-picard-con}
 \begin{aligned}
    \|\mathcal T^N_{\rho,u_0}[v]
    -
    \mathcal T^N_{\rho,u_0}[w]\|_{L^\infty(Q_T)}
    &\le
    T\,{\rm Lip}_{[-M,M]}(\rho)
    \|v-w\|_{L^\infty(Q_T)}
\\
    &\le
    \delta
    \|v-w\|_{L^\infty(Q_T)} .
\end{aligned}   
\end{equation}
Therefore $\mathcal T^N_{\rho,u_0}:U_M\to U_M$ is a $\delta$-contraction,
uniformly in $N$, $\rho\in\mathcal N_{H,M,L}^{\rm ReLU}$, and
$u_0\in U_{0,R}$. In particular,
$\Gamma^{\rm FNO}_{N,\ell,\rho}(u_0)\in U_M$ for any
$u_0\in U_{0,R}$.   
\end{remark}

\begin{remark}[Trainable parameters and computational interpretation]
\label{rem:F:trainable-and-computational-interpretation}
We clarify the implementation meaning of Definition~\ref{def:F:picard-type-fno}.

\begin{enumerate}[(a)]
\item 
The trainable part of the Picard-type FNO is only the scalar ReLU network
$\rho\in\mathcal N_{H,M,L}^{\rm ReLU}$. The parameters $N,\ell,H$ are
architecture or approximation parameters, not trainable weights. The Fourier
multiplier mask $\mathfrak m_N(\xi)$, the Fourier basis $e_\xi$, and the heat
factors $e^{-t|\xi|^2}$ are fixed by the PDE structure.

\item 
The dual-pair terms are not additional trainable parameters. For example,
$\langle \overline e_\xi,u_0\rangle_{\mathbb T^d}$ is the Fourier coefficient
of the input, and $\langle \Psi_{\xi,t},g\rangle_{Q_T}$ is the coefficient of
the fixed Fourier-heat kernel applied to $g$. These quantities are computed in
the forward pass. They depend on the current iterate through
$g=\rho(u^{(j)}_{N,\rho}[u_0])$, and therefore gradients flow through them to
the parameters of $\rho$, but the kernels themselves are not learned.

\item 
On a periodic spatial grid, the spatial dual pairs can be computed by FFT. The
Duhamel term is then evaluated mode by mode as an exponential convolution in
time,
\[
    \int_0^t
    e^{-(t-\tau)|\xi|^2}
    \widehat g(\tau,\xi)\,d\tau ,
\]
which can be approximated by quadrature or by a time-stepping recurrence on an
equally spaced time grid. Thus, the dual-pair notation is an analytic way of
writing a computable Fourier-heat convolution.

\item 
The numerical cost is not dimension-free. If $m^d$ is the number of spatial grid
points, $n_t$ is the number of time points, and $\ell$ is the Picard depth, then
a standard grid implementation has a leading spatial FFT cost of order
\[
    O\!\left(\ell\, n_t\, m^d\log(m^d)\right),
\]
up to the additional cost of the time convolution. Equivalently, if $N$ denotes the cutoff level in each coordinate direction,
then the tensor-product Fourier index set satisfies
$|\Lambda_N|=(2N+1)^d\asymp N^d$. Thus, the usual numerical cost of
high-dimensional spectral discretization remains.

\item 
The role of the theory is different from eliminating all numerical
dimension-dependence. The Fourier cutoff $N$ enters the error bound only through
the implementation error, while the Rademacher complexity is controlled by the
abstract Picard class and, in the torus example, by the scalar nonlinearity
class $\mathcal F_{M,L}$. Hence, increasing the Fourier rank improves the
implementation error without enlarging the statistical complexity term.

\item 
The constraints $\rho(0)=0$ and ${\rm Lip}_{[-M,M]}(\rho)\le L$ are part of the
theoretical architecture class. They ensure that the approximate Picard map
remains a self-map on $U_M$ and is a $\delta$-contraction under
\eqref{eq:F:small-time-conditions}. In practical training, one should enforce
or approximate these constraints, for instance by using a slope-constrained
piecewise affine parameterization, projection, or normalization. Without these
constraints, the no-clipping stability argument need not apply.
\end{enumerate}
\end{remark}

\subsection{Proof of Theorem~\ref{thm:torus-implementation-error}}
\label{app:F:proof-implementation-error}

In this appendix, we prove the implementation error estimate for the Picard-type FNO. The
argument follows the Picard-iteration strategy of \citet{Furuya2025}, but the
rank-$N$ part is handled differently. Instead of postulating an explicit
rank-$N$ approximation rate for the Green kernel, we use the compactness of
$U_0$ in $C(\mathbb T^d)$ and the compactness of $\mathcal F_{M,L}$ in
$C([-M,M])$ to obtain uniform convergence of the $L^\infty$-stable Fej\'er
cutoffs on the compact families generated by finitely many Picard steps.

\begin{theorem}[Theorem~\ref{thm:torus-implementation-error}, Implementation error for the Picard-type FNO]
\label{thm:F:implementation-error-restated}
Let $s_0>d/2$ and let
$U_0=\{u_0\in H^{s_0}(\mathbb T^d):\|u_0\|_{H^{s_0}}\le R_0\}\subset U_{0,R}$.
Assume \eqref{eq:F:small-time-conditions}. Then, for any
$\eta\in(0,1)$ and every $F\in\mathcal F_{M,L}$, there, exists a
one-dimensional ReLU network $\rho_F$ such that
\begin{equation}
\label{eq:F:rho-F-approx}
    \|F-\rho_F\|_{L^\infty([-M,M])}\le\eta,
    \quad
    \rho_F(0)=0,
    \quad
    {\rm Lip}_{[-M,M]}(\rho_F)\le L,
\end{equation}
and
\begin{equation}
\label{eq:F:rho-F-size}
    H(\rho_F)\le C(1+ML\eta^{-1}),
\end{equation}
where, $C>0$ is a universal constant.
Moreover, for each fixed $\ell\in\mathbb N$, there, exists a sequence
$a_N(\ell)\to0$ as $N\to\infty$ such that
\begin{equation}
\label{eq:F:implementation-sup-error}
    \sup_{F\in\mathcal F_{M,L}}
    \sup_{u_0\in U_0}
    \left\|
    \Gamma^{\rm FNO}_{N,\ell,\rho_F}(u_0)
    -
    \mathcal T_{F,u_0}^{[\ell]}[0]
    \right\|_{L^\infty(Q_T)}
    \le
    \frac{1-\delta^\ell}{1-\delta}\,T\eta
    +
    a_N(\ell).
\end{equation}
Here, the Fourier truncation error $a_N(\ell)$ can be chosen
independently of $\eta$.
Consequently, if
\[
    \operatorname{embed}_{N,\eta}
    \bigl(b_{F,\ell}\bigr)
    :=
    \Gamma^{\rm FNO}_{N,\ell,\rho_F},
    \quad
    b_{F,\ell}(u_0,z)
    :=
    \left(
    \mathcal T_{F,u_0}^{[\ell]}[0]
    \right)(z),
\]
then the implementation error 
\[
    \varepsilon_{\rm imp}^{\rm FNO}(N,\ell,\eta)
    :=
    \sup_{b\in\mathcal B_\ell^{\Omega_{\rm eq}}}
    \left|
        L[\operatorname{embed}_{N,\eta}(b)]-L[b]
    \right|
\]
satisfies
\begin{equation}
\label{eq:F:fno-implementation-error}
    \varepsilon_{\rm imp}^{\rm FNO}(N,\ell,\eta)
    \le 
    \sup_{F\in\mathcal F_{M,L}}
    \left|
    L\!\left[\Gamma^{\rm FNO}_{N,\ell,\rho_F}\right]
    -
    L[b_{F,\ell}]
    \right|
    \le
    4M
    \left\{
    \frac{1-\delta^\ell}{1-\delta}\,T\eta
    +
    a_N(\ell)
    \right\}.
\end{equation}
\end{theorem}

\begin{proof}
The proof of \eqref{eq:F:implementation-sup-error} is divided into three steps.

\paragraph{Step 1: Approximation of the nonlinearity.}
Fix $F\in\mathcal F_{M,L}$. If $L=0$, then, since $F(0)=0$,
we have $F\equiv0$ on $[-M,M]$, and we simply take
$\rho_F\equiv0$.

Suppose $L>0$. Choose a partition
\[
    -M=x_0<x_1<\cdots<x_K=M
\]
which contains $0$ and satisfies
\[
    h:=\max_{0\le k\le K-1}|x_{k+1}-x_k|\le \frac{\eta}{L},
    \quad
    K\le C(1+ML\eta^{-1})
\]
with a universal constant $C$. Let $\rho_F$ be the continuous
piecewise affine interpolation of $F$ on this partition.
The $\rho_F$ satisfies \eqref{eq:F:rho-F-approx}. In fact, we have $\rho_F(0)=0$ 
since the
partition contains $0$ and $F(0)=0$, and ${\rm Lip}_{[-M,M]}(\rho_F)\le L$ since 
every secant slope of $F$ on the partition has absolute
value at most $L$.
Moreover, for $x\in[x_k,x_{k+1}]$, writing
$x=(1-\theta)x_k+\theta x_{k+1}$ with $0\le\theta\le1$, we have
\[
\begin{aligned}
|F(x)-\rho_F(x)|
\le (1-\theta)|F(x)-F(x_k)|
     +\theta |F(x)-F(x_{k+1})| \le Lh.
\end{aligned}
\]
Therefore,
\[
    \|F-\rho_F\|_{L^\infty([-M,M])}\le Lh\le\eta .
\]
Thus, \eqref{eq:F:rho-F-approx} is shown. 

By the ReLU network representation of continuous 
piecewise linear functions, 
it is seen that 
$\rho_F$ is represented by a ReLU network with 
\eqref{eq:F:rho-F-size} (see the proof of
Proposition~1 in \citet{Yarotsky2017}).

Since $\rho_F(0)=0$ and
${\rm Lip}_{[-M,M]}(\rho_F)\le L$, the estimates used for the
Picard map remain valid with $F$ replaced by $\rho_F$. Consequently,
the approximate Picard map $\mathcal T^N_{\rho_F,u_0}$ defined in
\eqref{eq:F:approx-picard-map} satisfies
\eqref{eq:F:approx-picard-self-map} and
\eqref{eq:F:approx-picard-con} for all $v,w\in U_M$.

To prove \eqref{eq:F:implementation-sup-error}, we decompose 
\begin{equation}\label{eq:F:decomposition}
\begin{split}
      &\left\|
    \Gamma^{\rm FNO}_{N,\ell,\rho_F}(u_0)
    -
    \mathcal T_{F,u_0}^{[\ell]}[0]
    \right\|_{L^\infty(Q_T)}  
    = \left\|
    u^{(\ell)}_{N,\rho}[u_0]
    -
    u_{F,u_0}^{(\ell)}
    \right\|_{L^\infty(Q_T)}  \\
    & \le 
    \left\|
    u^{(\ell)}_{N,\rho}[u_0]
    -
    \widetilde u_{\rho_F,u_0}^{(\ell)}
    \right\|_{L^\infty(Q_T)}
    +
    \left\|
    \widetilde u_{\rho_F,u_0}^{(\ell)}
    -
    u_{F,u_0}^{(\ell)}
    \right\|_{L^\infty(Q_T)},
\end{split}
\end{equation}
where, $u_{F,u_0}^{(\ell)}$ and $u^{(\ell)}_{N,\rho}[u_0]$ are defined by \eqref{notation_picardsol} and \eqref{PFNO}, respectively. 
Here, $\widetilde u_{\rho_F,u_0}^{(\ell)}$ is defined by 
\[
    \widetilde u_{\rho_F,u_0}^{(0)}:=0,
    \quad
    \widetilde u_{\rho_F,u_0}^{(j+1)}
    :=
    \mathcal T_{\rho_F,u_0}
    \bigl[
    \widetilde u_{\rho_F,u_0}^{(j)}
    \bigr],
\]
where
\[
    \mathcal T_{\rho_F,u_0}[v](t)
    :=
    e^{t\Delta_{\mathbb T^d}}u_0
    +
    \int_0^t e^{(t-\tau)\Delta_{\mathbb T^d}}
    \rho_F(v(\tau))\,d\tau .
\]

\paragraph{Step 2: Error caused by replacing $F$ with $\rho_F$.}
In this step, we prove 
\begin{equation}
\label{eq:F:rho-F-picard-error}
    \sup_{F\in\mathcal F_{M,L}}
    \sup_{u_0\in U_0}    \left\|
    \widetilde u_{\rho_F,u_0}^{(\ell)}
    -
    u_{F,u_0}^{(\ell)}
    \right\|_{L^\infty(Q_T)}
    \le
    \frac{1-\delta^\ell}{1-\delta}\,T\eta .
\end{equation}
Since $e^{t\Delta_{\mathbb T^d}}$ is
$L^\infty$-contractive and \eqref{eq:F:rho-F-approx} holds, we estimate
\[
\begin{aligned}
    &
    \left\|
    \widetilde u_{\rho_F,u_0}^{(j+1)}
    -
    u_{F,u_0}^{(j+1)}
    \right\|_{L^\infty(Q_T)}
\\
    &\quad\le
    T
    \left\|
    \rho_F
    \left(
    \widetilde u_{\rho_F,u_0}^{(j)}
    \right)
    -
    F
    \left(
    u_{F,u_0}^{(j)}
    \right)
    \right\|_{L^\infty(Q_T)}
\\
    &\quad\le
    T
    \left\|
    \rho_F
    \left(
    \widetilde u_{\rho_F,u_0}^{(j)}
    \right)
    -
    \rho_F
    \left(
    u_{F,u_0}^{(j)}
    \right)
    \right\|_{L^\infty(Q_T)}
    +
    T
    \left\|
    \rho_F
    \left(
    u_{F,u_0}^{(j)}
    \right)
    -
    F
    \left(
    u_{F,u_0}^{(j)}
    \right)
    \right\|_{L^\infty(Q_T)}
\\
    &\quad\le
    \delta
    \left\|
    \widetilde u_{\rho_F,u_0}^{(j)}
    -
    u_{F,u_0}^{(j)}
    \right\|_{L^\infty(Q_T)}
    +
    T\eta .
\end{aligned}
\]
Since the zeroth iterates coincide, the induction argument yields
\eqref{eq:F:rho-F-picard-error}. 

\paragraph{Step 3: Uniform Fourier truncation error by compactness.}
In this step, 
we prove that for each fixed $\ell\in\mathbb N$, there, exists a sequence
$a_N(\ell)\to0$ as $N\to\infty$ such that
\begin{equation}\label{aaa}
\sup_{F\in\mathcal F_{M,L}}
    \sup_{u_0\in U_0}
    \left\|
    u_{N,\rho_F}^{(\ell)}[u_0]
    -
    \widetilde u_{\rho_F,u_0}^{(\ell)}
    \right\|_{L^\infty(Q_T)}
    \le
    a_N(\ell).
\end{equation}
For $G\in\mathcal F_{M,L}$, we define 
$\{u_{G,u_0}^{(j)}\}_j$ by \eqref{notation_picardsol}, and the Fej\'er-truncated Picard iterates by
\begin{equation}
\label{eq:F:truncated-picard-G}
    u_{N,G,u_0}^{(0)}:=0,
    \quad
    u_{N,G,u_0}^{(j+1)}
    :=
    S_N(\cdot)u_0
    +
    \mathcal K_N
    \left[
    G
    \left(
    u_{N,G,u_0}^{(j)}
    \right)
    \right]
\end{equation}
(note that $u_{N,G}^{(\ell)}[u_0] = u_{N,G,u_0}^{(\ell)}$). 
Since $G(0)=0$ and ${\rm Lip}_{[-M,M]}(G)\le L$, the same estimate as
\eqref{eq:F:approx-picard-self-map} and 
\eqref{eq:F:approx-picard-con} shows that the truncated
Picard map is a $\delta$-contraction from $U_M$ into itself. Thus, all iterates
in \eqref{eq:F:truncated-picard-G} are $U_M$-valued.

First, we prove that the sets
\[
    \mathcal C_j
    :=
    \left\{
        u_{G,u_0}^{(j)}:
        G\in\mathcal F_{M,L},\ u_0\in U_0
    \right\}
    \subset C([0,T];C(\mathbb T^d))
\]
are compact for $j=0,1,\ldots,\ell$.
To this end, we show by induction on $j$ that the map
\[
    (u_0,G)\mapsto u_{G,u_0}^{(j)}
\]
is continuous.
For $j=0$, this map is constant.
Suppose that it is continuous for some fixed $j$.
The map
\[
    (u_0,G,u)
    \mapsto
    e^{t\Delta_{\mathbb T^d}}u_0
    +
    \int_0^t e^{(t-\tau)\Delta_{\mathbb T^d}}G(u(\tau))\,d\tau
\]
is continuous from $U_0\times\mathcal F_{M,L}\times U_M$
into $C([0,T];C(\mathbb T^d))$.
Indeed, $u_0\mapsto e^{(\cdot)\Delta_{\mathbb T^d}}u_0$
is continuous by the $L^\infty$-contractivity of the heat
semigroup, and
\begin{equation}
\label{eq:F:composition-continuity}
    \|G(u)-\widetilde G(\widetilde u)\|_{L^\infty(Q_T)}
    \le
    \|G-\widetilde G\|_{L^\infty([-M,M])}
    +
    L\|u-\widetilde u\|_{L^\infty(Q_T)}.
\end{equation}
Therefore, the Picard recurrence and the induction hypothesis
show that $(u_0,G)\mapsto u_{G,u_0}^{(j+1)}$ is continuous.
This proves continuity by induction.
Since $U_0$ and $\mathcal F_{M,L}$ are compact in
$C(\mathbb T^d)$ and $C([-M,M])$, respectively,
each $\mathcal C_j$ is compact as the continuous image
of $U_0\times\mathcal F_{M,L}$.

Next, we define
\[
    \mathcal D_j
    :=
    \left\{
    G(u):
    G\in\mathcal F_{M,L},\ u\in\mathcal C_j
    \right\}
    \subset C([0,T];C(\mathbb T^d)).
\]
Then 
$\mathcal D_j$ is compact for any fixed $j$ by the same continuity estimate \eqref{eq:F:composition-continuity}.

Since $P_N$ is the Fej\'er operator 
\eqref{Fejer_proj}, the convergence 
$P_N f\to f$ in $C(\mathbb T^d)$ holds for each
$f\in C(\mathbb T^d)$. Hence, by the standard compactness argument, we have
\begin{equation}
\label{eq:F:uniform-fejer-on-compact}
    \sup_{f\in\mathcal K}
    \|P_Nf-f\|_{L^\infty(\mathbb T^d)}
    \to0
    \quad
    \text{for any compact }
    K\subset C(\mathbb T^d).
\end{equation}
Applying \eqref{eq:F:uniform-fejer-on-compact} with
$K=\mathcal H_0$ given by 
\[
    \mathcal H_0
    :=
    \left\{
    e^{t\Delta_{\mathbb T^d}}u_0:
    u_0\in U_0,\ t\in[0,T]
    \right\},
\]
where, $\mathcal H_0$ is compact because $U_0$ is compact in
$C(\mathbb T^d)$ and the map
$(u_0,t)\mapsto e^{t\Delta_{\mathbb T^d}}u_0$ is continuous from
$U_0\times[0,T]$ into $C(\mathbb T^d)$, 
we obtain 
\[
    \alpha_N
    :=
    \sup_{u_0\in U_0}
    \left\|
    S_N(\cdot)u_0
    -
    e^{(\cdot)\Delta_{\mathbb T^d}}u_0
    \right\|_{L^\infty(Q_T)}
    \to0 
\]
as $N\to \infty$.
For each $j$, apply \eqref{eq:F:uniform-fejer-on-compact} to the compact set
\[
    \mathcal H_j
    :=
    \left\{
    e^{(t-\tau)\Delta_{\mathbb T^d}}g(\tau):
    g\in\mathcal D_j,\ 0\le\tau\le t\le T
    \right\}
    \subset C(\mathbb T^d).
\]
This yields
\[
    \beta_{N,j}
    :=
    \sup_{g\in\mathcal D_j}
    \left\|
    \mathcal K_N[g]-\mathcal K[g]
    \right\|_{L^\infty(Q_T)}
\le
    T
    \sup_{h\in\mathcal H_j}
    \|P_Nh-h\|_{L^\infty(\mathbb T^d)}
    \to0 
\]
as $N\to \infty$ for each $j$, 
where
\[
    \mathcal K[g](t)
    :=
    \int_0^t e^{(t-\tau)\Delta_{\mathbb T^d}}g(\tau)\,d\tau .
\]

Let
\[
    d_{N,j}
    :=
    \sup_{G\in\mathcal F_{M,L}}
    \sup_{u_0\in U_0}
    \left\|
    u_{N,G,u_0}^{(j)}
    -
    u_{G,u_0}^{(j)}
    \right\|_{L^\infty(Q_T)} .
\]
Then $d_{N,0}=0$. Using \eqref{eq:F:small-time-conditions}, \eqref{notation_picardsol}, \eqref{eq:K_N}, 
 \eqref{eq:F:truncated-picard-G}, and \eqref{eq:F:composition-continuity},
we obtain
\begin{equation}
\label{eq:F:fourier-error-recursion}
    d_{N,j+1}
    \le
    \alpha_N
    +
    \beta_{N,j}
    +
    T L\,d_{N,j}
\le
    \alpha_N
    +
    \beta_{N,j}
    +
    \delta d_{N,j}.
\end{equation}
Indeed, the term $\beta_{N,j}$ controls
$(\mathcal K_N-\mathcal K)[G(u_{G,u_0}^{(j)})]$, and the last term controls
$\mathcal K_N[G(u_{N,G,u_0}^{(j)})-G(u_{G,u_0}^{(j)})]$.
Iterating \eqref{eq:F:fourier-error-recursion} gives
\begin{equation}
\label{eq:F:a-N-ell-definition}
    d_{N,\ell}
    \le
    \sum_{j=0}^{\ell-1}
    \delta^{\ell-1-j}
    \left(
    \alpha_N+\beta_{N,j}
    \right)
    =:
    a_N(\ell).
\end{equation}
Since $\ell$ is fixed, $\alpha_N\to0$ and $\beta_{N,j}\to0$ for each
$j=0,\ldots,\ell-1$, we have
\[
    a_N(\ell)\to0
    \quad
    \text{as } N\to\infty .
\]
Taking $G=\rho_F$, and recalling that
$\rho_F\in\mathcal F_{M,L}$ on $[-M,M]$, we obtain \eqref{aaa}, where, $u_{N,\rho_F,u_0}^{(\ell)}
=
\Gamma^{\rm FNO}_{N,\ell,\rho_F}(u_0)$ by
Definition~\ref{def:F:picard-type-fno}.

Combining \eqref{eq:F:decomposition}, \eqref{eq:F:rho-F-picard-error} and
\eqref{aaa}, we conclude
\[
    \sup_{F\in\mathcal F_{M,L}}
    \sup_{u_0\in U_0}
    \left\|
    \Gamma^{\rm FNO}_{N,\ell,\rho_F}(u_0)
    -
    \mathcal T_{F,u_0}^{[\ell]}[0]
    \right\|_{L^\infty(Q_T)}
    \le
    \frac{1-\delta^\ell}{1-\delta}\,T\eta
    +
    a_N(\ell).
\]
This proves \eqref{eq:F:implementation-sup-error}.

Finally, let $b_{F,\ell}(u_0,z)=
(\mathcal T_{F,u_0}^{[\ell]}[0])(z)$ and 
$\Gamma_F:=\Gamma^{\rm FNO}_{N,\ell,\rho_F}$. Since both
$b_{F,\ell}$ and $\Gamma_F$ take values in $[-M,M]$, and since
$\mathcal G_\ast(u_0)(z)\in[-M,M]$, we have
\[
\begin{aligned}
    &
    \left|
    \left|
    \Gamma_F(u_0)(z)-\mathcal G_\ast(u_0)(z)
    \right|^2
    -
    \left|
    b_{F,\ell}(u_0,z)-\mathcal G_\ast(u_0)(z)
    \right|^2
    \right|
\\
    &\qquad\le
    4M
    \left|
    \Gamma_F(u_0)(z)-b_{F,\ell}(u_0,z)
    \right|.
\end{aligned}
\]
Taking expectation and then the supremum over $F\in\mathcal F_{M,L}$ gives
\eqref{eq:F:fno-implementation-error}. The proof is complete.
\end{proof}

\begin{proposition}[Explicit Fourier truncation rate]
\label{prop:F:fourier-rate}
Assume the setting of
Theorem~\ref{thm:F:implementation-error-restated}.
For any $0<\gamma<\min\{s_0-d/2,1\}$, 
there exists a constant $C_{\gamma,d,s_0}>0$ such that
\begin{equation}\label{eq:F:fourier-rate}
    a_N(\ell)
    \le
    C_{\gamma,d,s_0}
    \frac{1-\delta^\ell}{1-\delta}
    \left(R_0+LM T^{1-\frac{\gamma}{2}}\right)N^{-\gamma}
\end{equation}
for any $N,\ell\in\mathbb N$.
In particular, this bound is independent of $\eta$
and uniform in the Picard depth $\ell$.
\end{proposition}

\begin{proof}
Write $S(t):=e^{t\Delta_{\mathbb T^d}}$.
For $0<\gamma<1$, the tensor-product Fej\'er operator satisfies
\begin{equation}\label{eq:F:fejer-holder-rate}
    \|(P_N-I)f\|_{L^\infty}
    \le C_{\gamma,d}N^{-\gamma}[f]_{C^{0,\gamma}},
\end{equation}
where $[f]_{C^{0,\gamma}}$ denotes the H\"older seminorm
on $\mathbb T^d$.
Indeed, the normalized one-dimensional Fej\'er kernel
$K_N$ satisfies
\[
    K_N(y)\le C\min\{N+1,(N+1)^{-1}|y|^{-2}\},
    \quad 0<|y|\le\pi.
\]
Hence
\[
    \int_{-\pi}^{\pi}|y|^\gamma K_N(y)\,dy
    \le C_\gamma N^{-\gamma}.
\]
Applying this estimate to the product kernel gives
\eqref{eq:F:fejer-holder-rate}.

Since $\gamma<s_0-d/2$, the Sobolev embedding gives
\[
    [u_0]_{C^{0,\gamma}}
    \le C_{\gamma,d,s_0}\|u_0\|_{H^{s_0}}.
\]
The heat semigroup preserves the H\"older seminorm.
Thus, the initial-data error defined above satisfies
\[
    \alpha_N
    \le C_{\gamma,d,s_0}R_0N^{-\gamma}.
\]

Next, the heat semigroup satisfies
\[
    [S(s)f]_{C^{0,\gamma}}
    \le C_{\gamma,d}s^{-\gamma/2}\|f\|_{L^\infty},
    \quad s>0.
\]
This follows from $L^\infty$-contractivity and the gradient
estimate
$\|\nabla S(s)f\|_{L^\infty}
\le C_d s^{-1/2}\|f\|_{L^\infty}$.
Combining this with \eqref{eq:F:fejer-holder-rate}, we obtain
\[
    \|(P_N-I)S(s)f\|_{L^\infty}
    \le C_{\gamma,d}N^{-\gamma}
    s^{-\gamma/2}\|f\|_{L^\infty}.
\]
For any $g\in\mathcal D_j$, we have
$\|g\|_{L^\infty(Q_T)}\le LM$.
Therefore,
\[
\begin{aligned}
    \|(\mathcal K_N-\mathcal K)[g]\|_{L^\infty(Q_T)}
    &\le
    C_{\gamma,d}LM N^{-\gamma}
    \sup_{0<t\le T}\int_0^t(t-\tau)^{-\gamma/2}\,d\tau\\
    &\le
    C_{\gamma,d}LM T^{1-\gamma/2}N^{-\gamma}.
\end{aligned}
\]
Hence
\[
    \beta_{N,j}
    \le C_{\gamma,d}LM T^{1-\gamma/2}N^{-\gamma}
\]
uniformly in $j$.
Substituting these estimates into
\eqref{eq:F:a-N-ell-definition} proves
\eqref{eq:F:fourier-rate}.
\end{proof}

\begin{remark}[Compactness and quantitative convergence]
\label{rem:F:compactness-vs-rank-rate}
\begin{enumerate}[(a)]
    \item 
The above proof differs from the Green-kernel approximation
argument of \citet{Furuya2025}, where the approximation error
is controlled through a chosen kernel expansion.
Here, we use the compactness of $U_0$ and $\mathcal F_{M,L}$
and the uniform convergence of Fej\'er cutoffs on the resulting
compact families to obtain $a_N(\ell)\to0$ for any fixed $\ell$.
\item 
Proposition~\ref{prop:F:fourier-rate} gives an explicit rate
under the same assumptions, using the Sobolev regularity
of the initial data and heat-semigroup smoothing.
Thus, the implementation error can be made arbitrarily small
by choosing $\eta>0$ sufficiently small and $N$ sufficiently
large, uniformly in $\ell$.
Increasing $\ell$ also reduces the Picard truncation error.
\item 
The Fourier cutoff $N$ enters only through the implementation
error and does not affect the Rademacher complexity
of the abstract Picard class.
In this sense, the bound avoids the curse of parametric
complexity at the statistical level, although the number
of Fourier modes and the computational cost still depend
on the spatial dimension.
\end{enumerate}
\end{remark}

\subsection{Proof of Lemma~\ref{lem:rademacher-torus}}
\label{app:F:proof-rademacher-torus}

We next estimate the Rademacher complexity of the abstract Picard class in the
torus setting. The key point is that the entropy is measured at the level of the
scalar nonlinearity class $\mathcal F_{M,L}$, and not at the level of the
Fourier-rank implementation.

\begin{lemma}[Lemma~\ref{lem:rademacher-torus}, Rademacher complexity of the Picard class]
\label{lem:F:rademacher-torus-restated}
Let $\mathcal S_{n,q}^{U_0}
=
\{(u_{0,i},\{z_{i,j}\}_{j=1}^q)\}_{i=1}^n$ be a trajectory-level sample as in
Section~\ref{subsec:risks}. Assume \eqref{eq:F:small-time-conditions}. Then, for any $\ell\in\mathbb N$, the
trajectory-normalized empirical Rademacher complexity of the 
Picard class satisfies
\begin{equation}
\label{eq:F:rademacher-torus-explicit}
    \widehat{\mathfrak R}^{\Omega_{\rm eq}}
    =
    \widehat{\mathfrak R}_{\mathcal S_{n,q}^{U_0}}
    \bigl(\mathcal B^{\Omega_{\rm eq}}_{\ell}\bigr)
    \le
    C M
    \sqrt{
    \frac{LT(1-\delta^\ell)}{(1-\delta)n}
    },
\end{equation}
where, $C>0$ is a universal constant. In particular,
\begin{equation}
\label{eq:F:rademacher-torus-uniform}
    \widehat{\mathfrak R}^{\Omega_{\rm eq}}
    \le
    C M
    \sqrt{
    \frac{LT}{(1-\delta)n}
    },
\end{equation}
so the Rademacher term is uniformly bounded in the Picard depth $\ell$ and is
independent of the Fourier cutoff $N$ of the Picard-type FNO implementation.
\end{lemma}

\begin{proof}
The case $\ell=0$ is trivial, since $\mathcal B_0^{\Omega_{\rm eq}}$ contains
only the zero predictor. We therefore assume $\ell\ge1$.
Let
\[
    d_{\mathcal S,2}(b,b')
    :=
    \left(
    \frac1{nq}
    \sum_{i=1}^n
    \sum_{j=1}^q
    \left|
    b(u_{0,i},z_{i,j})
    -
    b'(u_{0,i},z_{i,j})
    \right|^2
    \right)^{1/2}
\]
be the empirical pointwise metric on predictors. By the
trajectory-normalized Dudley entropy bound used in Section~\ref{sec:5},
\begin{equation}
\label{eq:F:dudley-picard-corrected}
    \widehat{\mathfrak R}_{\mathcal S_{n,q}^{U_0}}
    \bigl(\mathcal B_\ell^{\Omega_{\rm eq}}\bigr)
    \le
    \frac{C}{\sqrt n}
    \int_0^{2M}
    \sqrt{
    \log N
    \left(
    \mathcal B_\ell^{\Omega_{\rm eq}},
    d_{\mathcal S,2},
    \varepsilon
    \right)
    }
    \,d\varepsilon .
\end{equation}
Here, every element of $\mathcal B_\ell^{\Omega_{\rm eq}}$ is
$[-M,M]$-valued, since the Picard iterates stay in $U_M$. The factor
$n^{-1/2}$ comes from the trajectory-normalized complexity defined in
Section~\ref{sec:5}.

We first reduce the covering number of the Picard class to that of the scalar
nonlinearity class. Let $F,G\in\mathcal F_{M,L}$ and fix
$u_0\in U_{0,R}$. Define
\[
    u_F^{(0)}=u_G^{(0)}:=0,
    \quad
    u_F^{(j+1)}
    :=
    \mathcal T_{F,u_0}
    \bigl[
    u_F^{(j)}
    \bigr],
    \quad
    u_G^{(j+1)}
    :=
    \mathcal T_{G,u_0}
    \bigl[
    u_G^{(j)}
    \bigr].
\]
Both sequences are $U_M$-valued. Since the heat semigroup is
$L^\infty$-contractive and $F$ is $L$-Lipschitz on $[-M,M]$, we have
\[
\begin{aligned}
    \|u_F^{(j+1)}-u_G^{(j+1)}\|_{L^\infty(Q_T)}
    &\le
    T
    \left\|
    F(u_F^{(j)})-G(u_G^{(j)})
    \right\|_{L^\infty(Q_T)}
\\
    &\le
    \delta
    \|u_F^{(j)}-u_G^{(j)}\|_{L^\infty(Q_T)}
    +
    T
    \|F-G\|_{L^\infty([-M,M])} .
\end{aligned}
\]
Since the zeroth iterates coincide, induction gives
\[
    \left\|
    \mathcal T_{F,u_0}^{[\ell]}[0]
    -
    \mathcal T_{G,u_0}^{[\ell]}[0]
    \right\|_{L^\infty(Q_T)}
    \le
    T
    \frac{1-\delta^\ell}{1-\delta}
    \|F-G\|_{L^\infty([-M,M])} .
\]
Therefore, for any sample $\mathcal S_{n,q}^{U_0}$,
\begin{equation}
\label{eq:F:covering-reduction-corrected}
    N
    \left(
    \mathcal B_\ell^{\Omega_{\rm eq}},
    d_{\mathcal S,2},
    \varepsilon
    \right)
    \le
    N
    \left(
    \mathcal F_{M,L},
    \|\cdot\|_{L^\infty([-M,M])},
    \frac{(1-\delta)\varepsilon}
    {T(1-\delta^\ell)}
    \right).
\end{equation}

We next use the elementary entropy estimate for the one-dimensional Lipschitz
class:
\begin{equation}
\label{eq:F:lipschitz-class-entropy}
    \log N
    \left(
    \mathcal F_{M,L},
    \|\cdot\|_{L^\infty([-M,M])},
    r
    \right)
    \le
    C\frac{ML}{r},
    \quad r>0 .
\end{equation}
For completeness, we recall the short proof. If $L=0$, then
$\mathcal F_{M,L}=\{0\}$. Assume $L>0$. Choose a grid on $[-M,M]$ containing
$0$ with mesh size $h\asymp r/L$. The number of grid intervals is
$K\le CML/r$. Since $F(0)=0$ and ${\rm Lip}_{[-M,M]}(F)\le L$, the values at
adjacent grid points differ by at most $Lh\asymp r$. Quantizing the grid values
with mesh comparable to $r$ and using the adjacent-increment constraint yields
at most $C^K$ admissible quantized sequences. The corresponding piecewise affine
interpolants form an $r$-net in $L^\infty([-M,M])$. Hence
$\log N\le CK\le CML/r$.

Combining \eqref{eq:F:covering-reduction-corrected} and
\eqref{eq:F:lipschitz-class-entropy}, we obtain
\[
    \log N
    \left(
    \mathcal B_\ell^{\Omega_{\rm eq}},
    d_{\mathcal S,2},
    \varepsilon
    \right)
    \le
    C
    \frac{
    ML\,T(1-\delta^\ell)
    }{
    (1-\delta)\varepsilon
    } .
\]
Substituting this into \eqref{eq:F:dudley-picard-corrected} gives
\[
    \widehat{\mathfrak R}^{\Omega_{\rm eq}}
    \le
    \frac{C}{\sqrt n}
    \int_0^{2M}
    \sqrt{
    \frac{
    ML\,T(1-\delta^\ell)
    }{
    (1-\delta)\varepsilon
    }}
    \,d\varepsilon
\le
    C M
    \sqrt{
    \frac{LT(1-\delta^\ell)}{(1-\delta)n}
    } .
\]
This proves \eqref{eq:F:rademacher-torus-explicit}, and hence, 
\eqref{eq:F:rademacher-torus-uniform} follows 
since
$1-\delta^\ell\le1$. The proof of Lemma \ref{lem:F:rademacher-torus-restated} is finished.
\end{proof}

\subsection{Proof of Theorem~\ref{thm:torus-fno-generalization}}
\label{app:F:proof-torus-fno-generalization}

We next combine the implementation estimate for the Picard-type FNO with the
implementation-agnostic generalization error bound of Theorem~\ref{thm:generalization}.

\begin{theorem}[Theorem~\ref{thm:torus-fno-generalization}, Generalization error bound for the Picard-type FNO]
\label{thm:F:torus-fno-generalization-restated}
Assume the setting of Appendix~\ref{app:F:torus-setting} and Theorem \ref{thm:F:implementation-error-restated}. 
Let $\eta\in(0,1)$ and $\ell,N\in\mathbb N$. For each
$F\in\mathcal F_{M,L}$, choose a ReLU network $\rho_F$ as in
Theorem~\ref{thm:F:implementation-error-restated}, and define
\begin{equation}
\label{eq:F:embed-N-eta-thm74-corrected}
    \operatorname{embed}_{N,\eta}
    \left(
    b_{F,\ell}
    \right)
    :=
    \Gamma^{\rm FNO}_{N,\ell,\rho_F} .
\end{equation}
Let
\[
    \widehat b_\ell
    \in
    \operatorname*{argmin}_{b\in\mathcal B_\ell^{\Omega_{\rm eq}}}
    \widehat L_{n,q}[b],
\]
and set
\[
    \widehat\Gamma^{\rm FNO}_{N,\ell,\eta}
    :=
    \operatorname{embed}_{N,\eta}(\widehat b_\ell).
\]
Then there, exists a universal constant $C>0$ such that, for any
$0<\zeta<1$, with probability at least $1-\zeta$ over the trajectory-level
sample $\mathcal S_{n,q}^{U_0}$,
\begin{equation}
\label{eq:F:torus-fno-generalization-bound-corrected}
\begin{aligned}
    L
    \left[
    \widehat\Gamma^{\rm FNO}_{N,\ell,\eta}
    \right]
    \le\;&
    4M
    \left\{
    \frac{1-\delta^\ell}{1-\delta}\,T\eta
    +
    a_N(\ell)
    \right\}
    +
    \frac{M^2\delta^{2\ell}}{(1-\delta)^2}
\\
    &\quad
    +
    C M^2
    \sqrt{
    \frac{LT(1-\delta^\ell)}{(1-\delta)n}
    }
    +
    C M^2
    \sqrt{
    \frac{\log(2/\zeta)}{n}
    } .
\end{aligned}
\end{equation}
Here, $a_N(\ell)\to0$ as $N\to\infty$ is the Fourier implementation error
from Theorem~\ref{thm:torus-implementation-error}. In particular, the
statistical term is uniformly bounded in the Picard depth $\ell$ and is
independent of the Fourier cutoff $N$.
\end{theorem}

\begin{proof}
The proof is a direct specialization of Theorem~\ref{thm:generalization} to the torus heat equation class.

First, by Theorem \ref{thm:F:implementation-error-restated}  (Theorem~\ref{thm:torus-implementation-error}), the map
$\operatorname{embed}_{N,\eta}$ defined in
\eqref{eq:F:embed-N-eta-thm74-corrected} implements the abstract Picard class
$\mathcal B_\ell^{\Omega_{\rm eq}}$ by Picard-type FNOs. More precisely, for
every $F\in\mathcal F_{M,L}$,
\[
    \sup_{u_0\in U_0}
    \left\|
    \operatorname{embed}_{N,\eta}(b_{F,\ell})(u_0)
    -
    b_{F,\ell}(u_0,\cdot)
    \right\|_{L^\infty(Q_T)}
    \le
    \frac{1-\delta^\ell}{1-\delta}\,T\eta
    +
    a_N(\ell).
\]
Since both $b_{F,\ell}$ and
$\operatorname{embed}_{N,\eta}(b_{F,\ell})$ are $[-M,M]$-valued, and since
$\mathcal G_\ast(u_0)(z)\in[-M,M]$, the squared loss is $4M$-Lipschitz in the
prediction variable on this range. Therefore,
\begin{equation}
\label{eq:F:implementation-risk-thm74-corrected}
    \varepsilon_{\rm imp}^{\rm FNO}(N,\eta)
    :=
    \sup_{b\in\mathcal B_\ell^{\Omega_{\rm eq}}}
    \left|
    L[\operatorname{embed}_{N,\eta}(b)]
    -
    L[b]
    \right|
    \le
    4M
    \left\{
    \frac{1-\delta^\ell}{1-\delta}\,T\eta
    +
    a_N(\ell)
    \right\}.
\end{equation}

Next, apply Theorem~\ref{thm:generalization} {\rm (i)} with
$\operatorname{embed}=\operatorname{embed}_{N,\eta}$. The approximate Picard
maps defining the Picard-type FNO are $U_M$-valued by
\eqref{eq:F:approx-picard-self-map}; hence, the boundedness condition required
in Theorem~\ref{thm:generalization} is satisfied without adding an extra
clipping layer. Thus, with probability at least $1-\zeta$,
\begin{equation}
\label{eq:F:generalization-before-substitution-corrected}
    L
    \left[
    \widehat\Gamma^{\rm FNO}_{N,\ell,\eta}
    \right]
    \le
    \varepsilon_{\rm imp}^{\rm FNO}(N,\eta)
    +
    \frac{M^2\delta^{2\ell}}{(1-\delta)^2}
    +
    CM\widehat{\mathfrak R}^{\Omega_{\rm eq}}
    +
    CM^2
    \sqrt{
    \frac{\log(2/\zeta)}{n}
    } .
\end{equation}
Here, the term $M^2\delta^{2\ell}/(1-\delta)^2$ is the oracle error coming from
the Picard truncation estimate in Lemma~\ref{lem:picard_error1}.
By combining \eqref{eq:F:generalization-before-substitution-corrected} with \eqref{eq:F:implementation-risk-thm74-corrected} 
and 
Lemma \ref{lem:F:rademacher-torus-restated}
(Lemma~\ref{lem:rademacher-torus}), we obtain \eqref{eq:F:torus-fno-generalization-bound-corrected}.
This proves the theorem.
\end{proof}

\begin{corollary}[Parameter choices for a prescribed risk]
\label{cor:F:parameter-scaling}
Assume the setting of
Theorem~\ref{thm:F:torus-fno-generalization-restated},
and fix
$0<\gamma<\min\{s_0-d/2,1\}$.
For any $\varepsilon,\zeta\in(0,1)$, the parameters can be
chosen so that
\[
    L\!\left[\widehat\Gamma^{\rm FNO}_{N,\ell,\eta}\right]
    \le\varepsilon
\]
with probability at least $1-\zeta$.
A sufficient choice has the scaling
\[
    \eta\asymp\varepsilon,\quad
    H=O(\varepsilon^{-1}),\quad
    \ell=O\!\left(1+\log\varepsilon^{-1}\right),\quad
    N=O(\varepsilon^{-1/\gamma}),\quad
    n=O\!\left(\varepsilon^{-2}
    (1+\log(1/\zeta))\right).
\]
The implied constants may depend on
$R,M,L,T,\delta,R_0,d,s_0,\gamma$,
but not on $\varepsilon,\zeta$, or $q$.
The number of retained Fourier modes satisfies
\[
    |\Lambda_N|=(2N+1)^d
    =O(\varepsilon^{-d/\gamma}).
\]
\end{corollary}

\begin{proof}
By Proposition~\ref{prop:F:fourier-rate} and
Theorem~\ref{thm:F:torus-fno-generalization-restated},
there exists a constant $C>0$, depending only on the
fixed parameters, such that
\[
    L\!\left[\widehat\Gamma^{\rm FNO}_{N,\ell,\eta}\right]
    \le
    C\left(
        \eta+N^{-\gamma}+\delta^{2\ell}
        +\sqrt{\frac{1+\log(1/\zeta)}{n}}
    \right)
\]
with probability at least $1-\zeta$.
Choose the parameters so that each term on the right
is at most $\varepsilon/4$.
The network-size bound follows from
$H(\rho_F)\le C(1+ML\eta^{-1})$.
\end{proof}

The bound does not require $q$ to increase:
$q=1$ is sufficient for consistency as $n\to\infty$,
with the approximation parameters chosen accordingly.
This conclusion concerns the estimator obtained by
empirical risk minimization over the abstract Picard class
followed by the implementation map.

\begin{remark}[Fourier cutoff, depth, and the curse of parametric complexity]
\label{rem:F:rank-depth-pc}
The bound \eqref{eq:F:torus-fno-generalization-bound-corrected}
separates the three roles of the approximation parameters.
The Picard depth $\ell$ appears in the geometric truncation
term $M^2\delta^{2\ell}/(1-\delta)^2$.
The scalar network accuracy $\eta$ appears in the nonlinear
implementation term $T\eta(1-\delta^\ell)/(1-\delta)$.
The Fourier cutoff $N$ appears only through $a_N(\ell)$.
Proposition~\ref{prop:F:fourier-rate} gives
$a_N(\ell)\le C N^{-\gamma}$ uniformly in $\ell$ for any
$0<\gamma<\min\{s_0-d/2,1\}$.
Corollary~\ref{cor:F:parameter-scaling} gives sufficient
parameter choices for a prescribed population risk.
In contrast, the statistical term is controlled by the
one-dimensional entropy of $\mathcal F_{M,L}$ and does not
depend on the Fourier cutoff.
This is the sense in which the Picard-type FNO implementation
avoids the curse of parametric complexity at the statistical
level, although the numerical cost of evaluating the Fourier
representation still depends on the spatial dimension
through the number of active modes.
\end{remark}

\subsection{Long-Time Prediction by Rolling Out the Picard-Type FNO}
\label{app:F:long-time-fno}

We finally record the long-time consequence of the local Picard-type FNO
construction. This is a direct specialization of
Theorem~\ref{thm:blockwise-long-time} to the torus setting. The terminal trace map
\[
    \psi:U_M\to C(\mathbb T^d),
    \quad
    \psi(u):=u(T)
\]
is $1$-Lipschitz with respect to the $L^\infty$ norms:
\begin{equation}
\label{eq:F:terminal-trace-lipschitz}
    \|\psi(u)-\psi(v)\|_{L^\infty(\mathbb T^d)}
    \le
    \|u-v\|_{L^\infty(Q_T)} .
\end{equation}
Thus, in the notation of Theorem \ref{thm:blockwise-long-time_app} (Theorem~\ref{thm:blockwise-long-time}), we have
$L_\psi=1$ and $C_S=1$.
Let $\widehat\Gamma^{\rm FNO}_{N,\ell,\eta}:U_{0,R}\to U_M$ be a local
Picard-type FNO model. For example, one may take the estimator constructed in
Theorem~\ref{thm:torus-fno-generalization}, extended from $U_0$ to
$U_{0,R}$ by the same Picard-type FNO formula. Define the exact local
time-$T$ transition by
\[
    \Phi_\ast:=\psi\circ \mathcal G_\ast .
\]
For $\kappa\in\mathbb N$, set
\[
    U_0^{[\kappa]}
    :=
    \left\{
    u_0\in U_0:
    \Phi_\ast^j(u_0)\in U_{0,R}
    \text{ for } j=0,\ldots,\kappa
    \right\}.
\]
We assume $\Pi(U_0^{[\kappa]})>0$ and write
$\Pi_\kappa:=\Pi(\cdot\mid U_0^{[\kappa]})$.

For $u_0\in U_0^{[\kappa]}$, define the exact rollout states by
\[
    v_0:=u_0,
    \quad
    v_{j+1}:=\Phi_\ast(v_j),
    \quad
    j=0,\ldots,\kappa-1 .
\]
Let $\mathcal P_R:C(\mathbb T^d)\to U_{0,R}$ be the pointwise clipping map
\[
    (\mathcal P_R f)(x)
    :=
    \max\{-R,\min\{f(x),R\}\}.
\]
The approximate rollout transition induced by the Picard-type FNO is
\begin{equation}
\label{eq:F:approx-rollout-transition}
    \widehat\Phi^{\rm FNO}_{N,\ell,\eta}
    :=
    \mathcal P_R\circ \psi\circ
    \widehat\Gamma^{\rm FNO}_{N,\ell,\eta}.
\end{equation}
The approximate states are defined by
\[
    \widehat v_0:=u_0,
    \quad
    \widehat v_{j+1}
    :=
    \widehat\Phi^{\rm FNO}_{N,\ell,\eta}(\widehat v_j),
    \quad
    j=0,\ldots,\kappa-1 .
\]
The clipping in \eqref{eq:F:approx-rollout-transition} is used only to keep the
next input state inside $U_{0,R}$. It is not an output clipping layer in the
local Picard-type FNO; the local FNO itself is already $U_M$-valued by
\eqref{eq:F:approx-picard-self-map}.

For $j=0,\ldots,\kappa-1$, define the block-wise long-time risk by
\[
    L_j^{\rm block}
    \left[
    \widehat\Gamma^{\rm FNO}_{N,\ell,\eta}
    \right]
    :=
    \mathbb E_{u_0\sim\Pi_\kappa,\ z\sim\mu}
    \left[
    \left|
    \widehat\Gamma^{\rm FNO}_{N,\ell,\eta}(\widehat v_j)(z)
    -
    \mathcal G_\ast(v_j)(z)
    \right|^2
    \right].
\]
Here, $z=(x,s)\in Q_T$ is the local space-time coordinate on the $j$-th block,
corresponding to the global time $t=jT+s$.

Define the rollout-relevant state set by
\[
    \mathcal V_\kappa
    \left(
    \widehat\Gamma^{\rm FNO}_{N,\ell,\eta}
    \right)
    :=
    \left\{
    v_j(u_0),\widehat v_j(u_0):
    u_0\in U_0^{[\kappa]},
    \ j=0,\ldots,\kappa
    \right\}
    \subset U_{0,R}.
\]
The corresponding rollout-relevant local error is
\[
    \varepsilon_{{\rm loc},N,\ell,\eta}^{[\kappa]}
    :=
    \sup_{v\in
    \mathcal V_\kappa
    (
    \widehat\Gamma^{\rm FNO}_{N,\ell,\eta}
    )}
    \left\|
    \widehat\Gamma^{\rm FNO}_{N,\ell,\eta}(v)
    -
    \mathcal G_\ast(v)
    \right\|_{L^\infty(Q_T)} .
\]

\begin{theorem}[Long-time rollout of the Picard-type FNO]
\label{thm:F:long-time-fno}
Assume the setting of Appendix~\ref{app:F:torus-setting} and let
$\widehat\Gamma^{\rm FNO}_{N,\ell,\eta}:U_{0,R}\to U_M$ be a Picard-type FNO
local model. Then, for any $\kappa\in\mathbb N$ and every
$j=0,\ldots,\kappa-1$,
\begin{equation}
\label{eq:F:long-time-fno-bound}
    L_j^{\rm block}
    \left[
    \widehat\Gamma^{\rm FNO}_{N,\ell,\eta}
    \right]
    \le
    \left(
    \sum_{r=0}^j
    \left(
    \frac1{1-\delta}
    \right)^r
    \right)^2
    \left(
    \varepsilon_{{\rm loc},N,\ell,\eta}^{[\kappa]}
    \right)^2 .
\end{equation}
In particular, the rollout step introduces no additional Rademacher complexity
term. The block index $j$ appears only through the deterministic stability
factor in \eqref{eq:F:long-time-fno-bound}.
\end{theorem}

\begin{proof}
This is Theorem~\ref{thm:blockwise-long-time} applied to the present torus
setting. Indeed, by \eqref{eq:F:initial-data-lipschitz}, the local solution
operator satisfies
\begin{equation}
\label{eq:F:torus-solution-lipschitz-for-rollout}
    \|\mathcal G_\ast(u_0)-\mathcal G_\ast(v_0)\|_{L^\infty(Q_T)}
    \le
    \frac1{1-\delta}
    \|u_0-v_0\|_{L^\infty(\mathbb T^d)}
\end{equation}
for all $u_0,v_0\in U_{0,R}$. This is the estimate
\eqref{CDoI} with $C_S=1$. Moreover, the terminal trace map satisfies
$L_\psi=1$ by \eqref{eq:F:terminal-trace-lipschitz}. Hence, the stability
factor in Theorem~\ref{thm:blockwise-long-time} becomes
\[
    L_\psi\frac{C_S}{1-\delta}
    =
    \frac1{1-\delta}.
\]
Substituting this value into Theorem~\ref{thm:blockwise-long-time} gives
\eqref{eq:F:long-time-fno-bound}.
\end{proof}

\begin{remark}[Oracle local-error specialization]
\label{rem:F:oracle-long-time-local-error}
If one considers the target-specific oracle implementation
$\Gamma^{\rm FNO}_{N,\ell,\rho_{F_\ast}}$ and if the rollout-relevant states
are contained in a fixed compact set
$K_\kappa\Subset C(\mathbb T^d)\cap U_{0,R}$, then the compactness
argument in the proof of Theorem~\ref{thm:torus-implementation-error}, with
$U_0$ replaced by $K_\kappa$, yields a Fourier error
$a_{N,K_\kappa}(\ell)\to0$ such that
\[
    \sup_{v\in K_\kappa}
    \left\|
    \Gamma^{\rm FNO}_{N,\ell,\rho_{F_\ast}}(v)
    -
    \mathcal G_\ast(v)
    \right\|_{L^\infty(Q_T)}
    \le
    \frac{M\delta^\ell}{1-\delta}
    +
    \frac{1-\delta^\ell}{1-\delta}\,T\eta
    +
    a_{N,K_\kappa}(\ell).
\]
Consequently, for this oracle Picard-type FNO,
\[
    L_j^{\rm block}
    \left[
    \Gamma^{\rm FNO}_{N,\ell,\rho_{F_\ast}}
    \right]
    \le
    \left(
    \sum_{r=0}^j
    \left(
    \frac1{1-\delta}
    \right)^r
    \right)^2
    \left(
    \frac{M\delta^\ell}{1-\delta}
    +
    \frac{1-\delta^\ell}{1-\delta}\,T\eta
    +
    a_{N, K_\kappa}(\ell)
    \right)^2 .
\]
This illustrates how the local Picard truncation error, scalar-network
approximation error, and Fourier implementation error propagate over long
time horizons. The Fourier rank still appears only through the local
implementation error, and the rollout itself introduces no new statistical
complexity term.
\end{remark}

\begin{remark}[The case of defocusing nonlinearities]
\label{rem:F:defocusing-strictly-dissipative-rollout}
We record two useful refinements of the long-time rollout discussion for
defocusing nonlinearities.

\begin{enumerate}[(a)]
    \item Define the defocusing (or absorbing)  subclass by
\[
    \mathcal F^{\rm def}_{M,L}
    :=
    \left\{
    F\in\mathcal F_{M,L}:
    aF(a)\le0
    \text{ for all }a\in[-M,M]
    \right\}.
\]
Equivalently, $F(a)\le0$ for $a>0$ and $F(a)\ge0$ for $a<0$. If
$F_\ast\in\mathcal F^{\rm def}_{M,L}$, then the comparison principle gives
\[
    |u(t)|
    \le
    e^{t\Delta_{\mathbb T^d}}|u_0|,
    \quad
    \|u(t)\|_{L^\infty(\mathbb T^d)}
    \le
    \|u_0\|_{L^\infty(\mathbb T^d)}
\]
for the exact solution of
$\partial_tu-\Delta_{\mathbb T^d}u=F_\ast(u)$. Hence, the local mild solution
extends globally in time. In particular, if $u_0\in U_{0,R}$, then
\[
    \Phi_\ast^j(u_0)\in U_{0,R}
    \quad
    \text{for any }j\ge0 .
\]
Therefore, for any finite rollout horizon $\kappa\in\mathbb N$, the
admissible exact-rollout set satisfies
\[
    U_0^{[\kappa]}=U_0 .
\]
Thus, Theorem~\ref{thm:F:long-time-fno} applies for arbitrarily long finite
time horizons without imposing an additional admissibility assumption on the
exact trajectory. This should not be confused with a uniform-in-time error
bound: the deterministic factor in \eqref{eq:F:long-time-fno-bound} may still
grow with the block index $j$, and the rollout-relevant local error
$\varepsilon_{{\rm loc},N,\ell,\eta}^{[\kappa]}$ may depend on $\kappa$.

    \item 
A stronger conclusion is available under strict dissipativity. For
$\lambda>0$, define
\[
    \mathcal F^{\rm sd}_{M,L,\lambda}
    :=
    \left\{
    F\in\mathcal F_{M,L}:
    (F(a)-F(b))(a-b)
    \le
    -\lambda |a-b|^2
    \text{ for all }a,b\in[-M,M]
    \right\}.
\]
If $F$ is $C^1$, this condition is equivalent to
$F'(a)\le-\lambda$ on $[-M,M]$. In particular,
$\mathcal F^{\rm sd}_{M,L,\lambda}\subset\mathcal F^{\rm def}_{M,L}$.
Typical examples are
\[
    F(u)=-\alpha u-\beta |u|^{p-1}u,
    \quad
    \alpha\ge\lambda,\ \beta\ge0,\ p\ge1 
\]
(when $\alpha=0$ and $p>1$, the nonlinearity is defocusing
but not strictly dissipative). 
Assume now that $F_\ast\in\mathcal F^{\rm sd}_{M,L,\lambda}$. Then the exact
local solution operator is nonexpansive on each time block, and the exact
terminal transition is a contraction:
\begin{equation}
\label{eq:F:strict-nonexpansive-block}
    \|\mathcal G_\ast(u_0)-\mathcal G_\ast(v_0)\|_{L^\infty(Q_T)}
    \le
    \|u_0-v_0\|_{L^\infty(\mathbb T^d)},
\end{equation}
and
\begin{equation}
\label{eq:F:strict-terminal-contraction}
    \|\Phi_\ast(u_0)-\Phi_\ast(v_0)\|_{L^\infty(\mathbb T^d)}
    \le
    q
    \|u_0-v_0\|_{L^\infty(\mathbb T^d)},
    \quad
    q:=e^{-\lambda T}<1 .
\end{equation}
Indeed, for two exact solutions $u$ and $v$, the difference $w=u-v$ satisfies a
linear equation with a zeroth-order coefficient bounded above by $-\lambda$,
and the maximum principle yields the above estimates.

Consequently, the proof of Theorem~\ref{thm:F:long-time-fno} can be sharpened
as follows. Let
\[
    e_j
    :=
    \|\widehat v_j-v_j\|_{L^\infty(\mathbb T^d)} .
\]
Using \eqref{eq:F:strict-terminal-contraction} instead of the crude Lipschitz
bound \eqref{eq:F:torus-solution-lipschitz-for-rollout}, we obtain
\[
    e_{j+1}
    \le
    \varepsilon_{{\rm loc},N,\ell,\eta}^{[\kappa]}
    +
    q e_j,
    \quad
    e_0=0 .
\]
Hence
\[
    e_j
    \le
    \frac{1-q^j}{1-q}
    \varepsilon_{{\rm loc},N,\ell,\eta}^{[\kappa]} .
\]
Combining this with the block-wise nonexpansiveness
\eqref{eq:F:strict-nonexpansive-block}, we get
\[
    L_j^{\rm block}
    \left[
    \widehat\Gamma^{\rm FNO}_{N,\ell,\eta}
    \right]
    \le
    \left(
    1+
    \frac{1-q^j}{1-q}
    \right)^2
    \left(
    \varepsilon_{{\rm loc},N,\ell,\eta}^{[\kappa]}
    \right)^2 .
\]
In particular,
\[
    L_j^{\rm block}
    \left[
    \widehat\Gamma^{\rm FNO}_{N,\ell,\eta}
    \right]
    \le
    \left(
    1+\frac1{1-e^{-\lambda T}}
    \right)^2
    \left(
    \varepsilon_{{\rm loc},N,\ell,\eta}^{[\kappa]}
    \right)^2
    \quad
    \text{for all }j=0,\ldots,\kappa-1 .
\]
Thus, strict dissipativity replaces the growing stability factor in
\eqref{eq:F:long-time-fno-bound} by a factor uniformly bounded in the block
index.

\end{enumerate}
\end{remark}

\section{Additional Remarks}\label{app:G}

\paragraph{Potential extensions.}
Although this paper mainly treats nonlinear parabolic PDEs, the essence of the framework lies not in parabolicity itself, but in the Picard-type fixed-point structure and in the ability to estimate the statistical complexity of the corresponding solution class. Picard iteration based on Banach's fixed point theorem is a standard method for constructing solutions to nonlinear evolution equations \citep{Banach1922,Henry1981,Pazy1983,Zeidler,CazenaveHaraux1998}, and it also appears in the analysis of the Navier--Stokes equations, nonlinear dispersive equations, wave equations, and related problems \citep{GigaMiyakawa1985,Cannone2004,Cazenave2003,Tao2006}. Therefore, the framework developed in this paper may be extendable to a broader class of nonlinear evolution equations that admit a Picard-type fixed-point formulation.

Furthermore, even for problems where, Picard iteration is not directly applicable, many iterative solvers are used in PDEs and inverse problems, such as Newton methods, monotone iteration methods, and iterative regularization methods \citep{Deuflhard2011,Pao1992,KaltenbacherHuynh2022,Nguyen2023}. From this viewpoint, the idea of the present paper is not necessarily limited to Picard iteration, but may be extended to more general solver-aligned iterative frameworks. This direction is also consistent with recent work on neural general operator networks based on Banach fixed point iterations \citep{FeischlSchwabZehetgruber2025}, and such extensions are an important direction for future research.

\paragraph{Equation identification.}
The present paper uses the equation class $\Omega_{\rm eq}$ as structural prior
information, but it does not aim to identify the target equation specification
$\omega_\ast$. Although each $\omega\in\Omega_{\rm eq}$ determines a solution
operator $\mathcal G_\omega$, recovering $\omega_\ast$ from data or from
$\mathcal G_\ast$ is a separate inverse problem. This inverse direction raises
questions of identifiability, stability, and observation design, and is beyond
the scope of this paper. Extending the present framework toward equation
identification is an interesting direction for future work.

\paragraph{Weight-tied architectures.}
The Picard-type constructions in this paper can also be viewed as weight-tied
architectures. In Definition~\ref{def:F:picard-type-fno}, the same approximate
Picard transition $\mathcal T^N_{\rho,u_0}$ is repeatedly applied, and the same
scalar network $\rho$ is shared across all Picard steps. The Fourier--heat
operators $S_N$ and $\mathcal K_N$ are fixed by the PDE structure. Thus
increasing the Picard depth $\ell$ does not introduce independent layer-wise
parameters, but corresponds to iterating the same PDE-aligned transition.

This is closely related to the weight-tied neural operators discussed in
\citet[Remark~3]{Furuya2025} and to fixed-point or equilibrium-type models. The
present paper adds a statistical interpretation: because the Rademacher term is
controlled by the abstract Picard transition class, increasing the number of
Picard steps does not cause an unbounded growth of the entropy-based
statistical term. This suggests that weight tying is not only memory-efficient,
but also naturally aligned with the generalization theory of solver-aligned
operator learning. A systematic empirical study of such weight-tied
Picard-type neural operators is left for future work.

\paragraph{Beyond purely deterministic rollout.}
The long-time prediction result in this paper uses the learned local model
without retraining it after the first training on short-time interval. This is consistent with
the standard extension argument for nonlinear PDEs: once a local solution
operator is available, the solution is extended by using the terminal value of
one time block as the initial data for the next block. In this setting, the
rollout step is deterministic and introduces no additional Rademacher
complexity term.

From a practical viewpoint, however, this is only one possible strategy for
long-time prediction. It is natural to consider correction mechanisms that
reduce the accumulation of rollout errors. For example, one may introduce a
learned correction term at each time block, retrain or fine-tune the local model
along the rollout trajectory, or combine the learned model with data
assimilation when partial observations are available. Such approaches may
improve long-time stability in practice, especially for weakly stable or
unstable dynamics.

These extensions require a different theoretical treatment. Once correction,
retraining, or data assimilation is introduced during the rollout, the error is
no longer obtained only by deterministic propagation of a fixed local model. One
must also control the statistical error, optimization error, or assimilation
error generated at each correction step. Developing a generalization theory for
such corrected or data-assimilated rollouts is an important direction for future
work.

\paragraph{Further PDE-specific refinements.}
In Section~\ref{sec:torus-example} and Appendix~\ref{app:F:torus-setting}, we
treated semilinear heat equations on the torus as one concrete example of the
general theory. This example was chosen because the heat semigroup is explicitly
diagonalized by the Fourier basis and the Fej\'er cutoff gives an
$L^\infty$-stable implementation. However, the same strategy is not limited to
the torus setting. Similar results should be obtainable for other domains,
boundary conditions, linear operators, or basis expansions, provided that the
corresponding semigroup estimates, compactness properties, and implementation
errors can be controlled.

Moreover, different PDE classes carry different structural properties, and
these properties may lead to sharper or more specialized bounds. For example,
maximum principles, comparison principles, dissipativity, smoothing effects,
spectral decay, or conservation laws can be incorporated into the rollout or
implementation analysis when they are available. As a simple illustration, in
Remark~\ref{rem:F:defocusing-strictly-dissipative-rollout} we discussed
defocusing and strictly dissipative nonlinearities. In the defocusing case, the
comparison principle gives global admissibility of the exact rollout, while in
the strictly dissipative case the stability factor in the long-time error
propagation can be replaced by a uniformly bounded one. These examples suggest
that PDE-specific structure can be used not only to construct Picard-type
implementations, but also to improve the resulting long-time prediction bounds.
A systematic study of such PDE-specific refinements is left for future work.

\end{document}